\pdfoutput=1
\documentclass[10pt,onecolumn]{article}
\usepackage{arxiv_style}
\newcommand*{\tran}{^\top}
\newcommand*{\dd}{\mathrm{d}}
\newcommand{\vx}{\bm{\mathrm{x}}}
\newcommand{\vU}{\bm{\mathrm{U}}}
\newcommand{\vM}{\bm{\mathrm{M}}}
\newcommand{\vP}{\bm{\mathrm{P}}}
\newcommand{\vK}{\bm{\mathrm{K}}}

\newcommand{\vXi}{\bm{\Xi}}
\newcommand{\va}{\bm{\mathrm{a}}}
\newcommand{\vv}{\bm{\mathrm{v}}}
\newcommand{\vy}{\bm{\mathrm{y}}}
\newcommand{\vA}{\bm{\mathrm{A}}}
\newcommand{\vW}{\bm{\mathrm{W}}}
\newcommand{\vw}{\bm{\mathrm{w}}}
\newcommand{\vs}{\bm{\mathrm{s}}}
\newcommand{\vB}{\bm{\mathrm{B}}}
\newcommand{\vV}{\bm{\mathrm{V}}}
\newcommand{\vX}{\bm{\mathrm{X}}}
\newcommand{\vY}{\bm{\mathrm{Y}}}
\newcommand{\vG}{\bm{\mathrm{G}}}
\newcommand{\vz}{\bm{\mathrm{z}}}
\newcommand{\vI}{\bm{\mathrm{I}}}
\newcommand{\vxi}{\bm{\xi}}
\newcommand{\vphi}{\bm{\phi}}
\newcommand{\vSigma}{\boldsymbol{\Sigma}}
\newcommand{\vtheta}{\bm{\theta}}
\newcommand{\vpsi}{\bm{\psi}}
\newcommand{\vOmega}{\bm{\Omega}}
\newcommand{\vLambda}{\bm{\Lambda}}
\newcommand{\valpha}{\bm{\alpha}}
\newcommand{\vvarphi}{\bm{\varphi}}
\newcommand{\vTheta}{\bm{\Theta}}
\newcommand{\veta}{\bm{\eta}}
\newcommand{\vf}{\bm{\mathrm{f}}}
\newcommand{\vS}{\bm{\mathrm{S}}}
\newcommand{\vu}{\bm{\mathrm{u}}}
\newcommand{\vJ}{\bm{\mathrm{J}}}
\newcommand{\vQ}{\bm{\mathrm{Q}}}
\newcommand{\vR}{\bm{\mathrm{R}}}
\newcommand{\vH}{\bm{\mathrm{H}}}

\newcommand{\vC}{\bm{\mathrm{C}}}
\newcommand{\vD}{\bm{\mathrm{D}}}

\newcommand{\Ltrain}{\mathcal{L}_{\mathrm{train}}}
\newcommand{\Ltest}{\mathcal{L}_{\mathrm{test}}}

\DeclareMathOperator{\Tr}{Tr}

\title{First Learn, Then Memorize:\\ The Spectral Bias of Diffusion Models}

\author[1]{Rapha\"el Urfin%
\thanks{Corresponding author: \href{mailto:raphael.urfin@phys.ens.fr}{\texttt{raphael.urfin@phys.ens.fr}}}%
}
\author[2]{Tony Bonnaire}
\author[1]{Giulio Biroli}
\author[3]{Marc M\'ezard}
\affil[1]{Laboratoire de Physique de l'\'Ecole normale sup\'erieure, ENS, Universit\'e PSL, CNRS, Sorbonne Universit\'e, Universit\'e Paris Cit\'e, F-75005 Paris, France}
\affil[2]{Universit\'e Paris-Saclay, CNRS, Institut d'Astrophysique Spatiale, 91405 Orsay, France}
\affil[3]{Department of Computing Sciences, Bocconi University, Milano, Italy}
\date{\vspace{-7ex}}

\begin{document}
\pagenumbering{arabic}
\maketitle

\selectlanguage{american}

\begin{abstract}

Diffusion models trained on a finite dataset first learn to generate novel, high-quality samples and only much later collapse onto their training set.
We identify the mechanism behind this separation of timescales and the object that probes it. The training dynamics of the score function are governed---exactly, and at any width---by the Gram matrix of the Neural Tangent Kernel (NTK) evaluated on the noisy training data, so the timescales of generalization and of memorization must be encoded in its spectrum. We show that they are, and that the structure responsible has no analogue in standard kernel settings. The use of multiple noise realizations per sample ($m$ noised copies at a fixed noise level) in the score-matching loss is what restructures the Gram matrix spectrum into two distinct parts. The first, of large eigenvalues, carries the global features of the target distribution and is present already for $m=1$. The second, which the repeated noising creates,
consists of the smallest eigenvalues and is supported on eigenvectors aligned with the sample-specific noise directions; it sets a memorization timescale parametrically larger in the training set size $n$. We establish this picture on two fronts. Analytically, we solve the spectrum in the lazy high-dimensional limit for both linear ($n \asymp d$) and polynomial ($n \asymp d^k$) sample complexities, and prove through a bias--variance decomposition that the first bulk minimizes the approximation error while the second drives the error associated with memorization. Empirically, we show the same two-bulk structure in Convolutional NTKs on CelebA and in finite-width U-Nets trained well beyond the lazy regime, and we make the link causal: truncating the Gram matrix at rank $r$ tunes the generalization--memorization transition, and an $L_2$ penalty targeting the second bulk suppresses memorization in feature-learning U-Nets. \looseness=-1
\end{abstract}

\textit{\small\textbf{Keywords: }%
 {Diffusion Models} $|$ {Lazy Regime} $|$ {Memorization} }%
\vspace{1cm}
\section{Introduction}
\label{sect:intro}

Generative diffusion models \citep{sohl-dickstein_15, ho2020,song2019} have established a new state of the art in high-dimensional data generation, achieving unprecedented success across modalities ranging from image \citep{rombach2022latent} and video generation \citep{sora2024} to applications in the natural sciences \citep{Price_2025, Biferale_2024}. Rooted in the principles of out-of-equilibrium statistical mechanics, these models operate via two stochastic processes: a \emph{forward process} that progressively degrades the data into white noise, and a reverse-time generative process, called the \emph{backward process}, which generates new samples by simulating a diffusion process driven by a drift vector field, called the score function \citep{hyvarinen_05, Vincent_2011}, defined precisely as the gradient of the logarithm of the noisy data marginals. Recently, this paradigm has been unified under the broader framework of stochastic interpolants \citep{albergo2023stochastic, lipman_2023, liu2022rectified}, where the noising process can be arbitrary and the generative dynamics can be recast as a deterministic flow. \looseness=-1

Despite their empirical success, we still lack a rigorous theoretical understanding of how the training dynamics finds a score that captures the intrinsic structure of the target distribution. In particular, it remains unclear how highly overparametrized models of the score generate novel, high-quality samples without simply memorizing the finite training set. One prominent line of explanation is implicit dynamical regularization \citep{bonnaire2025diffusionmodelsdontmemorize,Favero2025_bigger}. The learning dynamics is governed by a strict separation of timescales: a fast timescale, $\tau_{\mathrm{gen}}$, during which the model learns the underlying structural features required to generate high-quality samples, and a much longer timescale, $\tau_{\mathrm{mem}}$, at which the network memorizes the dataset. Here, we identify the mechanism behind this separation of timescales, and
the object that reveals it: the spectrum of the Gram matrix of the Neural Tangent Kernel (NTK). \looseness=-1

\paragraph*{Generative diffusion and our setting.}
Standard diffusion models transport a target distribution $P_0$ on $\mathbb{R}^d$ to Gaussian
white noise $\mathcal{N}(0,\vI_d)$ via an Ornstein--Uhlenbeck (OU) \emph{forward process}
$\dd\vx = -\vx\,\dd t + \sqrt{2}\,\dd\vW(t)$, where $\vW(t)$ is a standard Wiener process.
Exact time-reversal of this process corresponds to using a guiding force field which is the exact score function
$\vs_{\text{exact}}(\vx,t)=\nabla_{\vx}\log P_t(\vx)$ \citep{Anderson_1982,
haussmann_1986}, where $P_t$ is the marginal density
at time $t$. Following \citet{hyvarinen_05} and \citet{Vincent_2011}, generation is
performed by using a parametrized score function $\vs(\vx,t)$. Given a dataset
$\mathcal D=\{\vx^\nu\}_{\nu=1,\ldots,n}$ consisting of $n$ i.i.d.\ samples from $P_0$, one
learns this score function by minimizing the Denoising Score Matching (DSM) loss. At fixed
$t$,\footnote{Fixing $t$ is a standard simplification in theoretical works \citep{cui_2024,george_2025}. Each noise level defines a separate regression problem, and
the memorization transition we study is a small-$t$ phenomenon; Sect.~\ref{sect:experiments}
reports a sweep over $t$. For simplicity, we drop the time argument $t$ of the score. \looseness=-1} it reads
\begin{align}\label{eq:m_infty_loss}
  \mathcal{L}_{\mathrm{DSM}}(\vs) = \frac{1}{2nd}\sum_{\nu=1}^n\mathbb{E}_{\vxi}\!
    \left[\left\lVert\vs(e^{-t}\vx^\nu+\sqrt{\Delta_t}\,\vxi)+\frac{\vxi}{\sqrt{\Delta_t}}\right\rVert^2\right],
\end{align}
where $\Delta_t=1-e^{-2t}$ and the expectation is over $\vxi\sim\mathcal{N}(0,\vI_d)$. Its
global minimizer is the \emph{empirical score}, the score of the noised empirical
distribution \citep{pidstrigach2022scorebased,Biroli_2024, li_2024_good_score}, which inevitably collapses the
backward trajectories onto the training samples unless $n$ grows
exponentially with the intrinsic data dimension \citep{Biroli_2024,achilli2024}. \looseness=-1

The expectation in \eqref{eq:m_infty_loss} is never performed exactly. Training proceeds by
stochastic gradient descent over batches of the $n$ samples, drawing a fresh noise $\vxi$
each time a sample is visited. Each sample is thus seen with $m$ different noise realizations, $m$ being the number of epochs.
For the sake of our analysis, we replace the expectation by a finite sum. This will be useful to make comparisons with standard supervised learning, corresponding to $m=1$, and allows us to recover (\ref{eq:m_infty_loss}) taking $m\rightarrow \infty$.
We attach to each sample $\vx^\nu$ a set of $m$ noise realizations
$\{\vxi^{\nu\alpha}\}_{\alpha=1,\ldots,m}$%
, drawn once and then \emph{fixed} throughout
training, and consider the empirical loss
\begin{align}\label{eq:train_loss}
  \Ltrain(\vs) = \frac{1}{2dnm}\sum_{\nu=1}^n\sum_{\alpha=1}^m
    \left\lVert\vs(\vY^{\nu\alpha})+\frac{\vxi^{\nu\alpha}}{\sqrt{\Delta_t}}\right\rVert^2,
    \qquad \vY^{\nu\alpha}=e^{-t}\vx^\nu+\sqrt{\Delta_t}\,\vxi^{\nu\alpha}.
\end{align}
While this resembles a regression problem, it is a non-standard one: the score is evaluated
on $N=nm$ points $\vY^{\nu\alpha}$ that are \emph{correlated}, forming clusters of $m$
points around each clean sample $e^{-t}\vx^\nu$. The number of repetitions $m$ is the parameter that controls this cluster structure, and it
interpolates between two problems that are usually treated separately: at $m=1$ every training input carries its own independent noise, the clusters are
  single points, and \eqref{eq:train_loss} is an ordinary supervised kernel regression on
  $n$ examples with i.i.d.\ label noise; as $m\to\infty$, instead, the within-cluster average reproduces the exact noise expectation,
  $\Ltrain\to \mathcal{L}_{\mathrm{DSM}}$, and the minimizer is the empirical score.
Generalization is measured by the standard test loss
\begin{align}\label{eq:test_loss}
  \Ltest(\vs) = \frac{1}{2d}\,\mathbb{E}_{\vx,\vxi}\!
    \left[\left\lVert\vs(e^{-t}\vx+\sqrt{\Delta_t}\,\vxi)+\frac{\vxi}{\sqrt{\Delta_t}}\right\rVert^2\right],
\end{align}
where the expectation is taken over both fresh noises and fresh samples.
We further decompose the test loss over the realizations of the learned score as
$\Ltest=\tfrac12\left(C_t+\mathcal{B}^2+\mathcal{V}\right)$, with a bias term
$\mathcal{B}^2$ and a variance term $\mathcal{V}$ given by
\begin{align}
\label{eq:bias_variance}
    \mathcal{B}^2=\frac{1}{d}\,\mathbb{E}_{\vy}\!\left[\lVert \vs_{\mathrm{exact}}(\vy)-\langle \vs_{\mathcal{D},\vtheta_0}(\vy)\rangle\rVert^2\right], \quad
    \mathcal{V}=\frac{1}{d}\,\mathbb{E}_{\vy}\!\left[\left\langle \lVert \vs_{\mathcal{D},\vtheta_0}(\vy)-\langle \vs_{\mathcal{D},\vtheta_0}(\vy)\rangle\rVert^2\right\rangle\right],
\end{align}
where $\vs_{\mathcal{D},\vtheta_0}$ denotes the minimizer of the training loss $\Ltrain$ for a given dataset $\mathcal{D}$ and initialization $\vtheta_0$ of the parameters. Here, $\mathbb{E}_{\vy}$ is the average over test samples $\vy\sim P_t$, whereas $\langle\cdot\rangle$ denotes the average over the realizations of the training set, and of other sources of randomness such as $\vtheta_0$. The constant $C_t$ depends only on the data distribution and is
independent of the learned score (see Appendix~\ref{app:bv_decomposition}). \looseness=-1

In practice the score is a neural network $\vs_{\vtheta}$ and \eqref{eq:train_loss} is
minimized by variants of gradient descent \citep{robbins1951stochastic,kingma2015Adam}.
Here we focus on the gradient flow $\dd\vtheta/\dd\tau=-d^2\nabla_{\vtheta}\Ltrain(\vtheta)$
of the training loss \eqref{eq:train_loss};
the factor $d^2$ is chosen so that the high-dimensional
limit is well defined. We shall occasionally add to the loss \eqref{eq:train_loss} an $L_2$ penalty
$\tfrac{\gamma}{2nmd}\lVert\vtheta\rVert^2$ of strength $\gamma$; as shown in
Appendix~\ref{app:subsec:regularization} this amounts to shifting the Gram matrix
introduced below by $\gamma\vI_N$. The relevant scale
for that ridge is $nm/d$, and we will use the rescaled ridge
$\tilde\gamma=\gamma\,d/(nm)$.
The corresponding evolution of the function $\vs_{\vtheta}$ reads
\begin{align*}
    \dot{\vs}_{\vtheta}(\vy)
    &= -d^2\sum_{\nu,\alpha}K_\tau(\vy,\vY^{\nu\alpha})\nabla_{\vs}\Ltrain(\vs_{\vtheta}(\vY^{\nu\alpha}))=\frac{-d}{nm}\sum_{\nu,\alpha}K_\tau(\vy,\vY^{\nu\alpha})
      \left(\vs_{\vtheta}(\vY^{\nu\alpha})+\frac{\vxi^{\nu\alpha}}{\sqrt{\Delta_t}}\right),
\end{align*}
for any test point $\vy\in\mathbb{R}^d$, where
$K_\tau(\vx,\vy)=\nabla_{\vtheta}\vs_{\vtheta}(\vx)\cdot\nabla_{\vtheta}\vs_{\vtheta}(\vy)$
is the neural tangent kernel (NTK) of the model \citep{Jacot_2018,
chizat2020lazytrainingdifferentiableprogramming, Geiger_2020}. Taken at pairs of training
points, the NTK defines the $N\times N$ {\bf{Gram matrix}}
$\vG^{\nu\alpha,\mu\beta}=K_\tau(\vY^{\nu\alpha},\vY^{\mu\beta})$, which closes the dynamics
on the training set and hence dictates the whole trajectory.\footnote{In practice, diffusion models are trained with fresh noise at every step, i.e.\ online in the noise and offline in the clean data. The training dynamics are then governed by the NTK acting as an integral operator on the empirical noisy distribution $\hat P_t=\frac1n\sum_{\nu=1}^n\mathcal{N}(e^{-t}\vx^\nu,\Delta_t\vI_d)$ \citep{Jacot_2018}, a setting recovered by the large-$m$ limit of our results.\looseness=-1} It is therefore natural to
expect the distinct timescales of generalization and memorization to be encoded in its
spectrum \citep{yao2007early,canatar2021spectral,bonnaire2025diffusionmodelsdontmemorize,Favero2025_bigger}, which makes it possible to probe the mechanism underpinning memorization. \looseness=-1

{\bf {Main contributions.}}
The picture supported by our results is that the Gram matrix of a diffusion model
is a different object from the one encountered in standard supervised learning.
In the latter case each training point carries its own independent
input, whereas in denoising score matching each clean sample is presented with
$m>1$ independent noise realizations, so the Gram matrix is built on clusters of
correlated points. This repetition leads to new features. It
restructures the spectrum of the Gram matrix into two well-separated parts: a
\emph{generalization} part of large eigenvalues, already present in standard supervised learning, whose modes encode the global features of the target
distribution; and a new \emph{memorization} part, made of the smallest eigenvalues
and absent at $m=1$, whose modes encode the sample-specific noise directions.
Since the training dynamics learns spectral modes in order of decreasing eigenvalue, this
gap translates directly into the separation of timescales between
$\tau_{\mathrm{gen}}$ and $\tau_{\mathrm{mem}}$. We establish this picture
analytically and numerically.
\begin{itemize}
\item \textbf{Theoretical characterization of the spectral hierarchy.} We provide
a precise characterization of the spectral properties of the Gram matrix in the
high-dimensional limit $d\gg1$, in the lazy regime \citep{chizat2020lazytrainingdifferentiableprogramming}, covering both the linear ($n\asymp d$) and
polynomial ($n\asymp d^k$) sample complexity regimes (see Fig.~\ref{fig:spectrum}). In both, the generalization
part is associated with the timescales over which the model captures the global
features of the target distribution by learning successive approximations of the
score function, while the memorization part leads to a timescale that is
parametrically larger in the training set size $n$, effectively delaying the
onset of memorization.
\item \textbf{Bias--variance decomposition.} In the linear regime, we derive
closed-form equations for the bias and the variance of the score estimator in the
large-$d$ limit. We find that, contrary to the supervised learning case $m=1$, in
diffusion models where each sample is noised $m>1$ times, the behavior of the bias
changes: the contribution of the memorization part leads to a $\Theta(1)$ bias
instead of vanishing at large sample complexity.
\item \textbf{Validation on realistic settings.} We validate our predictions on
CelebA images using Convolutional NTKs in the lazy regime
and finite-width U-Net architectures. We
show that the two-part structure persists beyond the lazy regime. To assess the
role of the two parts of the spectrum, we show that the truncation rank $r$ of
the Gram matrix is a control knob for the generalization--memorization
transition, and that a targeted $L_2$ regularization of the memorization part
suppresses memorization in feature-learning U-Nets, establishing a direct link
between this spectral component and memorization.
\end{itemize}

{\bf Related work.}
State-of-the-art image models are known to
reproduce a non-negligible part of their training data
\citep{Carlini_2023, somepalli_2022, somepalli_2023}. Further works examined how this phenomenon is influenced by factors such as data distribution, model architecture, and training procedure \citep{gu2023memorization, yoon2023diffusion, bonnaire2025diffusionmodelsdontmemorize, Favero2025_bigger}, as well as which samples are memorized first \citep{merger2026_coverage}. The emergence of memorization in the reverse generative trajectories has been analyzed under the exact empirical score hypothesis by \citet{Biroli_2024, achilli2024, ventura2025}, who showed that avoiding collapse onto the training data requires the sample complexity to scale exponentially with the intrinsic data dimension \citep{Achilli_2025_manifold, george_2025_manifold}. To explain why diffusion models often avoid this collapse in practice, recent literature has uncovered several sources of regularization. It has been shown that architectural biases and limited network capacity constrain memorization \citep{Kamb2024, kadkhodaie_2024, baptista_2025, buchanan2025}, while a finite learning rate prevents the network from collapsing to the exact empirical score \citep{Wu2025}. Closely related to our setting, several works have highlighted the beneficial role of early stopping in preventing memorization \citep{li2025generalizationpropertiesdiffusionmodels, baptista_2025, bonnaire2025diffusionmodelsdontmemorize, Favero2025_bigger}, even though it may be insufficient \citep{garnierbrun2026}. The performance of neural networks in the lazy regime has been studied rigorously in the supervised learning setting \citep{Ghorbani_2021, mei2022Hypercontractivity, misiakiewicz2022spectrum}. These analyses rely on the spectrum of the Gram matrix associated with the NTK operator. In the polynomial regime $n\asymp d^k$, the non-linear Gram matrix is equivalent to a deterministic polynomial matrix whose spectrum can be analyzed rigorously \citep{elkaroui_2010, Lu_2025,pandit2025}. Recent works have begun to study how this picture changes in the diffusion setting. Closest to ours, \citet{latourellevigeant2026generalizationmemorizationoverfittingdiffusion} analyze denoising score matching in the lazy regime and also find a transition from generalization to memorization. They work, however, with the exact expectation over the noise and with the infinite-dimensional kernel operator rather than the Gram matrix. \citet{han2024neural} instead leverage these results to derive bounds on the training dynamics of two-layer neural networks on the score-matching task. An extended discussion can be found in
Appendix~\ref{app:sect:related_works}.\looseness=-1

\section{Analytical results}
\label{sect:Analytical}

\paragraph*{Notation and assumptions.}
We write $N=nm$ for the cardinality of the training set. $\vI_d$ is the identity in dimension $d$ and $\bm{1}_m$ the all-ones
vector in dimension $m$ (likewise $\vI_N$ and $\bm{1}_N$), and $\vB_m=\vI_n\otimes\bm{1}_m\bm{1}_m^\top$.
We set $\vSigma=\mathbb{E}_{P_0}[\vx\vx^\top]\in\mathbb{R}^{d\times d}$ for the
data covariance and $\vv_\lambda$ for its eigenvector associated with the eigenvalue
$\lambda$.
We assume that the data are of the form $\vx=\vSigma^{1/2}\vz$, where $\vz$ has independent zero-mean, unit-variance, sub-Gaussian entries, with a covariance $\vSigma$ such that $\lambda_{\mathrm{max}}(\vSigma) = O(1)$ and $\Tr(\vSigma)/d$ converges to a constant denoted by $\sigma^2$ as $d \to \infty$.
We focus on inner-product kernels of the form $K(\vx, \vy) = f(\vx^\top \vy / d)$, where $f$ is a smooth function. A canonical example is provided by the NTK of a multilayer neural network in the infinite-width limit, given in Appendix~\ref{app:appendix_NTK}. \looseness=-1

\paragraph*{Lazy regime.}
In the following we consider the infinite-width limit, with the initialization scaling for which the network
enters the \emph{lazy} regime \citep{Jacot_2018,
chizat2020lazytrainingdifferentiableprogramming, Geiger_2020}: the weights barely move,
the NTK freezes to a deterministic kernel $K_\tau=K$, and the dynamics above becomes
linear. Assuming that $\vs_{\vtheta(0)}=0$ at initialization, it integrates to (see Appendix~\ref{app:ntk:flow})
\begin{align}
\label{eq:estimation_score_DSM}
    \vs_{\vtheta(\tau)}(\vx) = K(\vx, \vY)^\top \vG^{-1}
    \left( \vI_N - e^{-\frac{d \tau}{nm} \vG} \right)
    \left( -\frac{\vxi}{\sqrt{\Delta_t}} \right),
\end{align}
where $K(\vx,\vY)\in\mathbb{R}^{N}$ collects the kernel evaluations between the test point $\vx$ and the training set. The filter $\vI_N-e^{-d\tau\vG/nm}$ activates spectral modes in decreasing eigenvalue order, so fixing a training time $\tau$ is equivalent to imposing a spectral cutoff $\lambda_c\sim nm/(d\tau)$. A bulk of eigenvalues at $\lambda$ is therefore learned on the timescale
$\tau_\lambda\sim nm/(d\lambda)$, and a gap in the spectrum becomes a separation of timescales. \looseness=-1

\subsection{Linear regime: structure of the Gram matrix}
\label{sect:linear}

\paragraph*{Linear equivalent.} Following \citet{elkaroui_2010}, we extend the linear equivalence of non-linear Gram matrices to correlated data points $\{\vY^{\nu\alpha}\}_{1\le\nu\le n}^{1\le\alpha\le m}$. As $n,d\to\infty$ with $n/d\to\psi_n$ and $m=O(1)$, the Gram matrix $\vG^{\nu\alpha,\mu\beta}=f(\vY^{\nu\alpha\top}\vY^{\mu\beta}/d)$ is spectrally equivalent to

\begin{align}\label{eq:Glin}
  \vG_{\mathrm{lin}} = \mu_I\vI_N + \mu_B\vB_m + \mu_0\bm{1}_N\bm{1}_N\tran
    + \mu_1\frac{\vY\tran\vY}{d},
\end{align}
with more details given in Appendix~\ref{app:thm:lin_gram_empirical}. Crucially, in comparison to the standard case $m=1$, the noisy repetition of clean data points induces a novel regularization term $\mu_B \vB_m$. This term accounts for the intra-sample correlations across different noise realizations and is the primary driver of the spectral separation between generalization and memorization. Figure~\ref{fig:spectrum} (\textit{middle}) shows, at $m=4$, the convergence of the spectrum of the non-linear Gram matrix (colored continuous lines) to the linear equivalent \eqref{eq:Glin} (black dashed line) as $d$ grows. \looseness=-1

\paragraph*{Spectral properties of the Gram matrix.} We now characterize the spectral properties of the linearized Gram matrix. The term $\mu_0 \bm{1}_N \bm{1}_N^\top$ induces a single outlier eigenvalue at $nm\mu_0$, which is not relevant for our results and will not be discussed further. The remainder of the spectrum is determined by the Stieltjes transform $q(z) = \frac{1}{N} \Tr(\vG - z\vI_N)^{-1}$. Using the replica method \citep{mezard1987spin}, we derive a set of self-consistent equations for $q(z)$ in Appendix~\ref{sect:proof_replica}. They allow us to characterize the spectrum of the Gram matrix in the high-dimensional regime. \looseness=-1

\begin{figure}[t]
    \centering
    \includegraphics[width=\linewidth]{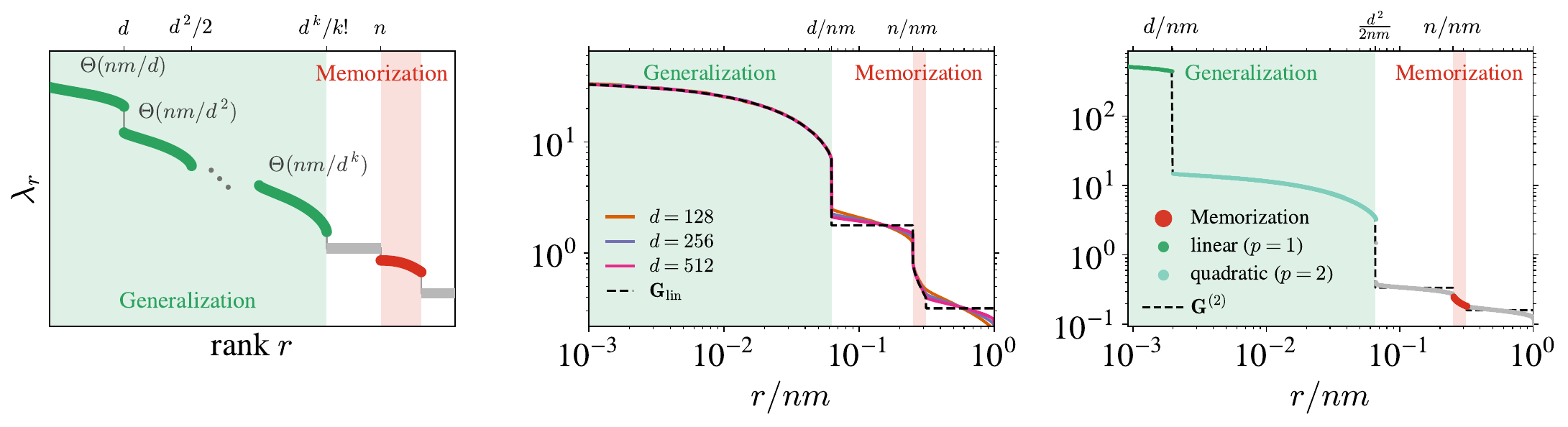}
    \caption{\textbf{Gram matrix spectrum.} Eigenvalues $\lambda_r$ in decreasing order against their rank. (\emph{Left}) Sketch of the spectrum predicted by Theorem~\ref{thm:poly} for the Gram matrix $\vG^{\nu\alpha,\mu\beta}=f(\vY^{\nu\alpha\top}\vY^{\mu\beta}/d)$. (\emph{Middle}) Spectrum of $\vG$ with
    isotropic $\vSigma=\sigma^2\vI_d$, $t=0.1$, $m=4$, $\psi_n=4$ and $f$ the NTK of a two-layer ReLU network (Appendix~\ref{app:kernels_used}) in the linear regime at several values of $d$; the dashed black curve is the linear equivalent $\vG_{\mathrm{lin}}$ of \eqref{eq:Glin} at $d=512$. (\emph{Right}) Quadratic regime $n=2d^2$ at $d=64$; the dashed black curve is the polynomial equivalent of Theorem~\ref{thm:poly}. Details in Appendix~\ref{appendix:polynomial_scaling}.\looseness=-1}
    \label{fig:spectrum}
\end{figure}

\begin{thm}[Spectrum of the Gram matrix]
\label{thm:Spectrum_gram_linear}
Let the empirical spectral density of $\vSigma$ converge to $\rho_{\vSigma}$, and take $d,n\to\infty$ at fixed $\psi_n=n/d$, then $\psi_n\gg1$. For each eigenvalue $\lambda$ of $\vSigma$, let $\lambda_t=e^{-2t}\lambda+\Delta_t$. The spectrum of $\vG-\mu_0\bm{1}_N\bm{1}_N\tran$ consists of two bulks of $d$ eigenvalues each:\looseness=-1
\begin{itemize}[nosep,leftmargin=*]
  \item \textbf{generalization bulk:} $\lambda_{\mathrm{gen}} \simeq \mu_1\psi_n m\lambda_t = \Theta(nm/d)$,
  \item \textbf{memorization bulk:} $\lambda_{\mathrm{mem}} \simeq \mu_I + \mu_B(m-1)\Delta_t/\lambda_t = \Theta(m)$,
\end{itemize}
and two atoms: $\mu_I$ with weight $(m-1-1/\psi_n)/m$ and $\bar\mu=\mu_I+m\mu_B$ with weight $(1-1/\psi_n)/m$.\looseness=-1
\end{thm}

While the generalization bulk is present in standard kernels ($m=1$), the multi-noise setting ($m>1$) triggers the emergence of the memorization bulk, cf.\ Fig.~\ref{fig:spectrum}. To understand why these components correspond to distinct learning regimes, we must first examine their associated eigenvectors. \looseness=-1

\begin{proposition}[Eigenvectors of the Gram matrix]
\label{prop:eigenvectors_gram_linear}

For $\psi_n\gg1$ and every eigenvalue $\lambda$ of $\vSigma$ with associated eigenvector $\vv_\lambda$, the eigenvectors of $\vG$ associated with the two bulks of Theorem~\ref{thm:Spectrum_gram_linear} are asymptotically\looseness=-1
\begin{align}\label{eq:eigenvectors}
    \vu_1^{\nu\alpha}\propto \vv_\lambda^\top\vY^{\nu\alpha},
    \qquad
    \vu_2^{\nu\alpha}\propto \vv_\lambda^\top\big[(me^{-2t}\lambda+\Delta_t)\sqrt{\Delta_t}\,\vxi_\perp^{\nu\alpha}
      -(m-1)\Delta_t\,\bar{\vY}^{\nu}\big],
\end{align}
with $\bar{\vY}^\nu=\frac1m\sum_\beta\vY^{\nu\beta}$ and
$\vxi_\perp^{\nu\alpha}=\vxi^{\nu\alpha}-\frac1m\sum_\beta\vxi^{\nu\beta}$. The eigenvectors
of the two atoms are the $\vvarphi\otimes\bm{1}_m/\sqrt m$ with $\vvarphi\in\mathrm{Ker}(\bar{\vY})$
and the elements of $\mathrm{Ker}(\vY)\cap\mathrm{Ker}(\vB_m)$; they do not contribute to the
estimated score.\looseness=-1
\end{proposition}
Theorem~\ref{thm:Spectrum_gram_linear} and Proposition~\ref{prop:eigenvectors_gram_linear} reveal a spectral hierarchy.
 The eigenvectors of the two bulks have a transparent interpretation. The generalization modes $\vu_1$ are the projections of the training points on the principal directions $\vv_\lambda$ of the data covariance, so learning them fits the component of the score carried by the population covariance, exactly as a standard kernel does at $m=1$. The memorization modes $\vu_2$ contain $\vxi_\perp$, the noise components that distinguish the different realizations of the same data point, capturing the irregularity of the empirical score close to each data point. These two bulks are separated by a wide spectral gap: the \emph{generalization} bulk scales as $\Theta(nm/d)$, while the \emph{memorization} bulk remains $\Theta(m)$.
Because of the filtering of Eq.~\eqref{eq:estimation_score_DSM}, this gap enforces a strict separation of timescales: the generalization modes are learned first whereas the memorization ones are learned much later, on a timescale that is parametrically larger in $n$.\looseness=-1

\subsection{Linear regime: bias--variance decomposition}
\label{sect:biasvariance}

\begin{figure}[t]
  \centering
  \includegraphics[width=0.9\linewidth]{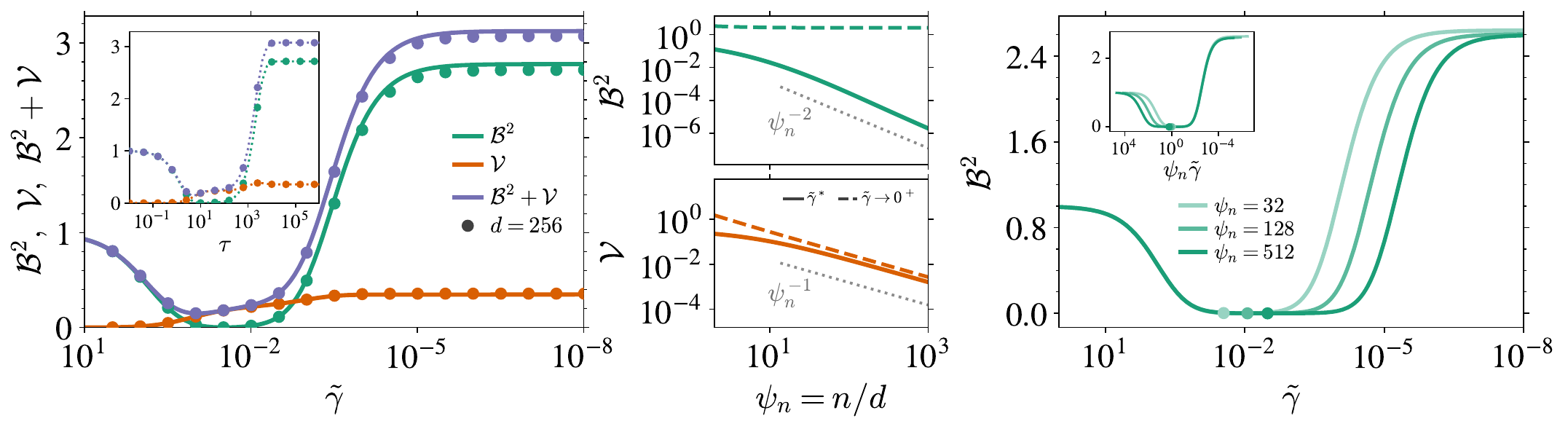}
  \caption{
   \textbf{Bias--variance decomposition.} $\mathcal{B}^2$, $\mathcal{V}$ and $\mathcal{B}^2+\mathcal{V}$ for the kernel $f(u)=\mathbb{E}[\tanh(z_1)\tanh(z_2)]$, with $(z_1,z_2)$ unit-variance Gaussians of correlation $u$ (Appendix~\ref{app:kernels_used}), at $\sigma^2=1$, $m=4$ and $t=0.1$. (\emph{Left}) Vs.\ the rescaled ridge $\tilde\gamma$ at $\psi_n=8$, axis reversed to show the correspondence with training time; the \textit{inset} shows the same vs.\ the training time $\tau$ of \eqref{eq:estimation_score_DSM}. Solid lines: Theorem~\ref{thm:bias_variance}; markers: experiments at $d=256$; in the inset, markers and dotted line are the finite-$d$ gradient flow. (\emph{Middle}) $\mathcal B^2$ (top) and $\mathcal V$ (bottom) vs.\ $\psi_n$, for the optimal ridge $\tilde\gamma^*=\arg\min_{\tilde\gamma}\Ltest$ (solid) and the ridgeless limit $\tilde\gamma\to0^+$ (dashed). (\emph{Right}) $\mathcal B^2$ vs.\ $\tilde\gamma$ for several $\psi_n$, dots marking $\tilde\gamma^*$; the \textit{inset} shows the same vs.\ $\psi_n\tilde\gamma$. Details in Appendix~\ref{app:subsec:krr_bv}. \looseness=-1
  }
\label{fig:bias_variance_m_1_m_4}
  \label{fig:bv_time}
  \label{fig:bv_psin}
\end{figure}

To get a sharper understanding of the phenomenon we now turn to the bias--variance decomposition of the estimator of the score. In the high-dimensional limit, we derive closed-form equations
for the bias and the
variance of the kernel ridge predictor,
\begin{align}\label{eq:krr_predictor}
    \vs_\gamma(\vy) = K(\vy,\vY)^\top\big(\vG+\gamma\,\vI_N\big)^{-1}
    \left(-\frac{\vxi}{\sqrt{\Delta_t}}\right),
\end{align}
where $\vY$ and $\vxi$ are the noisy data and noises used during training. \looseness=-1
\begin{thm}[Bias and variance of the kernel ridge predictor]
\label{thm:bias_variance}
In the limit $n,d\to\infty$ at fixed $\psi_n=n/d$, with $m=O(1)$ and $\vSigma=\sigma^2\vI_d$, the bias and the variance \eqref{eq:bias_variance}
of the kernel ridge predictor \eqref{eq:krr_predictor} concentrate, and are given by a
closed system of algebraic equations stated in
Appendix~\ref{app:bv_fixed_point}.\looseness=-1
\end{thm}

We use the ridge regularization as a proxy for the gradient flow dynamics in $\tau$,
studying the evolution of the equations of Theorem~\ref{thm:bias_variance} as a function
of $\tilde\gamma$ through the equivalence $\tilde\gamma \Leftrightarrow 1/\tau$ \citep{ali2019_earlystopping}. As shown in the left panel of Fig.~\ref{fig:bias_variance_m_1_m_4}, both ridge and gradient flow
show qualitatively the same evolution; we therefore study the closed-form equations on the kernel ridge
estimator as a function of $\tilde\gamma$ and infer the training-time behavior from it. Solving these equations numerically gives the scalings of the bias and the variance as a function of $\psi_n$ and $t$ that can be read off
Fig.~\ref{fig:bias_variance_m_1_m_4} (\textit{middle}) and Fig.~\ref{fig:bv_tscaling}.

\begin{result}[Scalings of the bias and the variance]
\label{res:bv_scalings}
Let $m>1$ and $1/(m\psi_n)\ll t\ll1$.
\begin{itemize}[leftmargin=1.2em,itemsep=1pt,topsep=-2pt]
  \item \emph{Ridgeless}, $\tilde\gamma\to0^+$:
  $\mathcal B^2=\Theta(t^{-2})$, independently of $\psi_n$, and
  $\mathcal V=\Theta(\psi_n^{-1}t^{-2})$.
  \item \emph{Optimal ridge}, $\tilde{\gamma}^*$:
  $\mathcal B^2=\Theta(\psi_n^{-2}t^{-2})$ and
  $\mathcal V=\Theta(\psi_n^{-1}t^{-1})$.
\end{itemize}
\end{result}

At the end of training, i.e.\ $\tilde{\gamma}\to 0^+$, the test loss is dominated by a persistent bias\footnote{There is a fundamental difference between the statistics of the score estimator on a test set and those of the generated distribution by integrating the backward process. At large $m$, the latter is the empirical one: it averages to the population distribution, hence has zero bias, but depends strongly on the training set, hence has large
variance.\looseness=-1} that does not vanish with the sample complexity and that diverges as the diffusion time $t\to0$. On the other hand, when the model is optimally early-stopped (equivalently, at the optimal ridge $\tilde\gamma^*$), the bias vanishes with the sample complexity at fixed $t$, like the variance, as illustrated in Fig.~\ref{fig:bias_variance_m_1_m_4} (\textit{middle}). \looseness=-1
In the right panel of Fig.~\ref{fig:bias_variance_m_1_m_4}, the bias starts to
rise at $\tilde\gamma\propto1/\psi_n$ (\textit{inset}), i.e.\ when the ridge
$\gamma$ falls to the $O(1)$ eigenvalues of the memorization bulk. Below this
value the ridge is too small to damp these modes: the regularization no longer acts on the memorization bulk, which is the source of the persistent bias.
\looseness=-1

\subsection{Polynomial scaling}
\label{sect:poly}

These results can be extended beyond the linear regime under a few additional assumptions (see Appendix~\ref{appendix:polynomial_scaling}). In the polynomial
regime $d^k\ll n\ll d^{k+1}$, the analogue of the linear equivalent \eqref{eq:Glin} is a polynomial equivalent, with the spectrum below, sketched in Fig.~\ref{fig:spectrum} (\textit{left}) and shown at finite $d$ in Fig.~\ref{fig:spectrum} (\textit{right}).\looseness=-1
\begin{thm}[Spectrum of the Gram matrix, polynomial regime]
\label{thm:poly}
In the limit $d,n\to\infty$ with $d^k\ll n\ll d^{k+1}$ and $m=O(1)$, the spectrum of $\vG$
consists, besides two atoms (see Appendix~\ref{appendix:polynomial_scaling}), of:\looseness=-1
\begin{itemize}[leftmargin=1.2em,itemsep=1pt,topsep=-2pt]
  \item for each degree $p\in\{1,\dots,k\}$, a group of $\Theta(d^p/p!)$ eigenvalues
  $\Theta(nm/d^p)$
  (\textbf{generalization} sub-bulks);
  \item a bulk of $\Theta(d^k/k!)$ eigenvalues
  $\Theta(m)$, independent of $n$ and absent at $m=1$ (\textbf{memorization} bulk).
\end{itemize}
\end{thm}
See Appendix~\ref{appendix:polynomial_scaling} for the precise statement. Since a mode of
eigenvalue $\lambda$ is learned on the timescale $\tau_\lambda\sim nm/(d\lambda)$, the
generalization sub-bulks are learned degree by degree, on timescales $\tau_p\sim d^{p-1}$:
at $m=1$ this is the classical hierarchical picture of kernel regression \citep{Ghorbani_2021}, in which
low-order moments of the data are learned first and higher-order statistics later
\citep{ricci2025feature, bardone2026}. Data repetition ($m>1$) preserves this hierarchy but
adds the memorization bulk at $\Theta(m)$, learned on $\tau_{\mathrm{mem}}\sim n/d$. As in
the linear regime, memorization is delayed relative to every degree the model can learn, by
a factor that grows with the sample complexity.
Note that, strictly speaking, full memorization---the ability of the learned
score to reproduce the training samples---requires $n$ to grow faster than any
power of $d$, i.e.\ $n\gg d^k$ for every fixed $k$. At any finite $k$ only an
approximation of the empirical score can be learned, and memorization and
overfitting are correspondingly partial \citep{misiakiewicz2022spectrum,latourellevigeant2026generalizationmemorizationoverfittingdiffusion}.\looseness=-1

\section{Experiments}
\label{sect:Numerical}
\label{sect:experiments}

\begin{figure}
    \centering
    \includegraphics[width=0.403\linewidth]{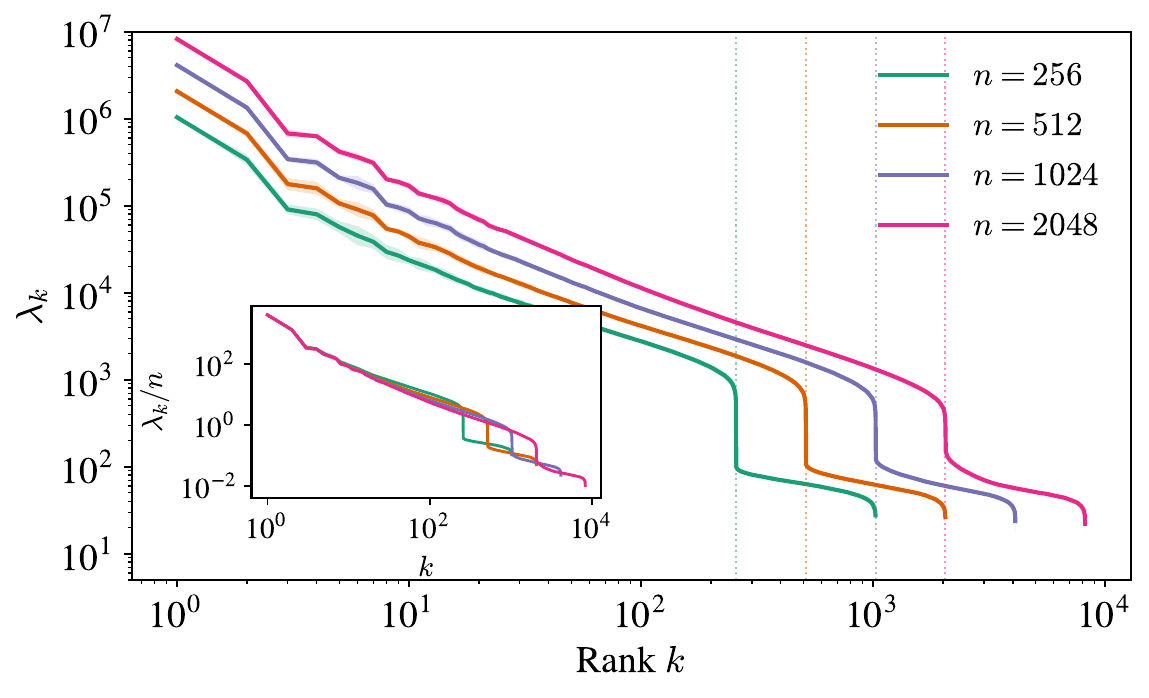}
    \includegraphics[width=0.392\linewidth]{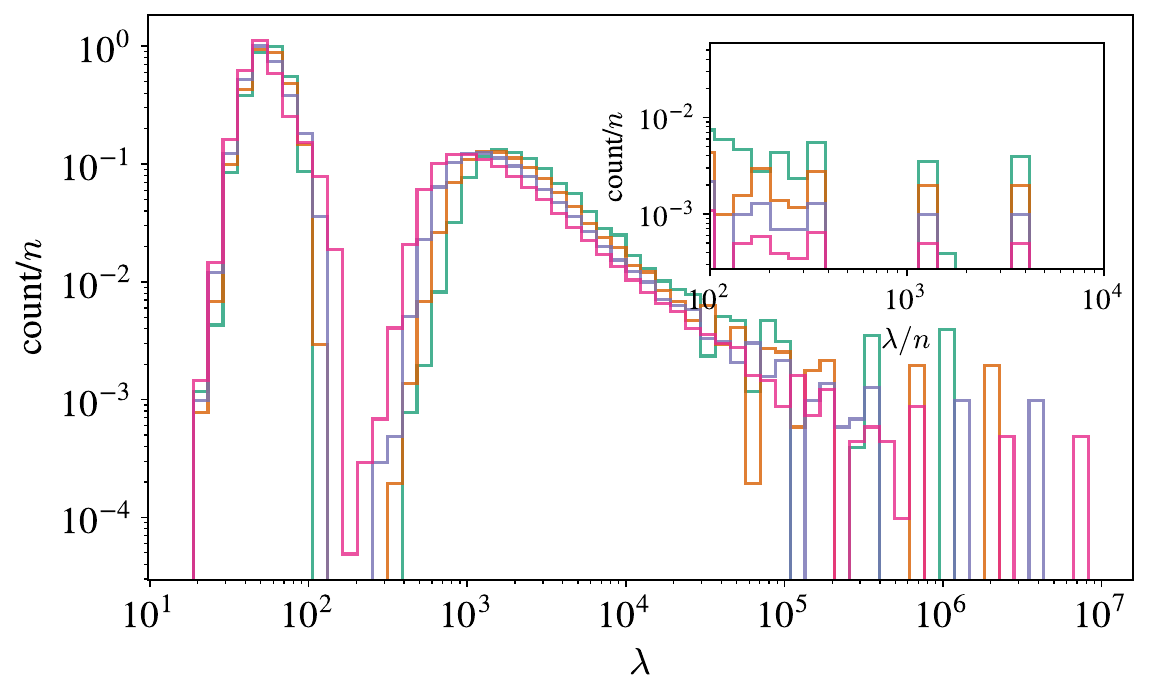}
    \caption{\textbf{Structure of the generalization and memorization bulks in the CNTK spectrum on CelebA.} \emph{(Left)} Ordered eigenvalues for several $n$ at $m=4$. The dotted lines indicate $r=n$. Inset: Same, rescaled by $n$.
    \emph{(Right)} Eigenvalue histogram with counts rescaled by $n$. Inset: collapse of the top eigenvalues under rescaling by $n$. $s=0.05$, averaged over 10 training sets.\looseness=-1}
    \label{fig:cntk_scaling}
\end{figure}

\paragraph*{Data, noising process, and model.} We work with the CelebA face dataset \citep{Liu2015_CelebA}, which we convert to grayscale, downsample to $32\times32$ pixels, and standardize. For computational convenience, our experiments adopt a flow-matching parametrization \citep{liu2022rectified, lipman_2023, albergo2023stochastic} in which the noising process is the linear interpolant\footnote{This differs from the Ornstein--Uhlenbeck interpolant of Sect.~\ref{sect:Analytical} only by a time reparametrization and a scaling of the input and the target, so that the analytical results remain unchanged (see Appendix~\ref{app:OU_vs_RF}). \looseness=-1}
$\vY^{\nu\alpha}_s = (1-s)\,\vx^{\nu} + s\,\vxi^{\nu\alpha}$, with $s\in [0,1]$, between a clean data point $\vx^\nu \sim P_0$ and an independent Gaussian noise $\vxi^{\nu\alpha}\sim\mathcal{N}(0,\vI_{d})$. The regression targets are the conditional velocities $\vxi^{\nu\alpha}-\vx^\nu$. Throughout the section, we use $n$ clean images and $m$ independent noise realizations per data point such that the empirical Gram matrix $\vG^{\nu\alpha,\mu\beta} = K(\vY^{\nu\alpha}_s,\vY^{\mu\beta}_s)$ has size $N\times N$ with $N=nm$. We take the velocity model to be a depth-$D$ vanilla convolutional neural network (CNN) with $3 \times 3$ filters and ReLU activations, and a final readout layer. Our choice is motivated by the existence of a closed-form Convolutional Neural Tangent Kernel (CNTK) in the infinite-width (lazy) limit \citep{Arora2019}. This allows us to replace the network by its kernel predictor, whose Gram matrix is computed in closed form rather than from parameter gradients (see Appendix~\ref{app:ntk:cnn}).
We then test whether these spectral properties persist in the feature-learning regime by training a finite-width U-Net in Sect.~\ref{sect:UNET}. \looseness=-1

\paragraph*{Gram-matrix eigendecomposition and truncated kernel regression.} At every flow time $s$, we approximate the top-$r$ eigenpairs of the Gram matrix $\vG$ via the randomized method of \citet{Halko2011} in which, for memory efficiency, the full Gram matrix is never materialized. Keeping only the top $r$ modes of $\vG$ defines the spectrally truncated estimator $\hat{\vv}_r$ of the velocity field $\mathbb E[\vxi-\vx_0\mid\vY_s=\vx]$, and approximately corresponds to early stopping of training in the lazy regime at $\tau\sim nm/(d\lambda_r)$, cf.~Eq.~\eqref{eq:estimation_score_DSM}. More details can be found in Appendix~\ref{app:subsec:cntk}.\looseness=-1

\paragraph*{Metrics.} To quantify the ability of the approximated velocity field to memorize, we adopt the first-to-second nearest-neighbor ratio criterion used in several previous studies \citep{yoon2023diffusion, gu2023memorization, bonnaire2025diffusionmodelsdontmemorize} and compute the fraction $f_{\mathrm{mem}}\in[0,1]$ of samples $\tilde{\vx}_0$ generated by the backward dynamics that are memorized.
We initialize the trajectories at $\tilde{\vx}_1 \in \{\vxi^{\mu\alpha}\}$; our experiment can therefore be viewed as a long-time version of the U-turn protocol of \citet{sclocchi_2024, Behjoo_2025}. As a measure of the quality of the generated images, we also report the Fréchet Inception Distance \citep[FID,][]{heusel2017gans} between 2,048 samples generated from fresh noise and 2,048 held-out test images. See Appendix~\ref{app:subsec:general} for more details. \looseness=-1

\subsection{Lazy regime: Gram matrix for real data}

\paragraph*{Structure of the Gram matrix spectrum in CelebA.}

Figure~\ref{fig:cntk_scaling} displays the spectrum of the CNTK Gram matrix computed from noised images of CelebA at fixed $s=0.05$. It exhibits the two-bulk structure predicted by Theorem~\ref{thm:Spectrum_gram_linear}, including a component absent in the standard $m=1$ case (see Appendix~\ref{app:subsec:cntk} for the scaling with $m$). The first bulk of large eigenvalues behaves like that of a classical Gram matrix: at fixed $m$, the top eigenvalues scale linearly with $n$ and collapse once rescaled (see insets).
This generalization bulk contains $n$ eigenvalues rather than the $d$ predicted for sub-Gaussian data in the proportional regime, a known feature of image statistics \citep{canatar2021spectral}. This likely reflects the strong anisotropy and low intrinsic dimension of natural images, whose covariance spectrum decays as a power law. For such
data, the bulks of Theorem~\ref{thm:poly} are not well separated and can merge \citep{wortsman2025kernelridgeregressionpowerlaw}. The key qualitative phenomenon nonetheless persists: the spectrum is partitioned into a first bulk scaling with $n$, and a second which is independent of it. We show below that they correspond to generalization and memorization, respectively.\looseness=-1

\begin{figure}[t]
    \centering
    \includegraphics[width=0.3525\linewidth]{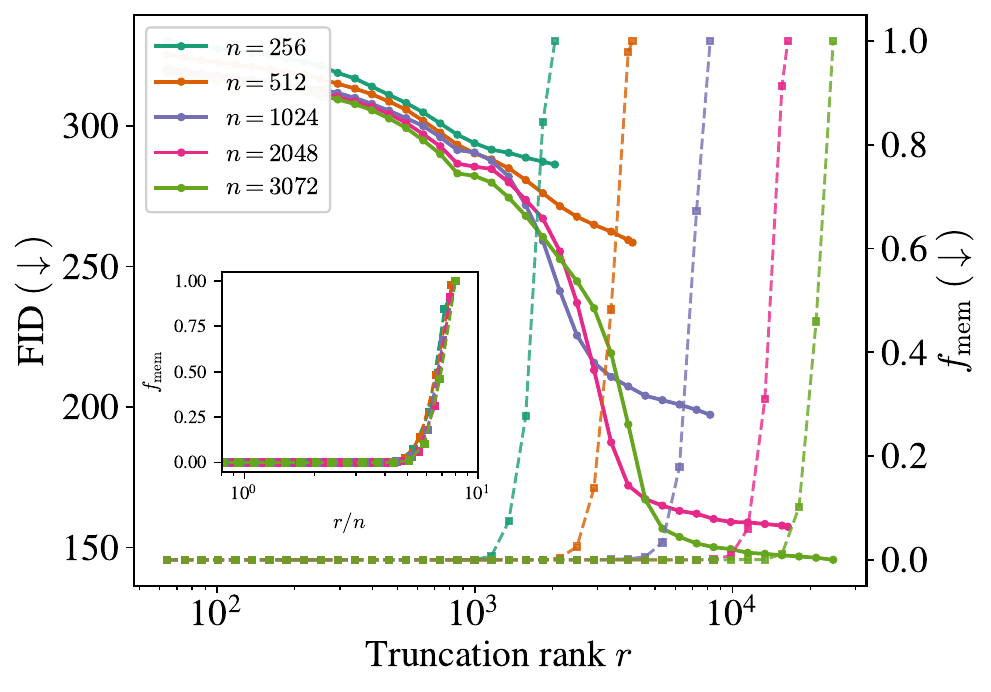}
    \includegraphics[width=0.3525\linewidth]{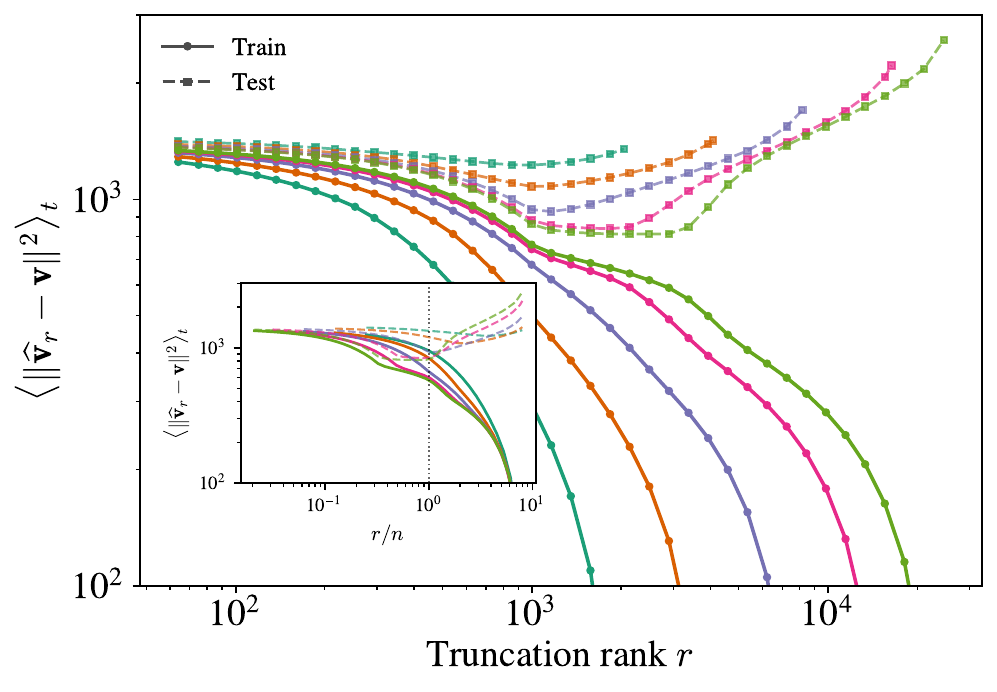}
    \includegraphics[width=0.21185\linewidth]{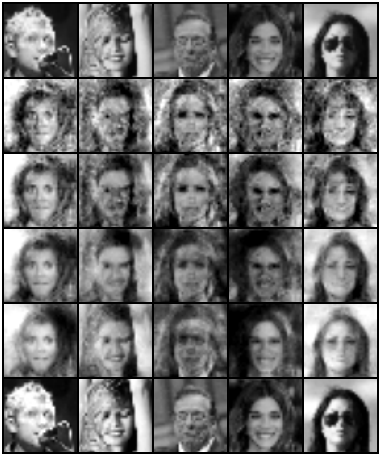}
    \caption{\textbf{Spectrally truncated kernel regression and generalization--memorization transition.}
    \emph{(Left)} FID (solid) and $f_\mathrm{mem}$ (dashed) vs. truncation rank $r$ for several $n$. Inset: $f_\mathrm{mem}$ vs. $r/n$.
    \emph{(Middle)} Time-averaged $L_2$ training (solid) and test (dashed) losses; the inset shows the same curves against $r/n$.
    \emph{(Right)} Top row: training images. Rows~2--6: generated samples ($n=3072$) starting from a fixed training noise associated with the top-row clean image at
    $r\in\{2048,\,4096,\,8192,\,16384,\,N\!=\!nm\}$.\looseness=-1
   }
   \label{fig:cntk_truncation}
\end{figure}

\paragraph*{Generalization--memorization transition in truncated kernel regression.}

Figure~\ref{fig:cntk_truncation} varies the truncation rank $r$ of the CNTK Gram matrix on CelebA for $n\in\{256,512,1024,2048,3072\}$ and $m=8$. At small $r$, the FID (left panel) of samples generated with the truncated velocity field $\hat{\vv}_r$ decreases steadily while $f_\mathrm{mem}$ (in inset) stays near zero. For small $n$, however, $f_\mathrm{mem}$ rises before the FID reaches its minimum: the estimator starts memorizing training points before it produces plausible images.
As $n$ grows, the FID drops at roughly the same $r$ but reaches lower minima, saturating beyond $n=2048$. In this regime, the rank-$r$ approximated velocity generates CelebA-like samples without copying any training point over a growing range of $r$. The threshold marking the onset of generalization is therefore essentially independent of $n$, whereas the memorization threshold scales linearly with $n$, as shown by the inset, where rescaling $r$ by $n$ makes all curves collapse. The train and test losses (middle panel) exhibit the same pattern: the training error decreases monotonically with $r$, while the test error develops a clean minimum around $r\approx n$ for $n$ sufficiently large, in agreement with Sect.~\ref{sect:Analytical}.
The right panel illustrates this transition at fixed $n=3072$: integrating the trajectories backward yields blurry reconstructions at small $r$, which sharpen as $r$ grows. \looseness=-1
Across an entire intermediate regime, the model produces \emph{new} samples, that differ from their paired training images (top row) in gender, accessories and hairstyles. At larger $r$, these features are progressively morphed into those of the associated training image, until eventually, at $r=N=nm$ (last row), every column reproduces its training image.
These three diagnostics agree on a transition from generalization to memorization as $r$ increases, confirming the role of the successive bulks of eigenvalues in the CNTK Gram matrix in separating the two phases. This behavior is in agreement with the previous findings of \citet{bonnaire2025diffusionmodelsdontmemorize}: as the two bulks of eigenvalues are learned on different timescales, they create a dynamical separation between a generalization and a memorization phase, whose timescale ratio is proportional to $n$. \looseness=-1

\begin{figure}
    \centering
    \includegraphics[width=0.35\linewidth]{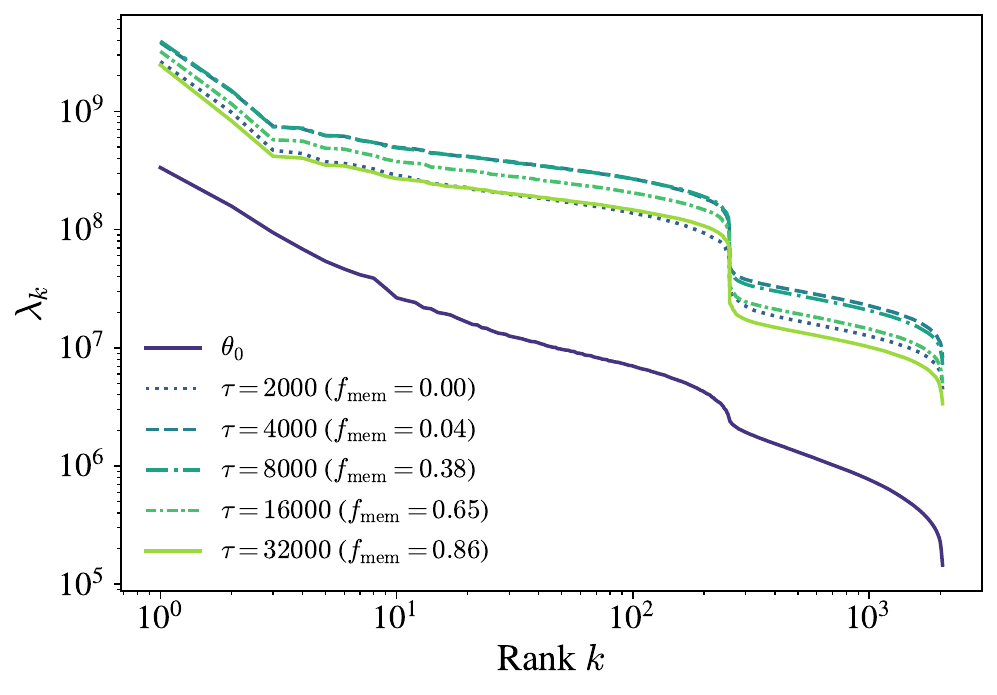}
    \includegraphics[width=0.35\linewidth]{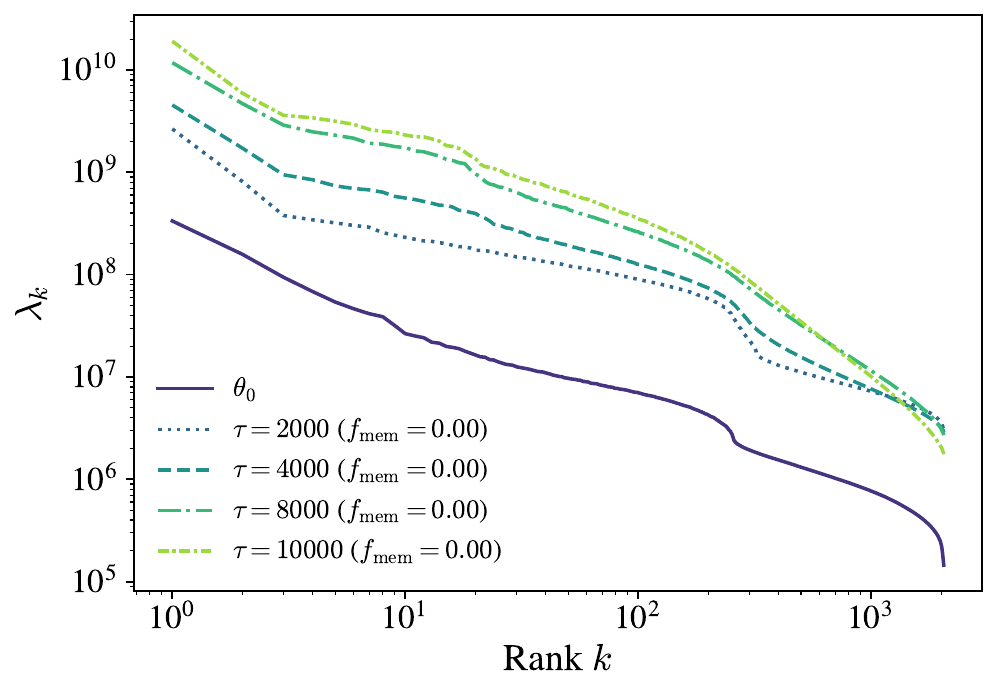}
    \caption{\textbf{Training-time-dependent empirical Gram eigenspectrum of a U-Net.} Eigenvalues of the empirical Gram matrix computed for several training times $\tau$ at $s=0.23$ for \emph{(Left)} the unregularized dynamics, and \emph{(Right)} the $L_2$-regularized dynamics.\looseness=-1
    }
    \label{fig:time_evolution_NTK}
\end{figure}

\subsection{Feature-learning regime and training-time-dependent NTK}
\label{sect:UNET}

We now investigate whether the spectral properties persist beyond the lazy regime, when the kernel evolves throughout training.
To this end, we train a finite-width U-Net in the feature-learning
regime on $n=256$ CelebA images with $m=8$ and track the eigenvalues of its empirical NTK Gram matrix, shown in the left panel of Fig.~\ref{fig:time_evolution_NTK}.
At initialization, the spectrum displays the same two-bulk structure as the CNTK, with a gap at $r\approx n$. While eigenvalue amplitudes shift non-monotonically during training, the bulk separation persists and the spectral gap even sharpens for $\tau >0$. Strikingly, even though training spans the entire generalization--memorization transition (with $f_\mathrm{mem}$ going from 0 to 0.86), the shape of the spectrum (two bulks with a gap at $n$) is preserved. This suggests that our analysis and the crucial role of the spectral bias we exhibit extend beyond the lazy regime. To further connect the second bulk to memorization, we perform an $L_2$-regularized training with a regularization parameter set to the $n$-th eigenvalue of the Gram matrix, tracked dynamically during training at each flow time $s$ (right panel of Fig.~\ref{fig:time_evolution_NTK}). This procedure therefore effectively damps the modes with eigenvalues below $\lambda_n$ and has two striking consequences: $f_\mathrm{mem}$ remains exactly zero over $10{,}000$ training steps (compared to $\approx0.5$ with no regularization) without sacrificing the test loss (see Appendix~\ref{app:subsec:unet}), and the second bulk gets progressively smoothed out during training. This is consistent with memorization being indeed controlled by this second component and provides a spectral interpretation of the findings of \citet{baptista_2025}, according to which regularization suppresses memorization in diffusion models. \looseness=-1

\section{Conclusion and limitations}
\label{sect:conclusion}
We analyze the spectrum of the Gram matrix that governs the lazy training dynamics of diffusion models and identify two well-separated bulks: a \emph{generalization} bulk, already present in standard kernel regression, and a \emph{memorization} bulk created by the repetition of each training sample with several noise realizations---specific to diffusion models. The two bulks are separated by a spectral gap that grows with the training set size $n$, leading to a separation of training timescales. We establish this structure analytically for both linear and polynomial sample complexities and, in the linear regime, derive closed-form equations for the bias and the variance: the generalization bulk yields a bias and a variance that vanish with the sample complexity, whereas the memorization bulk contributes a $\Theta(1)$ bias that has no counterpart in the supervised setting. Experiments on CelebA with the convolutional NTK exhibit the same spectral hierarchy, and a spectral truncation controls the generalization--memorization transition. The two-bulk structure further persists in a finite-width U-Net trained in the feature-learning regime, suggesting that this spectral mechanism is a general property of diffusion models. \looseness=-1

\paragraph*{Limitations.}
Our analytical results rely on several assumptions: the lazy regime, a fixed diffusion time $t$, and a finite number $m$ of noise realizations frozen throughout training.
While the experiments strongly suggest that the mechanism extends well beyond these hypotheses, a rigorous treatment of the feature-learning regime remains open. Finally, the gap between the predicted and observed numbers of
generalization eigenvalues calls for more realistic data models.\looseness=-1

\section*{Acknowledgment}

During the writing of this work, we became aware that Hugo Latourelle-Vigeant, Sinho Chewi, Aram-Alexandre Pooladian, John Sous and Theodor Misiakiewicz \citet{latourellevigeant2026generalizationmemorizationoverfittingdiffusion} had also been investigating the memorization–generalization transition
in the lazy regime. 
Their work is close to ours but quite complementary. Both works agree on the resulting theoretical picture. RU thanks Arie Wortsman for insightful discussions on kernel regression and anisotropic data. GB acknowledges support from the French government under the management of ANR: PEPR-IA (project MAGICALL ANR-25-PEIA-0004) and PR[AI]RIE-PSAI (ANR-23-IACL-
0008).
This work was performed using HPC resources from GENCI-IDRIS (Grants 2026-AD011016319R1 \& 2026-A0201016159).

\bibliography{bibliography}

\newpage

\begin{center}
    {\LARGE First Learn, Then Memorize: The Spectral Bias of Diffusion Models}\\[3ex]
    {\large Rapha\"el Urfin, Tony Bonnaire, Giulio Biroli, Marc M\'ezard}
\end{center}

\appendix

\providecommand{\vTheta}{\bm{\Theta}}

\providecommand{\veta}{\bm{\eta}}
\providecommand{\vPhi}{\bm{\Phi}}

\providecommand{\vvarphi}{\bm{\varphi}}
\providecommand{\vP}{\bm{\mathrm{P}}}
\providecommand{\vb}{\bm{\mathrm{b}}}
\providecommand{\vu}{\bm{\mathrm{u}}}
\providecommand{\ve}{\bm{\mathrm{e}}}
\providecommand{\vS}{\bm{\mathrm{S}}}
\providecommand{\vt}{\bm{\mathrm{t}}}
\providecommand{\vM}{\bm{\mathrm{M}}}
\providecommand{\vQ}{\bm{\mathrm{Q}}}
\providecommand{\vR}{\bm{\mathrm{R}}}
\providecommand{\vD}{\bm{\mathrm{D}}}
\providecommand{\vE}{\bm{\mathrm{E}}}
\providecommand{\vK}{\bm{\mathrm{K}}}
\providecommand{\vJ}{\bm{\mathrm{J}}}
\providecommand{\vC}{\bm{\mathrm{C}}}
\providecommand{\vO}{\bm{\mathrm{O}}}
\providecommand{\vH}{\bm{\mathrm{H}}}

\phantomsection
\addcontentsline{toc}{part}{Appendix}
{\centering\LARGE\bfseries Appendix\par}
\vspace{1.2em}

This appendix provides detailed derivations and additional results supporting the main text.
Sect.~\ref{app:appendix_NTK} reviews the Neural Tangent Kernel framework and gives
explicit expressions of the kernel for common architectures.
Sect.~\ref{app:OU_vs_RF} shows the equivalence between the Ornstein--Uhlenbeck process used
in Sect.~\ref{sect:Analytical} and the rectified-flow interpolant used in
Sect.~\ref{sect:Numerical}.
Sect.~\ref{app:sect:related_works} discusses related work in more depth and the
differences with this work.
Sect.~\ref{app:proof_analytical} then contains the proofs of the analytical results of
Sect.~\ref{sect:Analytical}.
Finally, Sect.~\ref{app:sect:numerical} gives the numerical details of
Sect.~\ref{sect:Numerical}, for both the closed-form CNTK and the finite-width U-Net.

\section{Neural Tangent Kernel}
\label{app:appendix_NTK}
In this section we present in more detail the lazy regime and the Neural Tangent Kernel.
\subsection{Introduction}
Consider a parametric family of functions $\vs_{\vtheta}:\mathbb{R}^d \to
\mathbb{R}^{d'}$ with $\vtheta \in \mathbb{R}^P$ and a training set $\{\vx^\nu, \vy^\nu\}_\nu$ with $\vx^\nu\in\mathbb{R}^d$ and $\vy^\nu\in\mathbb{R}^{d'}$. Trained under
gradient flow on the square loss $\mathcal{L}(\vtheta) = \frac{1}{2}\sum_\nu
\lVert\vs_{\vtheta}(\vx^\nu) - \vy^\nu\rVert^2$, the parameters evolve
as $\dot{\vtheta} = - \nabla_{\vtheta}\mathcal{L}(\vtheta)$, which
induces dynamics in function space
\begin{align}
    \frac{\dd}{\dd \tau}\,\vs_{\vtheta(\tau)}(\vx)
    \;=\; -\sum_{\nu}
    K\left(\vtheta(\tau);\,\vx,\vx^\nu\right)
    \cdot
    \big(\vs_{\vtheta(\tau)}(\vx^\nu) - \vy^\nu\big),
\end{align}
where the matrix-valued \emph{Neural Tangent Kernel} is
\begin{align}
\label{eq:ntk_def}
    K\left(\vtheta;\,\vx,\vx'\right)_{k,k'}
    \;=\; \sum_{i=1}^{P}
    \partial_{\vtheta_i}\vs_{\vtheta}(\vx)_{k}\,
    \partial_{\vtheta_i}\vs_{\vtheta}(\vx')_{k'},
    \qquad k,k'\in\{1,\dots,d'\}.
\end{align}
A priori the kernel \eqref{eq:ntk_def} depends on the current parameters
$\vtheta(\tau)$ and is therefore a stochastic, time-dependent operator.
The crucial observation of \citet{Jacot_2018} is that in the
infinite-width limit, with the appropriate variance scaling at
initialization, the kernel converges to
a \emph{deterministic} limit $K^\star(\vx,\vx')$, and the deviation of
the parameters from their initial value vanishes:
$\lVert \vtheta(\tau) - \vtheta(0)\rVert \to 0$. In this lazy regime,
the dynamics on $\vs_{\vtheta}$ becomes a linear kernel flow with the
constant kernel $K^\star$.

\subsection{Lazy regime and linearization}
\label{app:ntk:lazy}

The lazy regime can equivalently be stated as a first-order
linearization around the initialization \citep{chizat2020lazytrainingdifferentiableprogramming}:
\begin{align}
    \vs_{\vtheta}(\vx)\;\simeq\;
    \vs_{\vtheta(0)}(\vx) \;+\;
    \nabla_{\vtheta}\,\vs_{\vtheta(0)}(\vx)\cdot
    \big(\vtheta - \vtheta(0)\big).
\end{align}
The features $\nabla_{\vtheta}\,\vs_{\vtheta(0)}(\vx) \in
\mathbb{R}^{d'\times P}$ are random at initialization and \emph{frozen}
throughout training. Training the network is equivalent to ridgeless
linear regression on these random features, and the associated kernel is
exactly the NTK at initialization \eqref{eq:ntk_def} evaluated at
$\vtheta(0)$.
\subsection{Vector-valued outputs and decomposable kernels}
\label{app:ntk:vector}

For a score-matching task one needs a \emph{vector}-valued output
($d' = d$). For the standard NTK parametrizations of fully-connected
and convolutional networks, the matrix-valued NTK has the convenient
decomposable form
\begin{align}
\label{eq:ntk_decomposable}
    K(\vx,\vx')_{k,k'} \;=\; K_{\mathrm{scal}}(\vx,\vx')\,\delta_{k,k'},
\end{align}
i.e.\ a scalar kernel $K_{\mathrm{scal}}(\vx,\vx')$ tensored with the
identity in output space. This factorization holds whenever the output
layer is fully connected with i.i.d.\ initialization and no
output-coupling structure, and it reduces the spectral analysis of the
matrix-valued kernel to that of the scalar kernel: the operator $K$
acts diagonally across output coordinates and its eigenvalues are
$d$-fold copies of the eigenvalues of $K_{\mathrm{scal}}$. In what
follows, we therefore present only the scalar kernels.

\subsection{Kernel gradient flow}
\label{app:ntk:flow}

We derive here the closed-form solution \eqref{eq:estimation_score_DSM} of the main
text. The argument is the standard lazy-regime computation of \citet{Jacot_2018}, adapted
to the DSM loss \eqref{eq:app_train_loss}; its only difference is that the regression target is
the frozen noise $-\vxi/\sqrt{\Delta_t}$ rather than a label.

Recall the training loss \eqref{eq:train_loss} at fixed noise level,
\begin{align}\label{eq:app_train_loss}
  \Ltrain(\vtheta) = \frac{1}{2nmd}\sum_{\nu=1}^{n}\sum_{\alpha=1}^{m}
  \left\lVert \vs_{\vtheta}(\vY^{\nu\alpha}) + \frac{\vxi^{\nu\alpha}}{\sqrt{\Delta_t}}\right\rVert^2 ,
  \qquad \vY^{\nu\alpha} = e^{-t}\vx^\nu + \sqrt{\Delta_t}\,\vxi^{\nu\alpha},
\end{align}
and the gradient flow $\dd\vtheta/\dd\tau = -d^2\,\nabla_{\vtheta}\Ltrain(\vtheta)$.
Differentiating,
\begin{align*}
  \nabla_{\vtheta}\Ltrain(\vtheta)
  = \frac{1}{nmd}\sum_{\nu,\alpha} \nabla_{\vtheta}\vs_{\vtheta}(\vY^{\nu\alpha})^\top
    \left(\vs_{\vtheta}(\vY^{\nu\alpha}) + \frac{\vxi^{\nu\alpha}}{\sqrt{\Delta_t}}\right),
\end{align*}
so that, by the chain rule, the induced evolution of the function at an arbitrary point
$\vy\in\mathbb{R}^d$ is
\begin{align}
\label{eq:app_flow_function}
  \dot{\vs}_{\vtheta}(\vy)
  = \nabla_{\vtheta}\vs_{\vtheta}(\vy)\cdot\frac{\dd\vtheta}{\dd\tau}
  = -\frac{d}{nm}\sum_{\nu,\alpha} K_\tau(\vy,\vY^{\nu\alpha})
    \left(\vs_{\vtheta}(\vY^{\nu\alpha}) + \frac{\vxi^{\nu\alpha}}{\sqrt{\Delta_t}}\right),
\end{align}
with $K_\tau(\vx,\vy)=\nabla_{\vtheta}\vs_{\vtheta}(\vx)\cdot\nabla_{\vtheta}\vs_{\vtheta}(\vy)$
the NTK at training time $\tau$. The factor $d^2$ of the learning rate combines with the
normalization $1/(nmd)$ of the loss into the rate $d/(nm)$ appearing in
\eqref{eq:estimation_score_DSM}; with this choice the generalization eigenvalues
$\Theta(nm/d)$ are learned on times $\tau=O(1)$.

In the infinite-width (lazy) limit, $K_\tau=K$ is constant. By the decomposable structure \eqref{eq:ntk_decomposable} the $d$ output
components decouple and it suffices to treat each of them separately; we therefore drop the
output index and regard all quantities below as vectors in $\mathbb{R}^{N}$, $N=nm$, indexed
by the pair $(\nu\alpha)$.

Let $\vf(\tau)\in\mathbb{R}^{N}$ collect the values of the score on the training points,
$\vf_{\nu\alpha}(\tau)=\vs_{\vtheta(\tau)}(\vY^{\nu\alpha})$, let
$\vXi_{\nu\alpha}=\vxi^{\nu\alpha}/\sqrt{\Delta_t}$, and let $\vG\in\mathbb{R}^{N\times N}$,
$\vG^{\mu\beta,\nu\alpha}=K(\vY^{\mu\beta},\vY^{\nu\alpha})$, be the Gram matrix. Taking
$\vy=\vY^{\mu\beta}$ in \eqref{eq:app_flow_function} gives a closed system,
\begin{align}
  \dot{\vf} = -\frac{d}{nm}\,\vG\,(\vf + \vXi).
\end{align}
With the initialization $\vs_{\vtheta(0)}=0$, integrating gives
\begin{align}
\label{eq:app_residual}
  \vf(\tau) = -\left(\vI_N - e^{-\frac{d\tau}{nm}\vG}\right)\vXi.
\end{align}
The training loss along the flow is therefore
$\Ltrain(\vtheta(\tau)) = \tfrac{1}{2nmd}\Tr\!\big(e^{-\frac{2d\tau}{nm}\vG}\,\vXi\vXi^\top\big)$,
summed over output components.

For a general test point $\vy$, integrating from $0$ to $\tau$ with
$\vs_{\vtheta(0)}(\vy)=0$ gives
\begin{align}
\label{eq:app_filter}
  \vs_{\vtheta(\tau)}(\vy)
  &= -\frac{d}{nm}\,K(\vy,\vY)^\top\int_0^\tau e^{-\frac{d\sigma}{nm}\vG}\,\dd\sigma\;\vXi\\
  &= -K(\vy,\vY)^\top \vG^{-1}\left(\vI_N - e^{-\frac{d\tau}{nm}\vG}\right)\vXi\\
  &= -K(\vy,\vY)^\top \varphi_\tau(\vG)\,\vXi ,
\end{align}
where $\varphi_\tau(\lambda)=\big(1-e^{-\frac{d\tau}{nm}\lambda}\big)/\lambda$ is the
gradient-flow spectral filter and
$K(\vy,\vY)\in\mathbb{R}^{N}$ has entries $K(\vy,\vY^{\nu\alpha})$. Recalling
$\vXi=\vxi/\sqrt{\Delta_t}$, this is exactly \eqref{eq:estimation_score_DSM}.

Diagonalizing $\vG=\sum_{\lambda}\lambda\,\vu_\lambda\vu_\lambda^\top$, with $\vu_\lambda$ the unit eigenvectors of $\vG$, turns
\eqref{eq:app_filter} into a filter acting mode by mode. Since
$\varphi_\tau(\lambda)\simeq 1/\lambda$ for $\lambda \gg nm/(d\tau)$ and
$\varphi_\tau(\lambda)\simeq d\tau/(nm)$ for $\lambda \ll nm/(d\tau)$, at training time
$\tau$ the modes above
\begin{align*}
  \lambda_c \;\sim\; \frac{nm}{d\tau}
\end{align*}
have been learned --- their coefficient has saturated at the ridgeless value $1/\lambda$ ---
whereas those below are still essentially untouched. This is the spectral cutoff of
Sect.~\ref{sect:Analytical}: fixing a training time is equivalent, up to the width of the crossover,
to a hard spectral cutoff at $\lambda_c$, and a mode of eigenvalue $\lambda$ is learned on
the timescale $\tau_\lambda \sim nm/(d\lambda)$. Letting $\tau\to\infty$ recovers the
ridgeless kernel interpolator
$\vs_\infty(\vy) = -K(\vy,\vY)^\top\vG^{+}\vxi/\sqrt{\Delta_t}$, with $\vG^{+}$ the
Moore--Penrose pseudo-inverse; adding an $L_2$ penalty instead replaces $\vG$ by
$\vG+\gamma\vI_N$, as discussed in Sect.~\ref{app:subsec:regularization} below.

\subsubsection{The role of regularization}
\label{app:subsec:regularization}
Adding an $L_2$ penalty $\frac{\gamma}{2nmd}\lVert \vtheta\rVert^2$ to the loss \eqref{eq:app_train_loss} amounts to shifting the Gram matrix by $\gamma \vI_N$ in \eqref{eq:estimation_score_DSM},
\begin{align}
\label{eq:app_ridge_flow}
    \vs_{\vtheta(\tau)}(\vy)=K(\vy,\vY)^\top\left(\vG+\gamma \vI_N\right)^{-1}\left(\vI_N-e^{-\frac{d(\vG+\gamma \vI_N)\tau}{nm}}\right)\left(-\frac{\vxi}{\sqrt{\Delta_t}}\right).
\end{align}
\emph{Derivation.} In the lazy regime the model is linear in its parameters around the
initialization $\vtheta(0)$ that we assume to be equal to 0 for simplicity, at which $\vs_{\vtheta(0)}=0$. As in
Sect.~\ref{app:ntk:flow} we treat one output component at a time and write
$\vs_{\vtheta}(\vy)=\vPhi(\vy)\vtheta$, with $\vPhi(\vy)=\nabla_{\vtheta}\vs_{\vtheta}(\vy)$ the
(constant) feature map, $\vPhi_{\vY}\in\mathbb R^{N\times|\vtheta|}$ its rows at the training
points, so that $K(\vy,\vY)=\vPhi_{\vY}\vPhi(\vy)^\top$ and $\vG=\vPhi_{\vY}\vPhi_{\vY}^\top$.
With the penalty, the gradient flow $\dd\vtheta/\dd\tau=-d^2\nabla_{\vtheta}\big(\Ltrain+\frac{\gamma}{2nmd}\lVert\vtheta\rVert^2\big)$ reads
\begin{align}
  \frac{\dd\vtheta}{\dd\tau}=-\frac{d}{nm}\Big[\vPhi_{\vY}^\top\big(\vPhi_{\vY}\vtheta+\vXi\big)+\gamma\,\vtheta\Big],
  \qquad \vXi=\frac{\vxi}{\sqrt{\Delta_t}} .
\end{align}
Since $\vtheta(0)=0$ and the right-hand side lies in the row space of $\vPhi_{\vY}$ whenever
$\vtheta$ does, $\vtheta(\tau)=\vPhi_{\vY}^\top\va(\tau)$ for some $\va(\tau)\in\mathbb R^N$.
Substituting and using $\vPhi_{\vY}\vPhi_{\vY}^\top=\vG$ gives
$\vPhi_{\vY}^\top\dot\va=-\frac{d}{nm}\vPhi_{\vY}^\top\big[(\vG+\gamma\vI_N)\va+\vXi\big]$,
which is solved by
\begin{align}
  \dot\va=-\frac{d}{nm}\big[(\vG+\gamma\vI_N)\va+\vXi\big],\qquad
  \va(\tau)=-(\vG+\gamma\vI_N)^{-1}\Big(\vI_N-e^{-\frac{d\tau}{nm}(\vG+\gamma\vI_N)}\Big)\vXi .
\end{align}
Finally $\vs_{\vtheta(\tau)}(\vy)=\vPhi(\vy)\vPhi_{\vY}^\top\va(\tau)=K(\vy,\vY)^\top\va(\tau)$,
which is \eqref{eq:app_ridge_flow}. At $\tau\to\infty$ it reduces to kernel ridge regression
with ridge $\gamma$, $K(\vy,\vY)^\top(\vG+\gamma\vI_N)^{-1}(-\vXi)$, i.e.\ \eqref{eq:krr_predictor}.

A regularization strength $m\ll \gamma \ll \psi_nm$ suppresses the
low eigenvalues and thus prevents overfitting, as early stopping does: the window sits
between the memorization bulk, at $\Theta(m)$, and the generalization bulk, at
$\Theta(\psi_n m)$.

Throughout the bias--variance analysis it is convenient to measure the ridge in units of
$m\psi_n=N/d$, the scale of the generalization eigenvalues, and we write
\begin{align}\label{eq:app_gamma_tilde}
    \tilde\gamma \;=\; \frac{\gamma}{m\psi_n}
\end{align}
for the resulting \emph{rescaled} ridge.

\subsection{Two-layer fully-connected network}
\label{app:ntk:2layer}

Consider the two-layer network
\begin{align}
    f_{\vtheta}(\vx) \;=\; \frac{1}{\sqrt{p}}\,\va^{\top}\,
    \sigma\left(\frac{\vW \vx}{\sqrt{d}}\right),
    \qquad \vW \in \mathbb{R}^{p\times d},\quad \va\in\mathbb{R}^p,
\end{align}
with i.i.d.\ Gaussian initialization $\vW_{ij},\va_{i}\sim\mathcal{N}(0,1)$
and a pointwise activation $\sigma$. In the limit $p\to\infty$, the
NTK splits into two contributions, coming from the gradients with respect
to the second-layer and first-layer weights respectively, and converges
to the deterministic limit
\begin{align}
\label{eq:ntk_2layer}
    K^\star(\vx,\vx')
    =
    \underbrace{\mathbb{E}_{\vw\sim\mathcal{N}(0,\vI_d)}\left[
       \sigma\left(\tfrac{\vw^\top\vx}{\sqrt d}\right)
       \sigma\left(\tfrac{\vw^\top\vx'}{\sqrt d}\right)\right]}_{K_{\sigma}(\vx,\vx')}
    +
    \underbrace{\frac{\vx^\top\vx'}{d}\;
    \mathbb{E}_{\vw\sim\mathcal{N}(0,\vI_d)}\left[
       \sigma'\left(\tfrac{\vw^\top\vx}{\sqrt d}\right)
       \sigma'\left(\tfrac{\vw^\top\vx'}{\sqrt d}\right)\right]}_{K_{\sigma'}(\vx,\vx')}.
\end{align}
The first term $K_{\sigma}$ is the so-called \emph{conjugate} (or
NNGP) kernel; the second is the \emph{tangent} kernel of the first
layer.

For inputs of comparable norm
$\lVert\vx\rVert^2,\lVert\vx'\rVert^2 \asymp d$, both expectations
\eqref{eq:ntk_2layer} depend only on the normalized inner products
$\vx^\top\vx'/d$, $\lVert\vx\rVert^2/d$, $\lVert\vx'\rVert^2/d$. When
$\lVert\vx\rVert^2 = \lVert\vx'\rVert^2 = d$ exactly, $K^\star$ is a pure dot-product kernel,
\begin{align}
    K^\star(\vx,\vx') \;=\; f\left(\tfrac{\vx^\top\vx'}{d}\right),
\end{align}
for some scalar function $f$ depending on $\sigma$.

For
$\sigma(u) = \max(u,0)$ (ReLU), an explicit closed-form computation of
the Gaussian integrals in \eqref{eq:ntk_2layer} yields
\begin{align}
\label{eq:ntk_relu_2layer}
    K^\star(\vx,\vx')
    \;=\;
    \frac{\lVert\vx\rVert\,\lVert\vx'\rVert}{d}\;
    f_{\mathrm{ReLU}}\left(\tfrac{\vx^\top\vx'}{\lVert\vx\rVert\,\lVert\vx'\rVert}\right),
    \qquad
    f_{\mathrm{ReLU}}(u) \;=\; \frac{\sqrt{1-u^2}+2u\,(\pi - \arccos u)}{2\pi},
\end{align}
where the two terms come from $\mathbb E[\sigma(z_1)\sigma(z_2)]=\frac{\sqrt{1-u^2}+u(\pi-\arccos u)}{2\pi}$
and $\mathbb E[\sigma'(z_1)\sigma'(z_2)]=\frac{\pi-\arccos u}{2\pi}$ for standard Gaussians
$(z_1,z_2)$ of correlation $u$ \citep{cho2009}. The
function $f_{\mathrm{ReLU}}$ is smooth on $(-1,1)$, with
$f_{\mathrm{ReLU}}(0) = 1/(2\pi)$ and $f'_{\mathrm{ReLU}}(0) = 1/2$. With the He-normalized
activation $\sqrt2\max(u,0)$ used in our experiments (Sect.~\ref{app:kernels_used}) the kernel
is multiplied by $2$.

\subsection{Multilayer fully-connected network}
\label{app:ntk:multilayer}

For a depth-$L$ fully-connected network with widths
$d = n_0, n_1, \dots, n_L = d'$ and activations $\sigma$,
\begin{align}
    \valpha^{(0)}(\vx) \;=\; \vx, \qquad
    \tilde\valpha^{(\ell+1)}(\vx) \;=\; \frac{1}{\sqrt{n_\ell}}\,
    \vW^{(\ell)}\,\valpha^{(\ell)}(\vx) + \vb^{(\ell)}, \qquad
    \valpha^{(\ell)} \;=\; \sigma\big(\tilde\valpha^{(\ell)}\big),
\end{align}
with i.i.d.\ weights $\mathcal N(0,1)$, biases $\mathcal N(0,\beta^2)$, and output
$\tilde\valpha^{(L)}(\vx)$, the conjugate and tangent kernels are computed by a coupled recursion
on $\ell = 1,\dots,L$ \citep{Jacot_2018}. Writing $\Sigma^{(\ell)}(\vx,\vx')$ for the covariance of
the pre-activations at the two inputs, define the $2\times 2$ covariance matrix at layer $\ell$,
\begin{align}
    \bm{\Sigma}^{(\ell)}(\vx,\vx')
    \;=\;
    \begin{pmatrix} \Sigma^{(\ell)}(\vx,\vx) & \Sigma^{(\ell)}(\vx,\vx') \\
    \Sigma^{(\ell)}(\vx',\vx) & \Sigma^{(\ell)}(\vx',\vx')
    \end{pmatrix},
\end{align}
which collects the covariances of the pre-activations $\tilde\valpha^{(\ell)}$. The recursion is
\begin{align}
    \Sigma^{(1)}(\vx,\vx') &\;=\; \frac{\vx^\top\vx'}{d}+\beta^2, \\
    \Sigma^{(\ell+1)}(\vx,\vx')
    &\;=\; \mathbb{E}_{(u,v)\sim\mathcal{N}(0,\bm{\Sigma}^{(\ell)}(\vx,\vx'))}
    \big[\sigma(u)\,\sigma(v)\big] + \beta^2, \\
    \dot\Sigma^{(\ell+1)}(\vx,\vx')
    &\;=\; \mathbb{E}_{(u,v)\sim\mathcal{N}(0,\bm{\Sigma}^{(\ell)}(\vx,\vx'))}
    \big[\sigma'(u)\,\sigma'(v)\big],
\end{align}
for $\ell=1,\dots,L-1$, where $\beta^2$ is the bias variance. The conjugate kernel is
$K_{\mathrm{conj}}^\star = \Sigma^{(L)}$, and the NTK follows from
$\Theta^{(1)}=\Sigma^{(1)}$ and $\Theta^{(\ell+1)}=\Theta^{(\ell)}\dot\Sigma^{(\ell+1)}+\Sigma^{(\ell+1)}$,
i.e.\ from the unrolled product
\begin{align}
\label{eq:ntk_multilayer}
    K^\star(\vx,\vx')
    \;=\;
    \sum_{\ell=1}^{L}\Sigma^{(\ell)}(\vx,\vx')
    \;\prod_{\ell'=\ell+1}^{L}\dot\Sigma^{(\ell')}(\vx,\vx').
\end{align}
At $L=2$ and $\beta=0$ this is $\Sigma^{(1)}\dot\Sigma^{(2)}+\Sigma^{(2)}$, i.e.\
\eqref{eq:ntk_2layer}.
For a homogeneous activation such as ReLU, both $\Sigma^{(\ell)}$ and
$\dot\Sigma^{(\ell)}$ admit closed forms in terms of arc-cosine kernels, leading to a fully explicit (yet recursive) formula
for $K^\star$.

\subsection{Convolutional neural networks}
\label{app:ntk:cnn}
Let the inputs lie in $\mathbb{R}^{C_0\times P}$ with
$P$ pixels and $C_0$ channels. A depth-$L$ CNN is built from
convolutional layers
\begin{align}
    \tilde\alpha^{(\ell+1)}_{c,p}(\vx)
    \;=\; \frac{1}{\sqrt{C_\ell\,|\mathcal Q|}}
    \sum_{c'=1}^{C_\ell}\sum_{q\in\mathcal Q}
    W^{(\ell)}_{c,c',q}\,\alpha^{(\ell)}_{c',p+q}(\vx),
    \qquad
    \valpha^{(\ell)} = \sigma(\tilde\valpha^{(\ell)}),
\end{align}
with i.i.d.\ weights $W^{(\ell)}_{c,c',q}\sim\mathcal N(0,1)$, patch
support $\mathcal Q\subset\mathbb Z^k$ ($k=1,2$ in 1-D/2-D), and
periodic boundary conditions. The readout aggregates the $C_L\times P$
last-layer activations into the $d'$-dimensional output.

Define the layer-$\ell$
pre-activation covariance tensor
$\vSigma^{(\ell)}(\vx,\vx')\in\mathbb R^{P\times P}$, $\ell\ge1$, by
$\vSigma^{(\ell)}_{p,p'}(\vx,\vx')=\frac{1}{C_\ell}
\sum_c\mathbb E[\tilde\alpha^{(\ell)}_{c,p}(\vx)\,\tilde\alpha^{(\ell)}_{c,p'}(\vx')]$, and let
$\vSigma^{(0)}$ be the input overlap below. In the infinite-channel limit ($C_\ell\to\infty$
jointly), the recursion is
\begin{align}
\label{eq:cntk:Sigma_init}
    \vSigma^{(0)}_{p,p'}(\vx,\vx')
    &\;=\; \frac{1}{C_0}\sum_{c=1}^{C_0} x_{c,p}\,x'_{c,p'},
    \qquad
    \vSigma^{(1)}_{p,p'}(\vx,\vx')
    \;=\; \frac{1}{|\mathcal Q|}\sum_{q\in\mathcal Q}
    \vSigma^{(0)}_{p+q,\,p'+q}(\vx,\vx'), \\
\label{eq:cntk:Sigma_rec}
    \vSigma^{(\ell+1)}_{p,p'}(\vx,\vx')
    &\;=\; \frac{1}{|\mathcal Q|}\sum_{q\in\mathcal Q}
    \mathcal T\big[\vSigma^{(\ell)}\big]_{p+q,\,p'+q}(\vx,\vx'),\qquad \ell\ge1,
\end{align}
the input $\valpha^{(0)}=\vx$ entering the first layer without activation, and
where $\mathcal T$ and $\dot{\mathcal T}$ are defined as $\mathcal T[\vSigma]_{p,p'}=
\mathbb E_{(u,v)\sim\mathcal N(0,\,\vSigma_{\{p,p'\}})}[\sigma(u)\sigma(v)]$
with $\vSigma_{\{p,p'\}}=\big(\begin{smallmatrix}\vSigma_{p,p}&\vSigma_{p,p'}\\ \vSigma_{p',p}&\vSigma_{p',p'}\end{smallmatrix}\big)$,
and similarly for $\dot{\mathcal T}$ with $\sigma'$ in place of $\sigma$.

The convolutional NTK \citep{Arora2019} of the pre-activations is the tensor
$\vTheta^{(\ell)}\in\mathbb R^{P\times P}$ defined by $\vTheta^{(1)}=\vSigma^{(1)}$ and
\begin{align}
\label{eq:cntk:theta}
    \vTheta^{(\ell+1)}_{p,p'}(\vx,\vx')
    \;=\;
    \frac{1}{|\mathcal Q|}\sum_{q\in\mathcal Q}
    \Big(\vTheta^{(\ell)}\odot\dot{\mathcal T}\big[\vSigma^{(\ell)}\big]\Big)_{p+q,\,p'+q}(\vx,\vx')
    \;+\;\vSigma^{(\ell+1)}_{p,p'}(\vx,\vx'),
\end{align}
with $\odot$ the entrywise product. Because the patch average acts on the product
$\vTheta^{(\ell)}\odot\dot{\mathcal T}[\vSigma^{(\ell)}]$, the recursion does not unroll into
a product of layer-wise factors as in \eqref{eq:ntk_multilayer}. The readout acts on the
last activations $\valpha^{(L)}=\sigma(\tilde\valpha^{(L)})$, whose NTK tensor is
$\bar\vTheta^{(L)}=\vTheta^{(L)}\odot\dot{\mathcal T}[\vSigma^{(L)}]+\mathcal T[\vSigma^{(L)}]$. The scalar
NTK is then obtained by applying the readout to $\bar\vTheta^{(L)}$. For the two main readout choices, the CNTKs are
\begin{align}
\label{eq:cntk:readout}
    K^\star_{\mathrm{flat}}(\vx,\vx')
    &\;=\; \frac{1}{P}\sum_{p=1}^{P}
    \bar\vTheta^{(L)}_{p,p}(\vx,\vx')
    \quad\text{(flatten + linear)}, \\
    K^\star_{\mathrm{GAP}}(\vx,\vx')
    &\;=\; \frac{1}{P^2}\sum_{p,p'=1}^{P}
    \bar\vTheta^{(L)}_{p,p'}(\vx,\vx')
    \quad\text{(Global Average Pooling)}.
\end{align}

\subsection{Kernels used in the numerical experiments}
\label{app:kernels_used}
We collect here the kernels used in the numerical experiments of the analytical section. In all the synthetic experiments the data
are isotropic Gaussian, $\vSigma=\vI_d$ ($\sigma^2=1$), so that $\Gamma_t=1$, and the kernel
constants $\mu_I,\mu_B,\mu_0,\mu_1$ of Theorem~\ref{app:thm:lin_gram_empirical} are computed
from the same $f$ at the same $t$.

\paragraph*{Two-layer ReLU NTK.} We use the NTK
of a two-layer network with the He-normalized activation $\sigma(z)=\sqrt2\max(0,z)$, both
layers trained. In terms of the arc-cosine kernels of \citet{cho2009},
\begin{align}
  f(u)=\kappa_1(u)+u\,\kappa_0(u),\qquad
  \kappa_1(u)=\frac{\sqrt{1-u^2}+u(\pi-\arccos u)}{\pi},\qquad
  \kappa_0(u)=1-\frac{\arccos u}{\pi},
\end{align}
where $\kappa_1(u)=\mathbb E[\sigma(z_1)\sigma(z_2)]$ and $\kappa_0(u)=\mathbb E[\sigma'(z_1)\sigma'(z_2)]$
for standard Gaussians $(z_1,z_2)$ of correlation $u$. Its Taylor coefficients at the origin
are $f(0)=1/\pi$, $f'(0)=1$ and $f''(0)=3/\pi$; the activation is not odd, so $\mu_0\neq0$
and the rank-one outlier at $\simeq nm\mu_0$ is present and removed as in
Theorem~\ref{thm:Spectrum_gram_linear}. Since the arc-cosine kernels are defined only for
$|u|\le1$, the noised points are projected onto the sphere of radius $\sqrt d$ in the kernel
argument, $f(\hat\vY^{\nu\alpha\top}\hat\vY^{\mu\beta}/d)$ with
$\hat\vy=\sqrt d\,\vy/\lVert\vy\rVert$. At $\sigma^2=1$ this changes the arguments by
$O(d^{-1/2})$ and makes $\vG$ an exact inner-product kernel matrix.

\paragraph*{Dual kernel of $\tanh$.} The bias--variance theory and
the kernel ridge experiments use
\begin{align}
  f(u)=\mathbb E\big[\tanh(z_1)\tanh(z_2)\big],\qquad
  f'(u)=\mathbb E\big[\tanh'(z_1)\tanh'(z_2)\big].
\end{align}
This is the infinite-width limit of the
random-features kernel of Sect.~\ref{app:kernel_rf}, rather than an NTK, and it is an
inner-product kernel of the form assumed throughout. The activation is odd, hence
$f(0)=f''(0)=0$ and $\mu_0=0$ exactly. In the finite-$d$ experiments the kernel arguments are
projected onto the sphere as above (Sect.~\ref{app:subsec:krr_bv}); scores are not rescaled.

\paragraph*{Convolutional NTK.} The CelebA
experiments use the closed-form infinite-width NTK of a depth-$D$ CNN with $3\times3$
filters, ReLU activations and a $1\times1$ readout, computed with the recursion of
Sect.~\ref{app:ntk:cnn}; implementation details are in Sect.~\ref{app:subsec:cntk}.

\paragraph*{Empirical NTK of a U-Net.} In the
feature-learning experiments the kernel is the finite-width NTK
$\nabla_{\vtheta}\vv_{\vtheta}(\vx)\cdot\nabla_{\vtheta}\vv_{\vtheta}(\vy)$ of the U-Net at
training time $\tau$, estimated by random projections as described in
Sect.~\ref{app:subsec:unet}.

\section{Equivalence between the Ornstein--Uhlenbeck and rectified-flow interpolants}
\label{app:OU_vs_RF}
\subsection{Time reparametrization}
The analytical part (Sect.~\ref{sect:Analytical}) is
formulated for the variance-preserving Ornstein--Uhlenbeck (OU) process
\begin{equation}
\label{eq:app:OU_interpolant}
    \vY^{\nu\alpha}_t \;=\; e^{-t}\,\vx^\nu \;+\; \sqrt{\Delta_t}\,\vxi^{\nu\alpha},
\end{equation}
with $\Delta_t \;=\; 1 - e^{-2t}$ and $t$ going from 0 to $\infty$,
while the experimental section (Sect.~\ref{sect:Numerical}) uses the rectified-flow (RF) interpolant
\begin{equation}
\label{eq:app:RF_interpolant}
    \vY^{\nu\alpha}_s \;=\; (1-s)\,\vx^\nu \;+\; s\,\vxi^{\nu\alpha},
\end{equation}
with the time $s$ going from 0 to 1.
Both processes interpolate between the data distribution $P_0$ ($t = 0$ or $s = 0$) and a standard Gaussian distribution $\mathcal{N}(0,\vI_d)$
($t \to \infty$ or $s = 1$), through different interpolants. Both are of the form
\begin{align}
    \vY^{\nu\alpha}=a \vx^\nu+b\vxi^{\nu\alpha}
\end{align}
with $a,b\in[0,1]$: $(a,b)=(e^{-t},\sqrt{\Delta_t})$ with $a^2+b^2=1$ for OU, and $(a,b)=(1-s,s)$
for RF, which is not variance-preserving. The two are related by a change of time and a
global rescaling of the inputs,
\begin{align}
\label{eq:app:time_map}
  \vY^{\nu\alpha}_s=c_s\,\vY^{\nu\alpha}_{t(s)},\qquad
  c_s=\sqrt{(1-s)^2+s^2},\qquad
  e^{-t(s)}=\frac{1-s}{c_s},\qquad \Delta_{t(s)}=\frac{s^2}{c_s^2},
\end{align}
with the same data and noises. For an inner-product kernel the rescaling only changes the
kernel, $f(\vY_s^\top\vY'_s/d)=f_s(\vY_{t(s)}^\top\vY'_{t(s)}/d)$ with $f_s(u)=f(c_s^2u)$, so that
the RF Gram matrix at time $s$ is the OU Gram matrix at time $t(s)$ built with the kernel $f_s$.
\subsection{Substitution rule}
All quantities appearing in the linear-equivalent theorem
(Theorem~\ref{app:thm:lin_gram_empirical}) and hence in the spectrum analysis
(Theorem~\ref{thm:Spectrum_gram_linear}) depend on the OU process only
through the scalar quantities $(e^{-t},\sqrt{\Delta_t})$ and the kernel $f$. By
\eqref{eq:app:time_map}, the results for RF at time $s$ are those for OU at time $t(s)$ with the
kernel $f_s$; equivalently, they are obtained from the OU formulas with the substitution
\begin{equation}
\label{eq:app:substitution}
    e^{-t} \;\longleftrightarrow\; 1 - s, \qquad
    \sqrt{\Delta_t} \;\longleftrightarrow\; s,
\end{equation}
provided the identity $e^{-2t}+\Delta_t=1$ is not used, i.e.\ keeping
$\Gamma=(1-s)^2\sigma^2+s^2$ wherever $\Gamma_t$ appears.

\subsection{Relation between score and velocity field}
The OU and RF parametrizations differ in which target the network learns:
the OU formulation trains a score model on $-\vxi/\sqrt{\Delta_t}$, while
the RF formulation trains a velocity model on $\vxi - \vx^\nu$. The two
targets are related by the algebraic identity
\begin{equation}
\label{eq:app:score_velocity}
    \vv^*(\vx_s, s) \;=\; -\,\frac{\vx_s + s\,\vs^*(\vx_s, s)}{1 - s}, \qquad
    \vs^*(\vx_s, s) \;=\; -\,\frac{(1-s)\,\vv^*(\vx_s, s) + \vx_s}{s},
\end{equation}
which follows from Tweedie's formula applied to either parametrization.
This is a one-to-one affine map between the
trained predictors, so the spectral analysis of $\vG$ controls the
training dynamics of both equivalently.

\section{Extended discussion of related work}
\label{app:sect:related_works}
This section discusses related work in more depth.

\subsection{Effect of the architecture and training dynamics}
\label{app:subsec:architecture_dynamics}

To move beyond the empirical score, a recent line of work studies the effect of parametrizing the score with a neural network. The simplest such parametrization is a linear one, and it has been analyzed at both ends of the sample-complexity spectrum. \citet{merger2026} work in the proportional regime $n\asymp d$ and characterize the overfitting of the empirical risk minimizer for data with a power-law covariance. \citet{Wang2025}, at the other end, work in the population limit and follow the training dynamics, again for power-law data; they exhibit a spectral bias, the modes associated with the largest eigenvalues being learned first. The same sequential picture survives beyond linear models: \citet{bardone2026}, for a two-layer network, show that the statistics of the data distribution are learned successively during training. Closer to the present setting, \citet{george_2025} and \citet{bonnaire2025diffusionmodelsdontmemorize} both take the score to be a random-features network \citep{Rahimi_2007}, an architecture whose infinite-width limit is a kernel method \textemdash{} see Appendix~\ref{app:kernel_rf} for the correspondence between random-features networks and the kernels used here. \citet{george_2025} compute the learning curves of the empirical risk minimizer, while \citet{bonnaire2025diffusionmodelsdontmemorize} characterize the training dynamics and its timescales. We discuss the latter work at greater length in the next subsection.

\subsection{Discussion of \citet{bonnaire2025diffusionmodelsdontmemorize}}

The phenomenon this paper analyzes was identified independently by
\citet{bonnaire2025diffusionmodelsdontmemorize} and \citet{Favero2025_bigger}: a diffusion model trained on a finite dataset passes through a
long window during which it generalizes, and only memorizes its training set much later,
so that early stopping suffices to avoid memorization. Both papers establish
that the width of that window grows with the number of samples $n$.

\citet{bonnaire2025diffusionmodelsdontmemorize} characterize the two timescales
$\tau_{\mathrm{gen}}$ and $\tau_{\mathrm{mem}}$ within a specific analytically tractable
architecture, a random-features network \citep{Rahimi_2007}.
In that model the two timescales are read off the curvature of the loss at initialization.
Since the loss is quadratic in the trainable weights, the Hessian at the origin governs the
whole gradient flow: a direction of curvature $\lambda$ relaxes on a timescale $1/\lambda$. The authors find that its
spectrum splits into fast directions, learned first, and slow directions,
learned much later, and the gap between the two groups of eigenvalues is what opens the
window between $\tau_{\mathrm{gen}}$ and $\tau_{\mathrm{mem}}$. The connection with the
present analysis is direct: for a quadratic loss the Hessian in parameter space and the Gram
matrix in sample space are $\vPhi\tran\vPhi$ and $\vPhi\vPhi\tran$ for the same
feature matrix $\vPhi$, and therefore have the same non-zero spectrum up to normalization.
\citet{Favero2025_bigger} reach the same phenomenology from a kernel argument, stated for an
arbitrary isotropic kernel.

The question addressed here is whether the same phenomenology can be established directly on
the spectrum of the Gram matrix $\vG^{\nu\alpha,\mu\beta}=K(\vY^{\nu\alpha},\vY^{\mu\beta})$.
This object is defined for any architecture, since the time-dependent neural tangent kernel
exists whether or not training is lazy; what the lazy regime adds is that it is constant
along training, which is what makes a closed-form analysis possible, while outside it the
same matrix can still be measured, as we do in Sect.~\ref{sect:Numerical}. This is indeed the
case: the separation of timescales is read off that spectrum
(Theorem~\ref{thm:Spectrum_gram_linear}), through the two bulks located at
$\Theta(nm/d)$ and $\Theta(m)$, and hence the ratio
$\tau_{\mathrm{mem}}/\tau_{\mathrm{gen}}\sim\psi_n$.

\subsection{Concurrent work: \texorpdfstring{\citet{latourellevigeant2026generalizationmemorizationoverfittingdiffusion}}{Latourelle-Vigeant et al.}}
\label{app:subsec:concurrent}
The closest work to this paper is the concurrent study of
\citet{latourellevigeant2026generalizationmemorizationoverfittingdiffusion}. They also study the training dynamics of neural networks trained in the lazy regime on the score-matching task, in the linear regime $n\asymp d$. The main difference is that we fix the number of times each data point is noised to $1\le m<\infty$ and study the associated Gram matrix, while they take $m=\infty$ and focus on the infinite-dimensional kernel operator $\mathcal{K}$ on a vector-valued RKHS
\begin{align}
    (\mathcal{K}f)(\vx)=\int \dd\vy\; p_t^{\mathrm{emp}}(\vy)\, K(\vx,\vy)\, f(\vy),
\end{align}
where $p_t^{\mathrm{emp}}(\vy)$ is the noisy empirical distribution. They derive closed-form equations for the whole gradient-flow trajectory. Moreover, they characterize the generated distribution, which allows them to distinguish overfitting of the empirical score from memorization. Another difference is that they do not compute the bias--variance decomposition, which is an important part of our work.

\paragraph*{Three-timescale picture.} Instead of our two-timescale picture they describe three timescales: on $\tau=O(d)$ the model generalizes, as seen from a decrease in the test loss, then on timescales $\tau=O(n)$ the test loss starts to increase again. On $\tau=d^{1+\Theta(1)}$ the model starts to overfit without yet memorizing, and finally it memorizes on timescales of order $\tau=d^{\omega(1)}$.

\paragraph*{Linear versus polynomial regime.} Their analysis is confined to the linear scaling $n\asymp d$. Our
Sect.~\ref{sect:poly} and Fig.~\ref{fig:spectrum} (\textit{right}) show that the picture survives at $n\asymp d^k$:
generalization splits into $k$ sub-bulks indexed by polynomial degree, while memorization
remains a single $\Theta(m)$, $n$-independent component.

\paragraph*{Empirical evidence.} Their numerical experiments use synthetic
Gaussian data throughout, whereas our experiments of Sect.~\ref{sect:Numerical} use a CNTK on
CelebA and a finite-width U-Net outside the lazy regime.

\subsection{Discussion of \citet{han2024neural}}

\citet{han2024neural,han2026neural} study the training of a two-layer fully-connected ReLU
network by gradient descent to learn the score, and establish a minimax-optimal
generalization bound for the resulting estimator. Their starting point is that the statistical literature on diffusion
models is algorithm-agnostic: it assumes the empirical risk is minimized exactly, and says
nothing about whether gradient descent on a non-convex network actually reaches such a
minimizer. They close that gap by showing that the evolution of the trained network can be
approximated by a sequence of localized kernel regression problems, and by deriving an
early-stopping rule at
which the minimax rate is attained. The extended version \citep{han2026neural} adds that
explicit stopping rule together with the resulting estimation bound, neither of which appears
in the ICLR~2024 paper \citep{han2024neural}, and an experiment generating financial tabular
data on a credit-default dataset.

Two differences in the setting are worth noting. Their network takes the diffusion time as
an input, so the score is learned jointly across noise levels, whereas we work at fixed $t$
and study the spectrum of the Gram matrix at that noise level. Their sampling procedure draws
$N$ i.i.d.\ triples $(t_j,\vx_{0,j},\vx_{t_j})$, so that each clean sample is drawn afresh
and used once, at one diffusion time, with one noise realization, which cannot lead to memorization. In our notation this is
$m=1$. We warn the reader that $m$ denotes the
network width in \citet{han2024neural}, whereas here it denotes the number of noise
realizations per clean sample.

\section{Proofs of the analytical results}
\label{app:proof_analytical}
\subsection{Assumptions}
We assume that the data are of the form $\vx=\vSigma^{1/2}\vz$, where $\vz$ has independent zero-mean, unit-variance, sub-Gaussian entries, with a covariance $\vSigma$ such that $\lambda_{\mathrm{max}}(\vSigma) = O(1)$ and $\Tr(\vSigma)/d$ converges to a constant denoted by $\sigma^2$ as $d \to \infty$.
We focus on inner-product kernels of the form $K(\vx, \vy) = f(\vx^\top \vy / d)$, where $f$ is a smooth function. A canonical example is provided by the NTK of a multilayer neural network in the infinite-width limit.

To derive our results for the polynomial scaling regime ($n\asymp d^k$), we further assume that the kernel admits a diagonal representation in the basis of orthogonal polynomials:
\begin{align}
\label{eq:diagonal_kernel_hypothesis}
    K(\vx, \vy) = \sum_{p \ge 0} \sum_{\lvert \boldsymbol{j} \rvert = p} \frac{\mu_p}{d^p \boldsymbol{j}!} \psi_p^{\boldsymbol{j}}(\vx) \psi_p^{\boldsymbol{j}}(\vy),
\end{align}
where $\{\psi_p^{\boldsymbol{j}}\}$ are the (monic, unnormalized) polynomials obtained by orthogonalizing the monomials in $L^2(P_t)$ (see Lemma~\ref{lemma:ortho_polynomials} for the formal definition). This assumption implies that these polynomials are the eigenfunctions of the kernel. For the uniform distribution on the sphere $\mathbb{S}^{d-1}(\sqrt{d})$, every dot-product kernel is diagonal in the spherical harmonics, with eigenvalues depending only on the degree \citep{Bietti_2019, Ghorbani_2021}; for the isotropic Gaussian $\mathcal{N}(0, \sigma^2 \vI_d)$ the degree-only form holds asymptotically in $d$, the mixing between degrees (e.g.\ $(\vx^\top\vy)^2$ contains degree-$0$ terms) being subleading.

\subsection{Notation}
\label{app:subsec:notations}
Indices $\nu=1,\dots,n$ label the clean samples, $\alpha=1,\dots,m$ the noise realizations and $i=1,\dots,d$ the coordinates. We write $N=nm$ for the cardinality of the training set. $\vI_d$ is the identity matrix in dimension $d$, $\bm{1}_m$ the all-ones vector in dimension $m$ (likewise for $N$) and $\vB_m=\vI_n\otimes\bm{1}_m\bm{1}_m^\top$ is $m$ times the projector on $U=\{\vv^{\nu\alpha}\in \mathbb{R}^{nm}\;: \;\forall\nu,\;\alpha,\;\beta,\;\vv^{\nu\alpha}=\vv^{\nu\beta}\}$. We write $d_1=O(d_2)$ when $d_1/d_2$ remains bounded as $d_1,d_2\rightarrow\infty$, $d_1=\Theta(d_2)$ when both $d_1/d_2$ and $d_2/d_1$ remain bounded, and $d_1\asymp d_2$ when $d_1/d_2$ converges to a positive constant. The data distribution at $t=0$ is denoted $P_0$ while the noisy distribution at time $t>0$ is denoted $P_t$. We introduce $\vSigma=\mathbb{E}_{P_0}[\vx\vx^\top]\in\mathbb{R}^{d\times d}$ the data covariance and denote $\vv_\lambda\in\mathbb{R}^{d}$ its eigenvector associated with the eigenvalue $\lambda$. We write $\Delta_t=1-e^{-2t}$ for the noise variance, $\vSigma_t=e^{-2t}\vSigma+\Delta_t\vI_d$ for the covariance at noise level $t$ and $\lambda_t=e^{-2t}\lambda+\Delta_t$ for the eigenvalue of $\vSigma_t$ associated with $\vv_\lambda$; $\sigma^2=\Tr(\vSigma)/d$ and $\Gamma_t=\Tr(\vSigma_t)/d=e^{-2t}\sigma^2+\Delta_t$. Finally $\psi_n=n/d$ is the sample complexity and $\tilde\gamma=\gamma/(m\psi_n)$ the ridge rescaled by $N/d$, the scale of the generalization eigenvalues.

\subsection{Proof of the linear equivalent of the Gram matrix}
We first state the linear-equivalence theorem referred to in the main text.
\begin{thm}[Linear Equivalent of the Gram Matrix with Repeated Data]
\label{app:thm:lin_gram_empirical}
Consider the Gram matrix $\vG \in \mathbb{R}^{N \times N}$ ($N=nm$) defined by $\vG^{\nu\alpha, \mu\beta} = f\left(\frac{\sum_{i=1}^d\vY_i^{\nu\alpha}\vY_i^{\mu\beta}}{d} \right)$, with training samples $\vY^{\nu\alpha} = e^{-t}\vx^\nu + \sqrt{\Delta_t}\vxi^{\nu\alpha}$. In the limit $n, d \to \infty$ with $n/d \to \psi_n$ and $m = O(1)$, the matrix $\vG$ is asymptotically equivalent in spectrum to:
\begin{equation}
    \vG_{\mathrm{lin}} = \mu_I \vI_N + \mu_B \vB_m + \mu_0 \bm{1}_N \bm{1}_N^\top + \mu_1 \frac{\vY^\top \vY}{d},
\end{equation}
where $\vB_m = \vI_n \otimes \bm{1}_m \bm{1}_m^\top$ and the coefficients are given by:
\begin{small}
\begin{align*}
    &\mu_I = f(\sigma^2 e^{-2t} + \Delta_t) - f(\sigma^2 e^{-2t}) - \Delta_t f'(0), \quad &&\mu_B = f(\sigma^2 e^{-2t}) - f(0) - \sigma^2 e^{-2t} f'(0), \\
    &\mu_1 = f'(0), \quad &&\mu_0 = f(0) + \tfrac{1}{2d^2} f''(0) \Tr(\vSigma_t^2).
\end{align*}
\end{small}
\vspace{-1em}
\end{thm}

\begin{proof}
Write $a=(\nu,\alpha)\in[n]\times[m]$ for the $N=nm$ training indices, and
\[
    q^{ab} \;=\; \frac{1}{d}\sum_{i=1}^d \vY_i^{a}\vY_i^{b},
    \qquad
    \vY^{\nu\alpha} \;=\; e^{-t}\,\vx^{\nu} + \sqrt{\Delta_t}\,\vxi^{\nu\alpha},
\]
with $\vx^\nu\stackrel{\text{i.i.d.}}{\sim}P_0$ and
$\vxi^{\nu\alpha}\stackrel{\text{i.i.d.}}{\sim}\mathcal N(0,\vI_d)$, so that
$\vG^{ab}=f(q^{ab})$. Three index regimes have to be distinguished, according to how
much randomness the pair $(a,b)$ shares: $a=b$; $a\neq b$ within one cluster
($\nu=\mu$, $\alpha\neq\beta$); and different clusters ($\nu\neq\mu$).

\paragraph*{Step 1: uniform control of the overlaps.}
\begin{lemma}\label{app:lem:overlaps}
Under the assumptions of Sect.~\ref{app:proof_analytical} there is a constant $C$,
depending only on the sub-Gaussian norm of $P_0$ and on $\lambda_{\max}(\vSigma)$, such
that with probability $1-o(1)$,
\begin{align}
    &\max_{a}\big|q^{aa}-\Gamma_t\big| \;\le\; C\sqrt{\tfrac{\log d}{d}},
    \qquad
    \max_{\nu,\,\alpha\neq\beta}\big|q^{\nu\alpha,\nu\beta}-\sigma^2e^{-2t}\big|
      \;\le\; C\sqrt{\tfrac{\log d}{d}},\\
    &\max_{\nu\neq\mu}\big|q^{\nu\alpha,\mu\beta}\big| \;\le\; C\sqrt{\tfrac{\log d}{d}}.
\end{align}
\end{lemma}
Each of the three families consists of quadratic forms in independent sub-Gaussian
vectors, so each entry obeys a Bernstein-type bound with sub-exponential tails at scale
$d^{-1/2}$; a union bound over the $O(d^2)$ pairs
costs the factor $\sqrt{\log d}$. The centering constants are the means:
$\mathbb E[q^{aa}]=\Tr(\vSigma_t)/d=\Gamma_t$, $\mathbb
E[q^{\nu\alpha,\nu\beta}]=e^{-2t}\Tr(\vSigma)/d=\sigma^2e^{-2t}$ for $\alpha\neq\beta$
(the two noises are independent, only $\vx^\nu$ is shared), and $\mathbb
E[q^{\nu\alpha,\mu\beta}]=0$ for $\nu\neq\mu$.

We assume throughout that $f$ is three times continuously differentiable on a
neighborhood of $[-\delta,\Gamma_t+\delta]$ for some $\delta>0$, which is where
Lemma~\ref{app:lem:overlaps} confines every overlap. All the expansions below are taken
on the event of Lemma~\ref{app:lem:overlaps}.

\paragraph*{Step 2: the three regimes.}
On the diagonal, $f(q^{aa})=f(\Gamma_t)+f'(\Gamma_t)(q^{aa}-\Gamma_t)+O(d^{-1}\log d)$,
so the diagonal of $\vG$ equals $f(\Gamma_t)\vI_N$ up to a \emph{diagonal} matrix
$\vD$. Within a cluster, likewise, $f(q^{\nu\alpha,\nu\beta})=f(\sigma^2e^{-2t})+
f'(\sigma^2e^{-2t})(q^{\nu\alpha,\nu\beta}-\sigma^2e^{-2t})+O(d^{-1}\log d)$, so the
same-cluster off-diagonal part equals $f(\sigma^2e^{-2t})(\vB_m-\vI_N)$ up to a
\emph{block-diagonal} matrix $\vE$ with $n$ blocks of size $m\times m$. Across clusters
the overlaps are small and we expand at the origin,
\[
    f(q^{ab}) \;=\; f(0)+f'(0)\,q^{ab}+\tfrac12 f''(0)\,(q^{ab})^2 + R^{ab},
    \qquad
    |R^{ab}| \;\le\; \tfrac16\!\!\sup_{|u|\le\delta}\!|f'''(u)|\,\cdot|q^{ab}|^3 .
\]
Collecting the exact (non-remainder) pieces in the basis
$\{\vI_N,\ \vB_m-\vI_N,\ \bm{1}_N\bm{1}_N\tran-\vB_m,\ \vY\tran\vY/d\}$ and using
$\mathbb E[(q^{ab})^2]=\Tr(\vSigma_t^2)/d^2$ for $\nu\neq\mu$ gives exactly
$\vG_{\mathrm{lin}}$ with the four constants $\mu_I,\mu_B,\mu_0,\mu_1$ of the statement.

\paragraph*{Step 3: operator-norm bookkeeping.}
It remains to bound the residuals. $\vD$ is diagonal and $\vE$ is block diagonal with
blocks of fixed size $m=O(1)$, so their operator norms are controlled entrywise,
\[
    \lVert\vD\rVert_{\mathrm{op}} \;=\; \max_a|\vD^{aa}| \;=\; O\big(\sqrt{\log d/d}\big),
    \qquad
    \lVert\vE\rVert_{\mathrm{op}} \;\le\; m\max_{a\neq b}|\vE^{ab}| \;=\;
    O\big(\sqrt{\log d/d}\big).
\]
For the third-order remainder $\vR$ the Frobenius norm suffices, because $N\asymp d$:
\[
    \lVert\vR\rVert_{\mathrm{op}} \;\le\; \lVert\vR\rVert_{F}
    \;\le\; N\max_{\nu\neq\mu}|R^{ab}|
    \;=\; O\!\big(d\cdot (\log d/d)^{3/2}\big)
    \;=\; O\big(d^{-1/2}\log^{3/2}d\big).
\]
The quadratic term is the one place where entrywise bounds are not enough. The matrix $\vQ^{(2)}$ with entries $(q^{ab})^2$ for
$\nu\neq\mu$ has entries of order $1/d$ and dimension $N\asymp d$, so
$\lVert\vQ^{(2)}-\mathbb E\vQ^{(2)}\rVert_F=O(1)$: the Frobenius bound does
not vanish, and $\vQ^{(2)}$ concentrates around
$\tfrac{\Tr(\vSigma_t^2)}{d^2}(\bm{1}_N\bm{1}_N\tran-\vB_m)$. Nonetheless, as in \citet{elkaroui_2010}, one can prove that $\lVert\vQ^{(2)}-\mathbb E\vQ^{(2)}\rVert_{\mathrm{op}}=o(1)$.

\paragraph*{Step 4: conclusion.}
Summing the four contributions,
\[
    \lVert \vG-\vG_{\mathrm{lin}}\rVert_{\mathrm{op}}
    \;=\; O\big(d^{-1/2}\log^{3/2}d\big) + o(1) \;=\; o(1)
    \qquad\text{with probability } 1-o(1).
\]
By Weyl's inequality every eigenvalue of $\vG$ is within $o(1)$ of the corresponding
eigenvalue of $\vG_{\mathrm{lin}}$; in particular the two empirical spectral
distributions have the same weak limit, and any eigenvalue of $\vG_{\mathrm{lin}}$
separated from the rest of its spectrum by a gap bounded away from zero is tracked
individually. This is the sense in which $\vG$ and $\vG_{\mathrm{lin}}$ are used
interchangeably in the rest of the appendix.
\end{proof}
\subsection{Spectrum of the Gram matrix}

\begin{figure}
    \centering
    \includegraphics[width=0.4\linewidth]{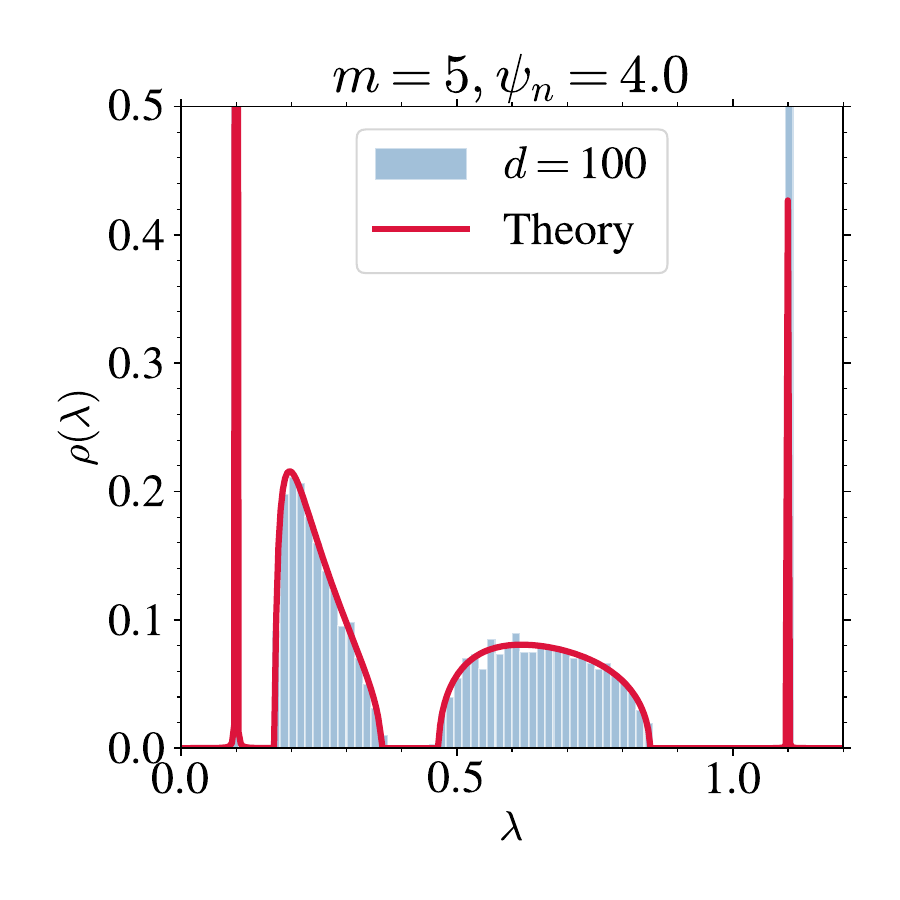}
    \includegraphics[width=0.4\linewidth]{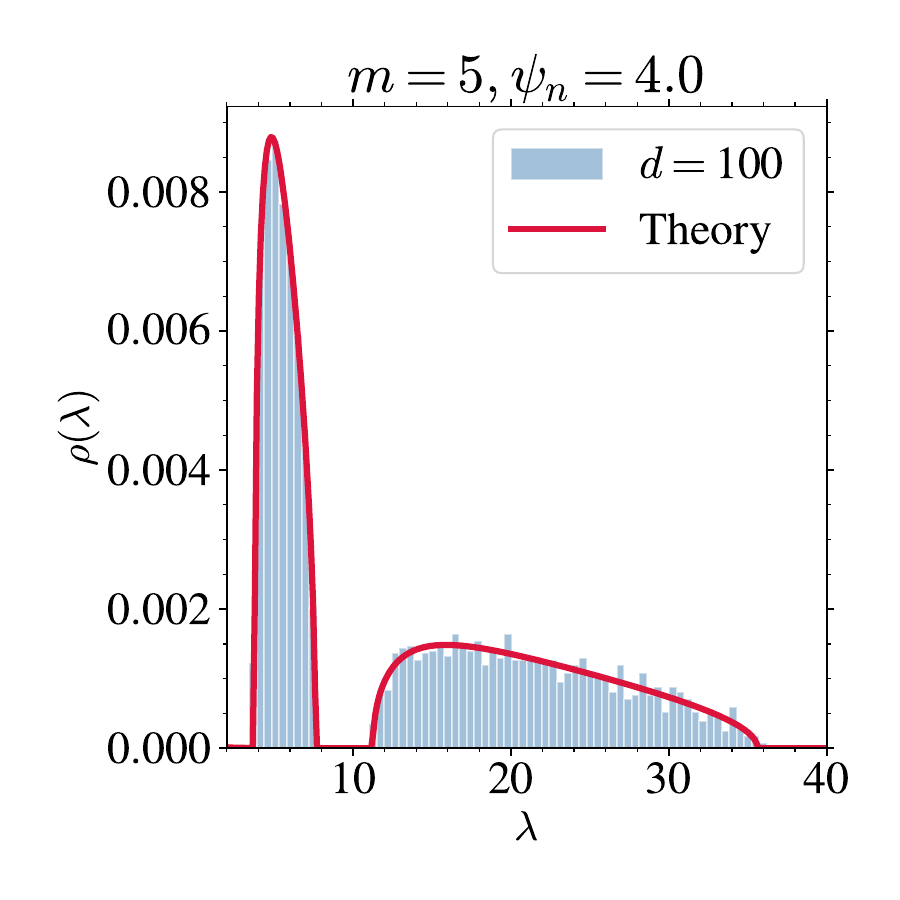}
    \caption{Empirical histogram of the eigenvalues of $\vG_{\mathrm{lin}}$ for $d=100$ (blue) averaged over 10 runs and the analytical prediction obtained by solving the equations on the Stieltjes transform for $m=5, \psi_n=4.0$ and $t=0.1$ for $\rho_{\vSigma}(\lambda)=\frac{1}{2}\delta(\lambda-1)+\frac{1}{2}\delta(\lambda-0.1)$ and for the kernel parameters $\mu_1=1,\mu_I=0.1,\mu_B=0.2$. The left and right panels corresponds to different ranges of $\lambda$.}
    \label{fig:spectrum_anisotropic_sigma}
\end{figure}

\begin{figure}
    \centering
    \includegraphics[width=0.5\linewidth]{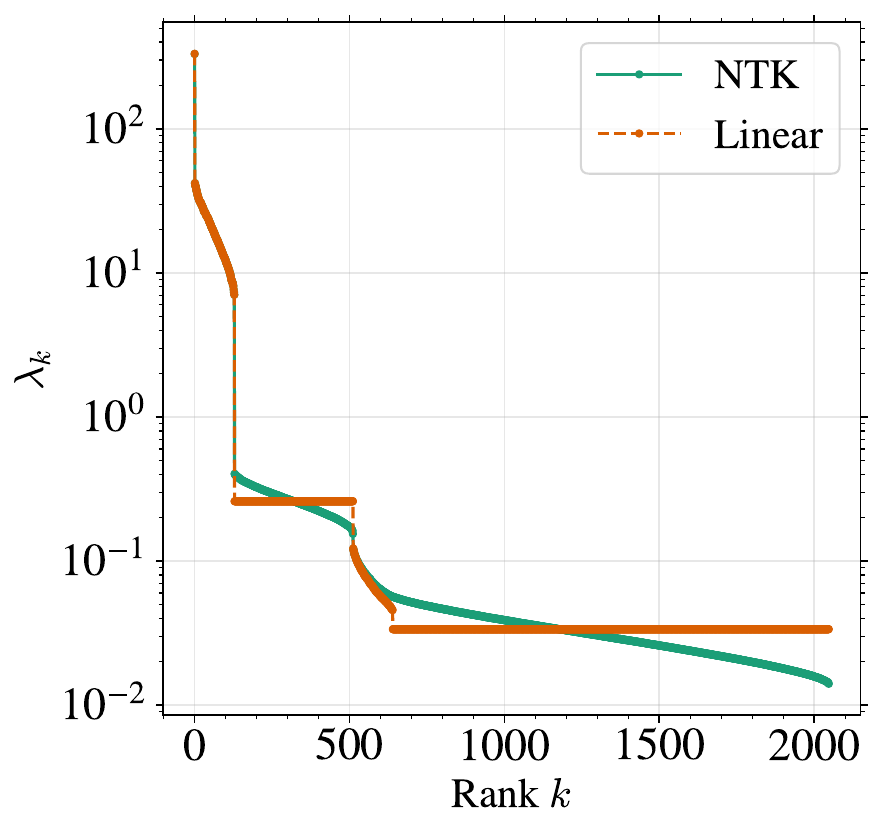}
    \caption{Comparison between the linear equivalent and the non-linear NTK spectrum in finite dimension for $d=128,\ \psi_n=4,\ m=4,$ and $t=0.1$.}
    \label{fig:comparaison_NTK_linear_G}
\end{figure}
We now state Theorem~\ref{thm:Spectrum_gram_linear} of the main text in full.
\begin{thm}[Spectral Distribution of the Gram Matrix]
\label{app:thm:Spectrum_gram_linear}
Let the empirical spectral density of the data covariance $\vSigma$ converge to $\rho_{\vSigma}(\lambda)$. In the high-dimensional limit $d, n \to \infty$ with $n/d \to \psi_n$, the Stieltjes transform $q(z)=\frac1N\Tr(\vG-z\vI_N)^{-1}$ and the auxiliary order parameter $r(z)$ of \eqref{eq:app_order_parameters} solve the system
\begin{align}
    z - \mu_I - m\mu_B+\frac{1}{mr}&= \int \dd \rho_{\vSigma}(\lambda)\frac{\mu_1 \Delta_t+m\mu_1e^{-2t}\lambda}{1+\mu_1 e^{-2t}\psi_nm^2\lambda r+\mu_1 \psi_n m\Delta_tq}, \\
    z-\mu_I+\frac{m-1}{m(q-r)} &=  \int \dd \rho_{\vSigma}(\lambda)\frac{\mu_1 \Delta_t}{1+\mu_1 e^{-2t}\psi_nm^2\lambda r+\mu_1 \psi_n m\Delta_tq}.
\end{align}
In the large-data regime ($\psi_n \gg 1$), the spectral density $\rho(\omega)$ of $\vG - \mu_0 \bm{1}_N \bm{1}_N^\top$, in the spectral variable $\omega$, admits a decomposition into four distinct components:
\begin{equation}
    \rho(\omega) \approx \underbrace{{\frac{m-1-1/\psi_n}{m} \delta(\omega - \mu_I)} + {\frac{1-1/\psi_n}{m} \delta(\omega - \bar{\mu})}}_{\text{atoms}} + \underbrace{\frac{1}{\psi_n m} \rho_1(\omega)}_{\text{generalization bulk}}
    + \underbrace{\frac{1}{\psi_n m} \rho_2(\omega)}_{\text{memorization bulk}} ,
\end{equation}
where $\bar{\mu} = m\mu_B + \mu_I$. To every eigenvalue $\lambda$ of $\vSigma$ there correspond
two eigenvalues of $\vG$, with the following asymptotic scalings, where
$\lambda_t = e^{-2t}\lambda + \Delta_t$ is the associated eigenvalue of
$\vSigma_t = e^{-2t}\vSigma + \Delta_t\vI_d$:
\begin{itemize}
    \item \textbf{Generalization bulk:} $\lambda_{\mathrm{gen}}(\lambda) \simeq \mu_1 \psi_n m \lambda_t=\Theta(nm/d)$.
    \item \textbf{Memorization bulk:} $\lambda_{\mathrm{mem}}(\lambda) \simeq \mu_I + \frac{\mu_B(m-1)\Delta_t}{\lambda_t}=\Theta(m)$.
\end{itemize}
\end{thm}
The proof of this theorem is presented in the next two subsections. In Sect.~\ref{sect:proof_replica}, we derive the self-consistent equations on the Stieltjes transform of $\vG$, while in Sect.~\ref{sect:asymp_spec_G} we establish the result on the asymptotic spectrum of $\vG$. An example of the spectrum for $\vSigma\neq \vI_d$ is presented in Fig.~\ref{fig:spectrum_anisotropic_sigma}, while the convergence of the two bulks $\rho_1,\ \rho_2$ is presented in Fig.\ref{fig:comparaison_NTK_linear_G} and Fig.~\ref{fig:convergence_spectrum_psi_n}.

\subsection{Replica computation of the Stieltjes transform of the Gram matrix}
\label{sect:proof_replica}

\paragraph*{The replica method, and in what sense our results are exact.}
In our context, we frequently need to compute the expected logarithm of a partition
function $\mathcal{Z}$ that arises from the Gaussian integral representation of the
resolvent, e.g., $\mathcal{Z} \propto \int \dd\vphi \, e^{-\frac{1}{2}\vphi^\top \vM
\vphi}$, where $\vM$ is a random matrix of interest. To compute such expectations of
logarithms of random variables, we employ the replica method \citep{mezard1987spin}. The
method rests on the identity
\begin{equation}
    \log \mathcal{Z} = \lim_{s\to 0} \frac{\mathcal{Z}^s - 1}{s}.
\end{equation}
Evaluating the integer moments $\mathbb{E}[\mathcal{Z}^s]$ for $s \in \mathbb{N}$, one
then continues the result analytically to $s \to 0$. This is the replica method of
statistical physics, a non-rigorous technique with a wide
range of applications, notably in statistical learning and the theory of neural networks
\citep{Ascoli_2020}. The replica method is widely believed to yield exact asymptotic
predictions, as has been established rigorously for a wide range of problems
\citep{guerra2002thermodynamic,talagrand2006parisi,barbier2019optimal,gerbelot2023asymptotic}.
More precisely, our results rely on a so-called replica-symmetry (RS) assumption. The RS
ansatz is known to hold for the trace of rational functions of random matrices, where it
has been shown to yield the exact same results as rigorous methods such as linear pencils
\citep{bodin2021model, george_2025, bonnaire2025diffusionmodelsdontmemorize,
vilucchio2025asymptotics}.

We first derive the self-consistent equations on the Stieltjes transform of $\vG$.
\begin{proof}
    Since $\vG$ and $\vG_{\mathrm{lin}}$ have asymptotically the same spectrum, we replace the non-linear Gram matrix by its linear equivalent. We compute the Stieltjes transform $q(z) = \tfrac{1}{N}\Tr(\vG_{\mathrm{lin}} - z\vI_N)^{-1}$
of the linear-equivalent Gram matrix
\[
\vG_{\mathrm{lin}} = \mu_I\,\vI_N + \mu_B\,\vB_m + \mu_0\,\bm{1}_N\bm{1}_N^\top + \mu_1\,\frac{\vY^\top\vY}{d}.
\]
The rank-one spike $\mu_0\bm{1}_N\bm{1}_N^\top$ contributes a single outlier eigenvalue
outside the bulk; we omit it from the replica calculation.

For $\Im z > 0$, write
\[
q(z) \;=\; \frac{2}{N}\,\partial_z \log \mathcal Z(z),
\qquad
\mathcal Z(z) \;=\; \int \dd \vphi\; e^{-\tfrac12 \vphi^\top(\vG_{\mathrm{lin}}-z\vI_N)\vphi}.
\]
Introducing $s$ replicas $\{\vphi^a\}_{a=1}^s$ and absorbing the identity term into a
shifted spectral parameter $z-\mu_I$,
\begin{align}
\mathbb{E}_{X,\Xi}[\mathcal Z^{s}]
\;=\; \int \prod_a \dd\vphi^a\,
e^{\frac12 (z-\mu_I)\,\vphi^a\ \cdot\ \vphi^a}\;
e^{-\frac{\mu_B}{2}\,\vphi^{a\top} \vB_m \vphi^a}\;
\mathbb{E}_Y\left[ e^{-\frac{\mu_1}{2d}\,\vphi^{a\top} \vY^\top \vY \vphi^a} \right].
\end{align}

\paragraph{Gaussian average over $\vY$.}
With sample-index pair $(\nu,\alpha)\in[n]\times[m]$ and $\vv=\bm{1}_m$, the second moment of $\vY$ is
\[
\vK_{ij}^{\nu\nu',\alpha\alpha'}
= \mathbb{E}\left[\vY_i^{\nu\alpha}\vY_j^{\nu'\alpha'}\right]
= e^{-2t}\,\vSigma_{ij}\vv^{\alpha}\vv^{\alpha'}\,\delta^{\nu\nu'}
+ \Delta_t\,\delta_{ij}\,\delta^{\nu\nu'}\delta^{\alpha\alpha'}.
\]
Setting $\vA_{ij}^{\nu\nu',\alpha\alpha'} = \delta_{ij}\sum_a \vphi^a_{\nu\alpha}\vphi^a_{\nu'\alpha'}$ and using the
Gaussian identity $\mathbb{E}_Y\left[ e^{-\tfrac{\mu_1}{2d}\Tr(\vY^\top \vA \vY)} \right]
= e^{-\frac12\log\det\left(\vI_d+\tfrac{\mu_1}{d}\vK\vA\right)}$, we obtain
\begin{align}
\mathbb{E}_{X,\Xi}[\mathcal Z^{s}]
\;=\; \int \prod_a \dd\vphi^a\,
e^{\frac{z-\mu_I}{2}\,\vphi^a\cdot\vphi^a}\,
e^{-\frac{\mu_B}{2}\,\vphi^{a\top} \vB_m \vphi^a}\,
e^{-\frac12\log\det\left(\vI_d+\tfrac{\mu_1}{d}\vK\vA\right)}.
\end{align}

\paragraph{Order parameters.}
Introduce the overlaps and their Lagrange-multiplier conjugates
\begin{align}\label{eq:app_order_parameters}
\vQ^{ab} \;=\; \tfrac{1}{nm}\,\vphi^a\cdot\vphi^b,
\qquad
\vR^{ab} \;=\; \tfrac{1}{nm^2}\,\vphi^a_{\nu\alpha}\,\vv^{\alpha}\vv^{\alpha'}\,\vphi^b_{\nu\alpha'},
\end{align}
via the identity
\begin{align}
1 \;=\; \int \dd\vQ \dd\hat \vQ\,\dd\vR\,\dd\hat \vR
e^{\frac12 \hat \vQ^{ab}(nm\,\vQ^{ab}-\vphi^a\cdot\vphi^b)
  + \frac12 \hat \vR^{ab}(nm^2\,\vR^{ab}-\vphi^a_{\nu\alpha}\vv^\alpha \vv^{\alpha'}\vphi^b_{\nu\alpha'})}.
\end{align}
The $\vB_m$ and $\vK$-determinant terms then become functions of $\vQ,\vR$ only:
\begin{align}
e^{-\frac{\mu_B}{2}\,\vphi^{a\top}\vB_m\vphi^a}\;
e^{-\frac12\log\det(\vI+\tfrac{\mu_1}{d}\vK\vA)}
=
e^{-\frac{\mu_B nm^2}{2}\Tr \vR}
e^{-\frac12\log\det\left(\vI+\tfrac{\mu_1}{d}(e^{-2t}nm^2\vSigma \vR + nm\Delta_t \vI_d\,\vQ)\right)}.
\end{align}
The remaining $\vphi$-integral is Gaussian:
\begin{align}
\int \dd\vphi^a\;
e^{-\frac12 \hat \vQ^{ab}\,\vphi^a\cdot\vphi^b
   -\frac12 \hat \vR^{ab}\,\vphi^a_{\nu\alpha}v^\alpha v^{\alpha'}\vphi^b_{\nu\alpha'}}
\;=\;
e^{-\frac{n}{2}\log\det(\hat \vQ^{ab}\delta_{\alpha\alpha'}+\hat \vR^{ab}\vv^\alpha \vv^{\alpha'})}.
\end{align}

\paragraph{Replica-symmetric ansatz.}
Set $\vQ^{ab}=q\delta^{ab}$, $\hat \vQ^{ab}=\hat q\delta^{ab}$, $\vR^{ab}=r\delta^{ab}$, $\hat \vR^{ab}=\hat r\delta^{ab}$ for all $a,b$. The
matrix $\hat q\,\delta_{\alpha\alpha'}+\hat r \vv^\alpha \vv^{\alpha'}$ has eigenvalues $\hat q$ (multiplicity
$m-1$) and $\hat q + m\hat r$ (multiplicity $1$), so
\[
\log\det(\hat q\,\delta_{\alpha\alpha'}+\hat r\,v^\alpha v^{\alpha'})
\;=\; (m-1)\log\hat q + \log(\hat q + m\hat r).
\]
Replacing the determinant over $\vSigma$ by an integral against its spectral measure
$\rho_{\vSigma}$ and dividing by $sd$, the per-replica action reads

\begin{align}
S(q,\hat q, r, \hat r)
\;=\;& -\,m\psi_n\,(z-\mu_I)\,q
\;+\; \psi_n(m-1)\log\hat q
\;+\; \psi_n\log(\hat q + m\hat r) \nonumber\\
&\;-\; \psi_n m\,q\hat q
\;-\; \psi_n m^2\,\hat r\,r
\;+\; \mu_B\,\psi_n m^2\,r \nonumber\\
&\;+\; \int \dd\rho_{\vSigma}(\lambda)\,
\log\left(1 + \mu_1\bigl(e^{-2t}\,\psi_n m^2\,\lambda\,r + \psi_n m\,\Delta_t\,q\bigr)\right).
\end{align}

\paragraph{Saddle-point equations.}
Stationarity $\partial_{\hat q}S=\partial_{\hat r}S=\partial_q S=\partial_r S=0$ gives

\begin{align}
m\,q \;&=\; \frac{m-1}{\hat q} \;+\; \frac{1}{\hat q + m\hat r}, \\[2pt]
m\,r \;&=\; \frac{1}{\hat q + m\hat r}, \\[2pt]
(z-\mu_I) + \hat q \;&=\; \int \dd\rho_{\vSigma}(\lambda)\;
\frac{\mu_1\,\Delta_t}{\,1 + \mu_1\,e^{-2t}\psi_n m^2\,\lambda\,r + \mu_1\,\psi_n m\,\Delta_t\,q\,}, \\[2pt]
\hat r \;&=\; \mu_B \;+\; \int \dd\rho_{\vSigma}(\lambda)\;
\frac{\mu_1\,e^{-2t}\,\lambda}{\,1 + \mu_1\,e^{-2t}\psi_n m^2\,\lambda\,r + \mu_1\,\psi_n m\,\Delta_t\,q\,}.
\end{align}
These four equations on $(q,r,\hat q, \hat r)$ can be reduced to the two equations on $(q,r)$ by eliminating $\hat q, \hat r$.
\end{proof}

\subsection{Asymptotic spectrum of the Gram matrix}
\label{sect:asymp_spec_G}

\begin{figure}
    \centering
     \includegraphics[width=0.47\linewidth]{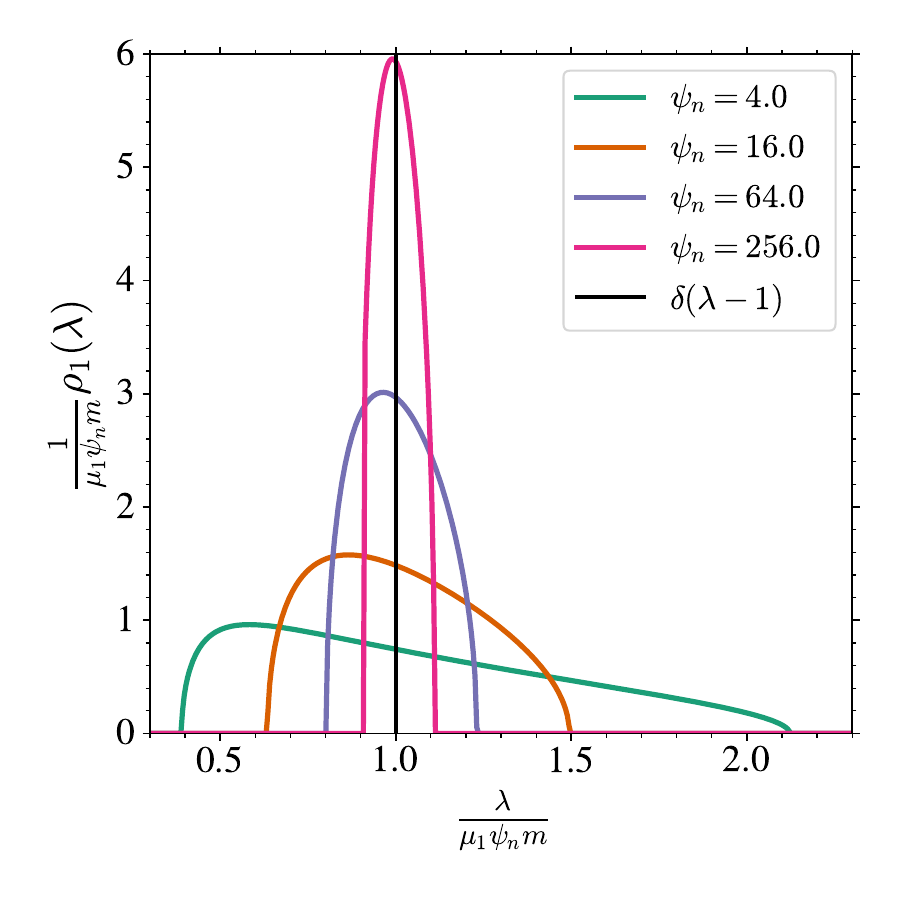}
    \includegraphics[width=0.47\linewidth]{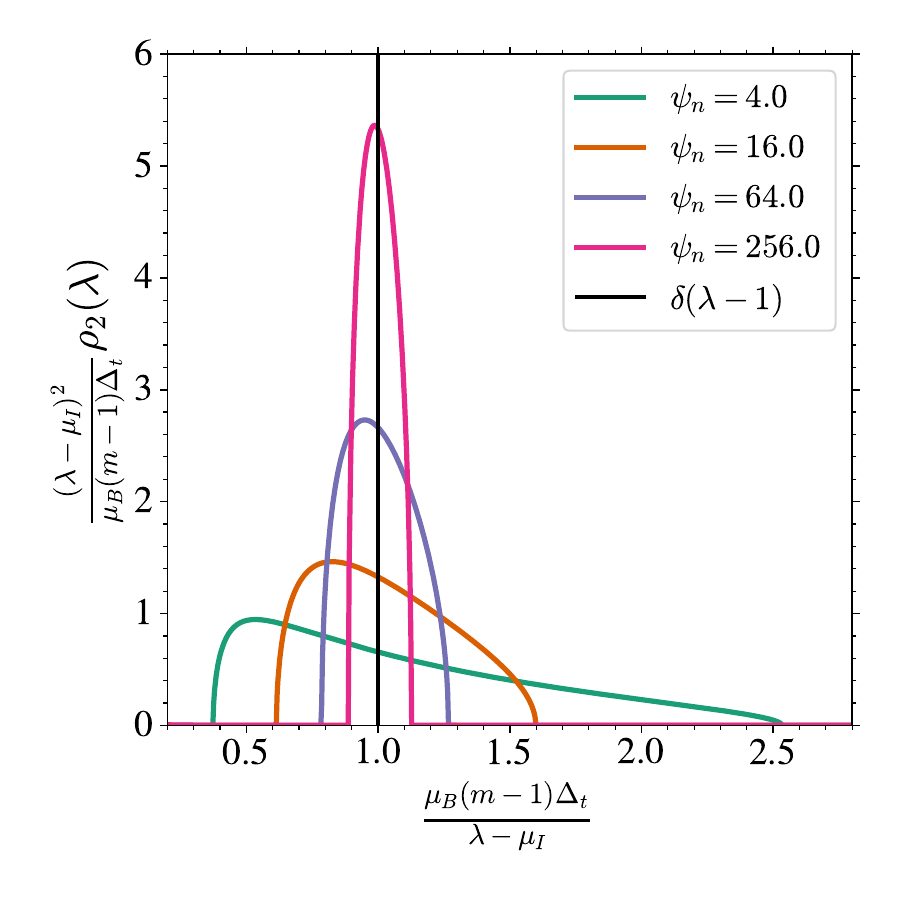}
    \caption{Analytical solutions of the rescaled bulks $\rho_1,\rho_2$ for different values of $\psi_n$ for $m=5,\ t=0.1,\vSigma=\vI_d,\ \mu_1=1.0,\ \mu_B=0.2,\ \mu_I=0.1,\ \mu_0=0.0 $ and the asymptotic limit $\delta(\lambda-1)$.}
    \label{fig:convergence_spectrum_psi_n}
\end{figure}

There are two ways to obtain the scaling and the structure of the spectrum of $\vG$: from the replica equations, and directly from linear-algebra arguments, the latter also giving the eigenvectors. We assume without loss of generality that $\mu_I=0$ since it only amounts to a constant shift in the spectrum.
\subsubsection{Derivation from the replica equations}
We solve the replica equations perturbatively to obtain the structure and the scalings of the spectrum of $\vG_{\mathrm{lin}}$.
\paragraph*{Peak at $z=m\mu_B$.} Assume $z=m\mu_B+\epsilon$ with $\epsilon\ll 1$ and plug in the ansatz $q=-\frac{q_0}{\epsilon}$ and $r=-\frac{r_0}{\epsilon}$. Then
\begin{align}
    \frac{1}{mr}&=m\hat{r}+\hat{q}=m\mu_B-z+\int \dd \rho_{\vSigma}(\lambda)\frac{\mu_1 (e^{-2t}\lambda m+\Delta_t)}{1+\mu_1 e^{-2t}\psi_nm^2\lambda r+\mu_1 \psi_n m\Delta_tq}\\
    &=-\epsilon(1+\int \dd \rho_{\vSigma}(\lambda)\frac{\mu_1 (e^{-2t}\lambda m+\Delta_t)}{\mu_1 e^{-2t}\psi_nm^2\lambda r_0+\mu_1 \psi_n m\Delta_tq_0})
\end{align}
and thus
\begin{align}
    mq&=\frac{m-1}{\hat{q}}+mr=\frac{m-1}{\int \dd \rho_{\vSigma}(\lambda)\frac{\mu_1 \Delta_t}{1+\mu_1 e^{-2t}\psi_nm^2\lambda r+\mu_1\psi_n m\Delta_tq}-z}+mr\\
    &=-\frac{1}{\epsilon}(1+\int \dd \rho_{\vSigma}(\lambda)\frac{\mu_1 (e^{-2t}\lambda m+\Delta_t)}{\mu_1 e^{-2t}\psi_nm^2\lambda r_0+\mu_1 \psi_n m\Delta_tq_0})^{-1}
\end{align}
at first order
\begin{align}
    &q=-\frac{1}{m}\frac{1}{\epsilon}\\
    &r=-\frac{1}{m}\frac{1}{\epsilon}
\end{align}
i.e.\ there are $n$ eigenvalues at this peak. Inserting $q=-1/(m\epsilon)$ back into the equation for $q$, the integral reduces to $1/\psi_n$ and expanding $(1+1/\psi_n)^{-1}\simeq1-1/\psi_n$ gives the correction
\begin{align}
    q=-\frac{1}{\epsilon}(\frac{1}{m}-\frac{1}{m\psi_n})
\end{align}
The correction yields $n-d$ eigenvalues.

\paragraph*{Peak at $z=0$.}
Assume $q=\frac{-q_0}{z}$ and $r=\frac{-r_0}{z}$.
\begin{align}
    \frac{1}{mr}=m\hat{r}+\hat{q}=m\mu_B-z+\int \dd \rho_{\vSigma}(\lambda)\frac{\mu_1 (e^{-2t}\lambda m+\Delta_t)}{1+\mu_1 e^{-2t}\psi_nm^2\lambda r+\mu_1 \psi_n m\Delta_tq}\sim m\mu_B
+O(z)
\end{align}
hence
\begin{align}
    mq&=\frac{m-1}{\hat{q}}\\
&=\frac{m-1}{\int \dd \rho_{\vSigma}(\lambda)\frac{\mu_1 \Delta_t}{1+\mu_1 e^{-2t}\psi_nm^2\lambda r+\mu_1 \psi_n m\Delta_tq}-z}\\
&=\frac{m-1}{\frac{z}{-\psi_nmq_0}-z}
\end{align}
This yields at first order
\begin{align}
    q_0=1-\frac{1}{\psi_nm}-\frac{1}{m}.
\end{align}
i.e. there are $nm-n-d$ eigenvalues at the peak.
\paragraph*{Bulk at $z=O(\psi_nm)$.} We solve the equations perturbatively, using the fact that at leading order $q\sim\frac{-1}{z}$ and $mr=\frac{-1}{z}$. Denote $z=\psi_nm \tilde{z}$. Then one has $mq=\frac{m-1}{\hat{q}}+mr$.
\begin{align}
    mq&=\frac{m-1}{\int \dd \rho_{\vSigma}(\lambda)\frac{\mu_1 \Delta_t}{1+\mu_1 e^{-2t}\psi_nm^2\lambda r+\mu_1 \psi_n m\Delta_tq}-z}\\
    &+\frac{1}{\int \dd \rho_{\vSigma}(\lambda)\frac{\mu_1(e^{-2t}\lambda m+ \Delta_t)}{1+\mu_1 e^{-2t}\psi_nm^2\lambda r+\mu_1 \psi_n m\Delta_tq}+m\mu_B-z}\\
    &=\frac{m-1}{\int \dd \rho_{\vSigma}(\lambda)\frac{\mu_1 \Delta_t}{1-\mu_1 e^{-2t}\lambda/\tilde{z} -\mu_1 \Delta_t/\tilde{z}}-z}+\frac{1}{\int \dd \rho_{\vSigma}(\lambda)\frac{\mu_1( e^{-2t}\lambda m+\Delta_t)}{1-\mu_1 e^{-2t}\lambda/\tilde{z} -\mu_1 \Delta_t/\tilde{z}}+m\mu_B-z}\\
    &\sim -\frac{m-1}{z}\left( 1+\frac{1}{z}\int\dd\rho_{\vSigma}\frac{\mu_1\Delta_t\tilde{z}}{\tilde{z}-\mu_1(e^{-2t}\lambda+\Delta_t)}\right)\\
    &-\frac{1}{z}\left(1+\frac{m}{z}\left(\mu_B+\int\dd\rho_{\vSigma}\frac{\mu_1 (e^{-2t}\lambda+\Delta_t/m)\tilde{z}}{\tilde{z}-\mu_1(e^{-2t}\lambda+\Delta_t)}\right)\right)\\
    &=-\frac{m}{z}-\frac{m}{z^2}\left(\mu_B+\int\dd\rho_{\vSigma}\frac{\mu_1(e^{-2t}\lambda+\Delta_t)\tilde{z}}{\tilde{z}-\mu_1(e^{-2t}\lambda+\Delta_t)}\right)
\end{align}
Hence at first order
\begin{align}
    q&=-\frac{1}{z}-\frac{1}{z^2}\left(\mu_B+\int\dd\rho_{\vSigma}\frac{\mu_1(e^{-2t}\lambda+\Delta_t)\tilde{z}}{\tilde{z}-\mu_1(e^{-2t}\lambda+\Delta_t)}\right)\\
    &=-\frac{1}{z}-\frac{1}{z^2}\left(\mu_B+\int\dd\bar{\rho}\frac{\bar{\lambda}\tilde{z}}{z-\bar{\lambda}}\right)\\
    &=-\frac{1}{z}-\frac{1}{\psi_n mz^2}\left(\psi_nm\mu_B+\int\dd\bar{\rho}\frac{\bar{\lambda}z}{z-\bar{\lambda}}\right)
\end{align}
by denoting $\bar{\lambda}=\psi_nm\mu_1(e^{-2t}\lambda+\Delta_t)$. Then
\begin{align}
    \frac{1}{q}=-z+\mu_B+\frac{1}{\psi_nm}\int\dd\bar{\rho}\frac{\bar{\lambda}z}{z-\bar{\lambda}}
\end{align}
We match this with a Bai--Silverstein equation \citep{SILVERSTEIN1995}. For
$\bm{\mathrm{Z}}\in\mathbb R^{N\times d}$ with i.i.d.\ standard Gaussian entries and aspect ratio
$c=d/N$, the Stieltjes transform $v(z)$ of the spectral density of
$\bm{\mathrm{Z}}\vSigma' \bm{\mathrm{Z}}^\top/d$, with $\vSigma'=\mu_1(e^{-2t}\vSigma+\Delta_t\vI_d)$ of spectral density $\rho'$, satisfies
\begin{align}
    \frac{1}{v(z)}&=-z+\int \frac{\dd \rho'(\lambda')\,\lambda'}{1+\lambda' v(z)/c}
    \sim-z+\int\frac{\dd\rho'(\lambda')\,\lambda' z}{z-\lambda'/c}
    =-z+c\int\frac{\dd\bar{\rho}(\bar\lambda)\,\bar\lambda z}{z-\bar\lambda},
\end{align}
at first order, using $v\sim-1/z$. With $c=1/(\psi_nm)$, this is the equation above: the
density has total mass $1/(\psi_n m)$, i.e.\ $d$ eigenvalues as expected, and at leading order
the generalization bulk is the non-zero spectrum of
\begin{align}
    \mu_1\frac{\bm{\mathrm{Z}}(e^{-2t}\vSigma+\Delta_t\vI_d)\bm{\mathrm{Z}}^\top}{d},\qquad
    \bm{\mathrm{Z}}\sim\mathcal{N}(0,\vI_{N\times d}),
\end{align}
the constant $\mu_B$ being an $O(1)$ shift, below the resolution of this expansion since the
bulk has width $\Theta(\sqrt{\psi_nm})$.

\paragraph*{Bulk at $z=O(m)$.}
To study the second bulk, we assume $z = m\tilde{z}$ where $\tilde{z} = O(1)$. In this regime the leading-order behavior of the Stieltjes transforms is fixed by the two atoms found above: the atom at $0$ carries a fraction $1-\tfrac1m-\tfrac{1}{\psi_n m}$ of the spectrum and the atom at $m\mu_B$ a fraction $\tfrac1m(1-\tfrac1{\psi_n})$, so that
\begin{align}
    q \sim -\frac{1}{z}\Big(1-\frac{1}{m}\Big)+\frac{1}{m(m\mu_B-z)},\qquad  mr \sim \frac{1}{m\mu_B - z}.
\end{align}
(Note that $q\sim-1/z$, which would correspond to putting \emph{all} the mass at the origin, is not consistent with the weight $q_0=1-\tfrac1m-\tfrac1{\psi_n m}$ obtained in the previous paragraph; keeping the correct weight is what produces the factor $m-1$ below.) We expand the saddle-point equation $mq = \frac{m-1}{\hat{q}} + \frac{1}{\hat{q} + m\hat{r}}$ by substituting the expressions for $\hat{q}$ and $\hat{r}$ and keeping terms up to $O(\frac{1}{\psi_n})$:

\begin{align}
    q &= \frac{m-1}{m}\,\frac{1}{-z + \int \dd\rho_{\vSigma}(\lambda) \frac{\mu_1 \Delta_t}{1 + \mu_1 \psi_n m (e^{-2t}\lambda\, mr + \Delta_t q)}} \nonumber \\
    &\quad + \frac{1}{m} \frac{1}{m\mu_B - z + \int \dd\rho_{\vSigma}(\lambda) \frac{\mu_1 (e^{-2t}\lambda m + \Delta_t)}{1 + \mu_1 \psi_n m (e^{-2t}\lambda\, mr + \Delta_t q)}} \nonumber \\
    &\approx \frac{-1}{z}\Big(1-\frac1m\Big) + \frac{1}{m(m\mu_B - z)} \\
    &- \frac{1}{m \psi_n z(m\mu_B - z)} \int \dd\rho_{\vSigma}(\lambda) \frac{m e^{-2t}\lambda z^2 + \Delta_t z^2 + (m-1)\Delta_t(m\mu_B - z)^2}{m e^{-2t}\lambda z + \Delta_t z - (m-1)\Delta_t(m\mu_B - z)}
\end{align}

To find the density $\rho(\omega)$ in the spectral variable $\omega$, we use the Stieltjes inversion formula $\rho(\omega) = \frac{1}{\pi} \lim_{\epsilon \to 0^+} \text{Im}[q(\omega + i\epsilon)]$. The imaginary part is generated by the pole in the integrand where the denominator $D(\lambda, \omega) = m e^{-2t}\lambda \omega + \Delta_t\omega - (m-1)\Delta_t(m\mu_B - \omega)=m\omega\lambda_t-m(m-1)\Delta_t\mu_B$ vanishes, with $\lambda_t=e^{-2t}\lambda+\Delta_t$. This occurs at:
\begin{equation}
    \lambda^*(\omega) = e^{2t} \Delta_t \left( \frac{(m-1)\mu_B}{\omega} - 1 \right)
\end{equation}
equivalently $\omega=(m-1)\mu_B\Delta_t/\lambda_t$, in agreement with the memorization eigenvalue $\lambda_{\mathrm{mem}}(\lambda)$ (at $\mu_I=0$) obtained by the linear-algebra route in Sect.~\ref{sect:asymp_spec_G_linalg} below.

Applying the Sokhotski--Plemelj theorem $\text{Im} \frac{1}{x - i\epsilon} = \pi \delta(x)$ and using $|\partial_\lambda D|=m\omega e^{-2t}$ at the pole, we obtain, since at $\lambda=\lambda^*(\omega)$ the numerator reduces to $(m-1)\Delta_t\,m\mu_B(m\mu_B-\omega)$,
\begin{equation}
    \rho(\omega) = \frac{e^{2t}\Delta_t (m-1)\mu_B}{\psi_n m\, \omega^2}\, \rho_{\vSigma} \left( e^{2t} \Delta_t \left( \frac{(m-1)\mu_B}{\omega} - 1 \right) \right)
    \label{eq:eigv_small_replicas}
\end{equation}
which is exactly the pushforward of $\rho_{\vSigma}$ by the map $\lambda\mapsto (m-1)\mu_B\Delta_t/\lambda_t$, carrying total mass
\begin{align}
    \int\dd\omega\, \rho(\omega)=\frac{1}{\psi_nm},
\end{align}
i.e.\ exactly $d$ eigenvalues. Restoring the shift $\mu_I$ amounts to replacing $\omega$ by $\omega-\mu_I$ in \eqref{eq:eigv_small_replicas}.

\paragraph*{Conclusion.} The density of eigenvalues is composed of
\begin{itemize}
    \item a delta peak at $\lambda=0$ with $nm-n-d$ eigenvalues, the null space of $\vG$ due to its finite rank;
    \item a bulk at $\lambda=\Theta(m)$ with $d$ eigenvalues, whose density is given by \eqref{eq:eigv_small_replicas};
    \item a delta peak at $\lambda=\mu_B m$ with $n-d$ eigenvalues;
    \item a bulk at $\lambda=\Theta(\psi_nm)$ with $d$ eigenvalues, which is that of the population kernel operator of the noisy distribution $P_t$.
\end{itemize}

\subsubsection{Derivation with linear algebra arguments}
\label{sect:asymp_spec_G_linalg}
In this subsection we derive the eigenvalues and eigenvectors of the linearized Gram matrix.
Besides re-deriving the eigenvalues of
Theorem~\ref{app:thm:Spectrum_gram_linear}, this route gives the eigenvectors, and proves
the following proposition, which restates Proposition~\ref{prop:eigenvectors_gram_linear}
of the main text.

\begin{proposition}[Eigenvectors of the Gram matrix]
\label{app:prop:eigenvectors_gram_linear}
For $\psi_n\gg1$ and every eigenvalue $\lambda$ of $\vSigma$ with associated eigenvector
$\vv_\lambda$, the eigenvectors of $\vG$ associated with the generalization and the
memorization bulks of Theorem~\ref{app:thm:Spectrum_gram_linear} are, at leading order,
\begin{align}\label{eq:app_eigenvectors}
    \vu_1^{\nu\alpha}\propto \vv_\lambda^\top\vY^{\nu\alpha},
    \qquad
    \vu_2^{\nu\alpha}\propto \vv_\lambda^\top\big[(me^{-2t}\lambda+\Delta_t)\sqrt{\Delta_t}\,\vxi_\perp^{\nu\alpha}
      -(m-1)\Delta_t\,\bar{\vY}^{\nu}\big],
\end{align}
with $\bar{\vY}^\nu=\frac1m\sum_\beta\vY^{\nu\beta}$ and
$\vxi_\perp^{\nu\alpha}=\vxi^{\nu\alpha}-\frac1m\sum_\beta\vxi^{\nu\beta}$. The eigenvectors
of the two atoms are the $\vvarphi\otimes\bm{1}_m/\sqrt m$ with $\vvarphi\in\mathrm{Ker}(\bar{\vY})$
and the elements of $\mathrm{Ker}(\vY)\cap\mathrm{Ker}(\vB_m)$; they do not contribute to the
estimated score.
\end{proposition}

\begin{proof}

As above we set $\mu_I=0$, which only shifts the whole spectrum by $\mu_I$, and we drop the
rank-one term $\mu_0\bm{1}_N\bm{1}_N^\top$, which contributes a single outlier outside the
bulk. We are left with the spectrum of the $nm \times nm$ matrix
\begin{equation}
    \vG = \mu_B \vB_m + \mu_1 \frac{\vY^\top \vY}{d}
\end{equation}
Let $\ve_{\nu\alpha}$ denote the canonical orthonormal basis, where $\nu \in \{1, \dots, n\}$ and $\alpha \in \{1, \dots, m\}$. We define the $n$ block-averaged vectors
\begin{equation}
    \vu^\nu = \frac{1}{\sqrt{m}}\sum_{\alpha=1}^m \ve_{\nu\alpha}
\end{equation}
which are orthonormal.

We introduce the $d$ feature vectors $\vphi_i$
\begin{equation}
    \vphi_i = \sqrt{m} e^{-t} \sum_\nu \vx_i^\nu \vu^\nu + \sqrt{\Delta_t} \sum_{\nu,\alpha} \vxi_i^{\nu\alpha} \ve_{\nu\alpha}
\end{equation}
where $\vx^\nu \sim \mathcal{N}(0, \vSigma)$ and $\vxi^{\nu\alpha} \sim \mathcal{N}(0, \vI_d)$ are mutually independent. We can write our Gram matrix as
\begin{equation}
    \vG = \mu_B m \sum_\nu \vu^\nu (\vu^\nu)^\top + \frac{\mu_1}{d}\sum_i \vphi_i  \vphi_i^\top
\end{equation}

We study the asymptotic spectrum in the limit $d \gg 1$ and $\psi_n = n/d \gg 1$, keeping $m$ finite. Let $\lambda$ and $\vv_i^\lambda$ denote the eigenvalues and eigenvectors of the covariance matrix $\vSigma$.

Because $m$ is finite, the block-averaged noise is not negligible. We split the noise into its block-average and a transverse component:
\begin{equation}
    \veta_i^\nu = \frac{1}{\sqrt{m}}\sum_{\alpha=1}^m \vxi_i^{\nu\alpha}, \qquad \vxi_{i,\perp}^{\nu\alpha} = \vxi_i^{\nu\alpha} - \frac{1}{\sqrt{m}}\veta_i^\nu
\end{equation}
By definition, $\sum_{\alpha} \vxi_{i,\perp}^{\nu\alpha} = 0$. Projecting onto the eigenbasis of $\vSigma$, we define the coordinates:
\begin{equation}
    z^{\lambda\nu} = \frac{1}{\sqrt{\lambda}}\sum_i \vv_i^\lambda \vx_i^\nu, \qquad \bar t^{\lambda\nu} = \sum_i \vv_i^\lambda \veta_i^\nu, \qquad s^{\lambda\nu\alpha} = \sum_i \vv_i^\lambda \vxi_{i,\perp}^{\nu\alpha}
\end{equation}
By Gaussianity and the orthogonality of the $\vv_i^\lambda$, the variables $z^{\lambda\nu}$ and $\bar t^{\lambda\nu}$ are independent $\mathcal{N}(0,1)$. The transverse variables $s^{\lambda\nu\alpha}$ are centered, orthogonal to the all-ones vector in $\alpha$, and satisfy
\begin{equation}
    \mathbb{E}[s^{\lambda\nu\alpha} s^{\lambda\nu\beta}] = \delta_{\alpha\beta} - \frac{1}{m}
\end{equation}
Thus, for large $n$, $\sum_{\nu,\alpha} (s^{\lambda\nu\alpha})^2 \simeq n(m-1)$.

For each eigenvalue $\lambda$, we define three normalized vectors:
\begin{align}
    \vU^\lambda &= \frac{1}{\sqrt{n}}\sum_\nu z^{\lambda\nu} \vu^\nu \\
    \vX^\lambda &= \frac{1}{\sqrt{n}}\sum_\nu \bar t^{\lambda\nu} \vu^\nu \\
    \vP^\lambda &= \frac{1}{\sqrt{n(m-1)}}\sum_{\nu,\alpha} s^{\lambda\nu\alpha}\ve_{\nu\alpha}
\end{align}
Because $z$, $\bar t$, and $s$ are independent and we are in the high-dimensional limit $n \gg 1$, the family $\{\vU^\lambda, \vX^\lambda, \vP^\lambda\}_{\lambda\in \mathrm{Spec}(\vSigma)}$ is asymptotically orthonormal. Furthermore, different $\lambda$ sectors do not mix at leading order.

We compute the overlaps of the feature vectors $\vphi_i$ with our basis. By standard concentration of measure, we find:
\begin{align}
     \vphi_i^\top\vU^\lambda &\simeq \sqrt{mn\lambda} e^{-t} \vv_i^\lambda \\
    \vphi_i^\top\vX^\lambda &\simeq \sqrt{n\Delta_t} \vv_i^\lambda \\
     \vphi_i^\top\vP^\lambda &\simeq \sqrt{n(m-1)\Delta_t} \vv_i^\lambda
\end{align}

Let $\kappa = \mu_1 \frac{n}{d}$. The operator $\mu_B\vB_m$ acts as $\mu_B m\,\mathrm{diag}(1,1,0)$ on the subspace $(\vU^\lambda, \vX^\lambda, \vP^\lambda)$, because $\vP^\lambda$ is entirely transverse to the block structure. The sample covariance part yields a rank-one update. Defining the vector $\vw_\lambda = (\sqrt{m\lambda} e^{-t}, \sqrt{\Delta_t}, \sqrt{(m-1)\Delta_t})^\top$, the operator $\vG$ projected onto this $3 \times 3$ subspace is asymptotically:
\begin{equation}
    \vM_\lambda^{(3)} \simeq \mu_B m
    \begin{pmatrix} 1 & 0 & 0 \\ 0 & 1 & 0 \\ 0 & 0 & 0 \end{pmatrix}
    + \kappa \vw_\lambda \vw_\lambda^\top
\end{equation}

To block-diagonalize this system, we introduce $D_\lambda = m\lambda e^{-2t} + \Delta_t$ and rotate the parallel sector into a mode coupled to the perturbation and an orthogonal mode:
\begin{align}
    \vPhi^\lambda &= \frac{\sqrt{m\lambda} e^{-t}\vU^\lambda + \sqrt{\Delta_t}\vX^\lambda}{\sqrt{D_\lambda}} \\
    \vPhi_\perp^\lambda &= \frac{\sqrt{\Delta_t}\vU^\lambda - \sqrt{m\lambda} e^{-t}\vX^\lambda}{\sqrt{D_\lambda}}
\end{align}
By construction, $\vPhi_\perp^\lambda$ is orthogonal to $\vw_\lambda$ and remains an exact eigenvector of $\vG$ with eigenvalue $\mu_B m$. The non-trivial spectrum is governed by the $2 \times 2$ matrix acting on the remaining coupled subspace $(\vPhi^\lambda, \vP^\lambda)$:
\begin{equation}
    \bm{\mathcal{M}}^\lambda \simeq
    \begin{pmatrix}
        \mu_B m + \kappa D_\lambda & \kappa \sqrt{(m-1)\Delta_t D_\lambda} \\
        \kappa \sqrt{(m-1)\Delta_t D_\lambda} & \kappa (m-1)\Delta_t
    \end{pmatrix}
\end{equation}

The characteristic equation for $\bm{\mathcal{M}}^\lambda$ yields two eigenvalues for each $\lambda$, $\lambda_{\mathrm{gen}}(\lambda)$ (upper sign) and $\lambda_{\mathrm{mem}}(\lambda)$ (lower sign):
\begin{equation}
\begin{split}
    \lambda_{\mathrm{gen}/\mathrm{mem}}(\lambda) &= \frac{\mu_B m + \kappa(D_\lambda + (m-1)\Delta_t)}{2}\\
    &\quad \pm \frac{1}{2}\sqrt{\left[\mu_B m + \kappa(D_\lambda - (m-1)\Delta_t)\right]^2 + 4\kappa^2(m-1)\Delta_t D_\lambda}
\end{split}
\end{equation}
In the strict limit $\kappa = \mu_1 \frac{n}{d} \gg 1$, we can expand these roots to isolate the dominant scaling:
\begin{align}
    \lambda_{\mathrm{gen}}(\lambda) &= \mu_1 \frac{nm}{d}\left(e^{-2t}\lambda+\Delta_t \right) + \frac{\mu_B  D_\lambda}{e^{-2t}\lambda+\Delta_t} + O\left(\frac{d}{n}\right) \\
    \lambda_{\mathrm{mem}}(\lambda) &= \frac{\mu_B (m-1)\Delta_t}{e^{-2t}\lambda+\Delta_t} + O\left(\frac{d}{n}\right)
\end{align}

This confirms that the spectrum separates into a large bulk scaling with $\psi_n=n/d$, which depends on the signal covariance $\vSigma$, and a smaller memorization bulk whose eigenvalues $\lambda_{\mathrm{mem}}(\lambda)$ are related to the spectrum of $\vSigma$ via the exact mapping:
\begin{equation}
    \lambda_{\mathrm{mem}}(\lambda) = \frac{\mu_B(m-1)\Delta_t}{\lambda e^{-2t} + \Delta_t}
\end{equation}

\paragraph*{Eigenvectors.} We have established that for each covariance mode $\lambda$, the operator $\vG = \mu_B \vB_m + \mu_1 \frac{\vY^\top \vY}{d}$ acts non-trivially on the three-dimensional subspace spanned by the asymptotically orthonormal vectors $(\vU^\lambda,\vX^\lambda,\vP^\lambda)$ introduced above, the sample covariance acting as a rank-one perturbation along $\vw_\lambda$. We now construct the eigenvectors of $\vG$ explicitly and relate them to the raw signal components $\vx_i^\nu$ and noise components $\vxi_i^{\nu\alpha}$.

The direction orthogonal to the perturbation $\vw_\lambda$ within the parallel subspace $(\vU^\lambda, \vX^\lambda)$ remains an exact eigenvector of $\vG$ with the unshifted eigenvalue $\mu_B m$. With $D_\lambda = m\lambda e^{-2t} + \Delta_t$ as above, this vector is
\begin{equation}
    \vPhi_\perp^\lambda = \frac{\sqrt{\Delta_t} \vU^\lambda - \sqrt{m\lambda} e^{-t} \vX^\lambda}{\sqrt{D_\lambda}}
\end{equation}
We now substitute the definitions of $\vU^\lambda$ and $\vX^\lambda$ in terms of the raw coordinates $z^{\lambda\nu} = \frac{1}{\sqrt{\lambda}}\sum_i \vv_i^\lambda \vx_i^\nu$ and $\bar{t}^{\lambda\nu} = \sum_i \vv_i^\lambda \veta_i^\nu$, where $\veta_i^\nu = \frac{1}{\sqrt{m}}\sum_\alpha \vxi_i^{\nu\alpha}$ is the block-averaged noise.
\begin{align}
    \vPhi_\perp^\lambda
    &= \frac{1}{\sqrt{n D_\lambda}} \sum_{i=1}^d \vv_i^\lambda \sum_{\nu=1}^n \left( \sqrt{\frac{\Delta_t}{\lambda}} \vx_i^\nu - \sqrt{m\lambda} e^{-t} \veta_i^\nu \right) \vu^\nu
\end{align}

The remaining two eigenvectors are orthogonal linear combinations of the coupled parallel state $\vPhi^\lambda$ and the transverse noise state $\vP^\lambda$, parametrized by a mixing angle $\theta_\lambda$ (we call them $\vE_{\lambda,\pm}$, rather than $\vu$, to avoid a clash with the block-averaged basis $\vu^\nu$; $\vE_{\lambda,+}\propto\vu_1$ and $\vE_{\lambda,-}\propto\vu_2$ are the eigenvectors of Proposition~\ref{app:prop:eigenvectors_gram_linear}, with eigenvalues $\lambda_{\mathrm{gen}}(\lambda)$ and $\lambda_{\mathrm{mem}}(\lambda)$ respectively):
\begin{align}
    \vE_{\lambda,+} &= \cos\theta_\lambda \vPhi^\lambda + \sin\theta_\lambda \vP^\lambda \\
    \vE_{\lambda,-} &= -\sin\theta_\lambda \vPhi^\lambda + \cos\theta_\lambda \vP^\lambda
\end{align}
where $\vPhi^\lambda = \frac{1}{\sqrt{D_\lambda}} \left( \sqrt{m\lambda} e^{-t}\vU^\lambda + \sqrt{\Delta_t}\vX^\lambda \right)$.

The angle is determined by the effective $2 \times 2$ block $\bm{\mathcal{M}}^\lambda$. In the asymptotic limit $n/d \gg 1$, the coupling parameter $\kappa = \mu_1 \frac{n}{d} \to \infty$, and the mixing angle simplifies to:
\begin{equation}
    \cos\theta_\lambda \approx \sqrt{\frac{D_\lambda}{D_\lambda + (m-1)\Delta_t}}, \qquad \sin\theta_\lambda \approx \sqrt{\frac{(m-1)\Delta_t}{D_\lambda + (m-1)\Delta_t}}
\end{equation}

Substituting the asymptotic angle into $\vE_{\lambda,+}$, we obtain:
\begin{equation}
    \vE_{\lambda,+} \approx \frac{1}{\sqrt{D_\lambda + (m-1)\Delta_t}} \left( \sqrt{D_\lambda} \vPhi^\lambda + \sqrt{(m-1)\Delta_t} \vP^\lambda \right)
\end{equation}
Expanding $\vPhi^\lambda$, the $\sqrt{D_\lambda}$ factor cancels:
\begin{equation}
    \vE_{\lambda,+} \propto \sqrt{m\lambda} e^{-t} \vU^\lambda + \sqrt{\Delta_t} \vX^\lambda + \sqrt{(m-1)\Delta_t} \vP^\lambda
\end{equation}
We now expand this entirely into the raw variables. Recall that $s^{\lambda\nu\alpha} = \sum_i \vv_i^\lambda \vxi_{i,\perp}^{\nu\alpha}$, where $\vxi_{i,\perp}^{\nu\alpha}$ is the transverse noise.
\begin{align}
    \vE_{\lambda,+} &\propto \frac{1}{\sqrt{n}} \sum_{i=1}^d \vv_i^\lambda \sum_{\nu=1}^n \left( \sqrt{m} e^{-t} \vx_i^\nu \vu^\nu + \sqrt{\Delta_t} \veta_i^\nu \vu^\nu + \sqrt{\Delta_t} \sum_{\alpha=1}^m \vxi_{i,\perp}^{\nu\alpha} \ve_{\nu\alpha} \right)
\end{align}
Recognizing that the block-averaged noise and transverse noise reconstruct the full noise vector exactly ($\veta_i^\nu \vu^\nu + \sum_\alpha \vxi_{i,\perp}^{\nu\alpha} \ve_{\nu\alpha} = \sum_\alpha \vxi_i^{\nu\alpha} \ve_{\nu\alpha}$), the term in parentheses is precisely the empirical data vector $\vphi_i$. Thus,
\begin{equation}
    \vE_{\lambda,+} \propto \frac{1}{\sqrt{n}} \sum_{i=1}^d \vv_i^\lambda \vphi_i
\end{equation}

Using the orthogonal rotation, the eigenvector corresponding to the memorization bulk is:
\begin{equation}
    \vE_{\lambda,-} \approx \frac{1}{\sqrt{D_\lambda + (m-1)\Delta_t}} \left( \sqrt{(m-1)\Delta_t} \vPhi^\lambda - \sqrt{D_\lambda} \vP^\lambda \right)
\end{equation}
Dropping the overall normalization and substituting the basis vectors, we express this mode in terms of the raw variables:
\begin{align}
    \vE_{\lambda,-} \propto \sum_{i=1}^d \vv_i^\lambda \sum_{\nu=1}^n \left[ \sqrt{(m-1)\Delta_t} \left( \sqrt{m} e^{-t} \vx_i^\nu + \sqrt{\Delta_t} \veta_i^\nu \right) \vu^\nu - \frac{D_\lambda}{\sqrt{m-1}} \sum_{\alpha=1}^m \vxi_{i,\perp}^{\nu\alpha} \ve_{\nu\alpha} \right]
\end{align}

This proves Proposition~\ref{app:prop:eigenvectors_gram_linear}. Indeed
$\vE_{\lambda,+}\propto\frac1{\sqrt n}\sum_i\vv_i^\lambda\vphi_i$ has components
$\vv_\lambda^\top\vY^{\nu\alpha}$, which is $\vu_1$. For $\vE_{\lambda,-}$, using
$\vu^\nu=\frac1{\sqrt m}\sum_\alpha\ve_{\nu\alpha}$ and
$\veta_i^\nu=\sqrt m\,\bar\vxi_i^\nu$, the bracket has $(\nu\alpha)$ component
$\sqrt{(m-1)\Delta_t}\,\bar{\vY}^\nu-\frac{D_\lambda}{\sqrt{m-1}}\vxi_\perp^{\nu\alpha}$
once projected on $\vv_\lambda$; multiplying by $-\sqrt{(m-1)\Delta_t}$ gives
$D_\lambda\sqrt{\Delta_t}\,\vxi_\perp^{\nu\alpha}-(m-1)\Delta_t\,\bar{\vY}^\nu$, which is
$\vu_2$ with $D_\lambda=me^{-2t}\lambda+\Delta_t$. Finally $\vPhi_\perp^\lambda$ and the
kernel of $\vG$ give the eigenvectors of the two atoms.

\paragraph*{The atoms do not contribute to the estimated score.} By
\eqref{eq:estimation_score_DSM}, the $k$-th coordinate of the estimator is
$\vs^k(\vy)=K(\vy,\vY)^\top h_\tau(\vG)\,(-\vxi_k/\sqrt{\Delta_t})$, where
$\vxi_k\in\mathbb R^N$ collects the $k$-th coordinates of the training noises and
$h_\tau(\lambda)=(1-e^{-d\tau\lambda/(nm)})/\lambda$ (or $h(\lambda)=1/(\lambda+\gamma)$ with a
ridge). Let $\vP$ be the orthogonal projector on an eigenspace $E$ of $\vG_{\mathrm{lin}}$ on
which $\vG_{\mathrm{lin}}$ acts as $c\,\vI$. The contribution of $E$ to the estimator is
\begin{align}
  \vs_E^k(\vy)=-\frac{h_\tau(c)}{\sqrt{\Delta_t}}\,K(\vy,\vY)^\top\vP\,\vxi_k ,
\end{align}
which vanishes at every $\tau$ and every ridge as soon as $\vP K(\vy,\vY)=0$. We now restore
the rank-one term and take the full linear equivalent
$\vG_{\mathrm{lin}}=\mu_I\vI_N+\mu_B\vB_m+\mu_0\bm1_N\bm1_N^\top+\mu_1\vY^\top\vY/d$, and the
linearized test kernel $K_{\mathrm{lin}}(\vy,\vY)=\mu_0\bm1_N+\mu_1\vY^\top\vy/d$, valid for a test
point $\vy$ independent of the training set since $\vy^\top\vY^{\nu\alpha}/d=O(d^{-1/2})$.
\begin{itemize}[leftmargin=1.2em]
  \item \emph{Atom at $\bar\mu=\mu_I+m\mu_B$.} For $\vvarphi\in\mathbb R^n$,
  $\vG_{\mathrm{lin}}(\vvarphi\otimes\bm1_m)=\bar\mu\,\vvarphi\otimes\bm1_m
  +\mu_0m(\bm1_n^\top\vvarphi)\bm1_N+\mu_1m\,\vY^\top\bar\vY\vvarphi/d$. The eigenspace is
  therefore $E_{\mathrm{par}}=\{\vvarphi\otimes\bm1_m:\ \vvarphi\in\mathrm{Ker}(\bar\vY),\
  \bm1_n^\top\vvarphi=0\}$. The second condition, which comes from the rank-one term, removes
  a single dimension out of $n-d$ and leaves the weight of the atom unchanged. For
  $\vvarphi\otimes\bm1_m\in E_{\mathrm{par}}$,
  \begin{align}
    (\vvarphi\otimes\bm1_m)^\top K_{\mathrm{lin}}(\vy,\vY)
    =\mu_0m\,\bm1_n^\top\vvarphi+\frac{\mu_1m}{d}\,\vy^\top\bar\vY\vvarphi=0 .
  \end{align}
  \item \emph{Atom at $\mu_I$.} For $\vpsi\in E_{\mathrm{trans}}=\mathrm{Ker}(\vY)\cap\mathrm{Ker}(\vB_m)$,
  $\vB_m\vpsi=0$ means $\sum_\alpha\vpsi^{\nu\alpha}=0$ for every $\nu$, hence
  $\bm1_N^\top\vpsi=0$, and $\vY\vpsi=0$; thus $\vG_{\mathrm{lin}}\vpsi=\mu_I\vpsi$ and
  \begin{align}
    \vpsi^\top K_{\mathrm{lin}}(\vy,\vY)=\mu_0\,\bm1_N^\top\vpsi+\frac{\mu_1}{d}\,\vy^\top\vY\vpsi=0 .
  \end{align}
\end{itemize}
For the linearized kernel both atoms therefore contribute exactly zero to the estimated
score, for every training time and every ridge. The non-linear remainder
$K_{\mathrm{res}}=K-K_{\mathrm{lin}}$ has entries
$f(\vy^\top\vY^{\nu\alpha}/d)-f(0)-f'(0)\vy^\top\vY^{\nu\alpha}/d=O(d^{-1})$. Its mean over
the entries is $\simeq f''(0)\vy^\top\vSigma_t\vy/(2d^2)$, i.e.\ proportional to $\bm1_N$, and
is annihilated by both projectors, since $\bm1_N\perp E_{\mathrm{par}}$ and
$\bm1_N\perp E_{\mathrm{trans}}$. Its fluctuating part has norm $O(\sqrt N/d)=O(\sqrt{m\psi_n/d})$;
if it is asymptotically uncorrelated with $\vP\vxi_k$, its contribution is of the same order
and vanishes as $d\to\infty$ at fixed $\psi_n$ and $m$.
\end{proof}

\subsection{Derivation of the spectrum in the polynomial scaling}
\label{appendix:polynomial_scaling}
For the polynomial regime we assume that the eigenfunctions of the kernel are given by a family of orthogonal polynomials.
\begin{lemma}
\label{lemma:ortho_polynomials}
    Let $P$ be a probability measure on $\mathbb{R}^d$ with finite moments of all orders. Assume further that the space of multivariate polynomials is dense in $L^2(P)$. Then, there exists a countable family of polynomials $\left\{\psi_p^{\boldsymbol{j}}\right\}_{\boldsymbol{j} \in \mathbb{N}^d}$, where $p = |\boldsymbol{j}|$ denotes the total degree, that forms an orthogonal basis of $L^2(P)$.
\end{lemma}

\begin{proof}
    We equip $L^2(P)$ with the inner product $\langle f, g \rangle = \int_{\mathbb{R}^d} f(x)g(x) \, \dd P(x)$. Let $\boldsymbol{j} = (j_1, \dots, j_d) \in \mathbb{N}^d$ be a multi-index, with total degree $p = |\boldsymbol{j}| = \sum_{i=1}^d j_i$. Because $P$ has finite moments of all orders, every monomial $x^{\boldsymbol{j}} = x_1^{j_1} \cdots x_d^{j_d}$ belongs to $L^2(P)$.

    To construct an orthogonal basis, we establish a strict well-ordering on the set of multi-indices $\mathbb{N}^d$. We choose the graded lexicographic order, denoted by $\prec$. For two multi-indices $\boldsymbol{j}, \boldsymbol{l}$, we say $\boldsymbol{j} \prec \boldsymbol{l}$ if $|\boldsymbol{j}| < |\boldsymbol{l}|$, or if $|\boldsymbol{j}| = |\boldsymbol{l}|$ and the first non-zero entry in the difference $\boldsymbol{j} - \boldsymbol{l}$ is negative.

    We now apply the Gram--Schmidt orthogonalization process to the sequence of monomials ordered by $\prec$, by induction along $\prec$ (every multi-index has finitely many predecessors).

    Let $\boldsymbol{0} = (0, \dots, 0)$. For the base case (where $p=0$), we define:
    \begin{equation}
        \psi_0^{\boldsymbol{0}}(x) = 1.
    \end{equation}

    For the inductive step, let $\boldsymbol{j} \in \mathbb{N}^d$ with total degree $p = |\boldsymbol{j}|$. Assume we have constructed a mutually orthogonal set of polynomials $\left\{\psi_q^{\boldsymbol{l}}(x)\right\}_{\boldsymbol{l} \prec \boldsymbol{j}}$, where $q = |\boldsymbol{l}|$. We define the polynomial corresponding to $\boldsymbol{j}$ by projecting the monomial $x^{\boldsymbol{j}}$ onto the orthogonal complement of the span of all preceding polynomials:
    \begin{align}
        \psi_p^{\boldsymbol{j}}(x) &= x^{\boldsymbol{j}} - \sum_{\boldsymbol{l} \prec \boldsymbol{j}} \frac{\langle x^{\boldsymbol{j}}, \psi_q^{\boldsymbol{l}} \rangle}{\langle \psi_q^{\boldsymbol{l}}, \psi_q^{\boldsymbol{l}} \rangle} \psi_q^{\boldsymbol{l}}(x).
    \end{align}

    By construction, $\psi_p^{\boldsymbol{j}}(x)$ is a polynomial of total degree $p$ with leading term $x^{\boldsymbol{j}}$, and $\langle \psi_p^{\boldsymbol{j}}, \psi_q^{\boldsymbol{l}} \rangle = 0$ for all $\boldsymbol{l} \prec \boldsymbol{j}$.

    This process yields a countable family of mutually orthogonal polynomials. Because the span of this family equals the span of all multivariate monomials, and polynomials are dense in $L^2(P)$, this family $\left\{\psi_p^{\boldsymbol{j}}\right\}_{\boldsymbol{j} \in \mathbb{N}^d}$ is an orthogonal basis of $L^2(P)$.
\end{proof}

\begin{remark}
    The orthogonal basis $\{\psi_p^{\boldsymbol{j}}\}$ constructed above reduces to classical families for the standard data distributions: products of one-dimensional Hermite polynomials for the standard Gaussian \citep{Bach2023Hermite}, and spherical harmonics for data uniform on the sphere.
\end{remark}
This kernel form is natural because, for data uniform on the sphere, every dot-product kernel is diagonalized by spherical harmonics \citep{Ghorbani_2021,Bietti_2019}.

We first derive a polynomial equivalent of the Gram matrix in the regime $d^{k+1}\gg n\gg d^{k}$ for a kernel satisfying the diagonal hypothesis \eqref{eq:diagonal_kernel_hypothesis}, with $\psi_p^{\boldsymbol{j}}$ the polynomials of Lemma~\ref{lemma:ortho_polynomials}.
\begin{proposition}[Polynomial Equivalent of the Gram matrix]
\label{app:prop:poly_equiv}
   In the regime $d^{k+1}\gg n\gg d^k\gg 1$, the Gram matrix is equivalent to
    \begin{align}
        \vG=\mu^{>k}_I\vI_N+\mu_B^{>k}\vB_m+\sum_{p=0}^k\sum_{|\boldsymbol{j}|=p}\frac{\mu_p}{d^p\boldsymbol{j}!}\psi_p^{\boldsymbol{j}}(\vY)\psi_p^{\boldsymbol{j}}(\vY)^\top
    \end{align}
    with the constants defined as
\begin{gather}
    \mu_B^{>k} \;=\; \sum_{p > k} \frac{\mu_p}{d^p}\,
    \mathbb{E}_{\vY,\vY'}\left[\sum_{|\boldsymbol{j}| = p} \frac{\psi_p^{\boldsymbol{j}}(\vY)\,\psi_p^{\boldsymbol{j}}(\vY')}{\boldsymbol{j}!}\right],
    \nonumber\\
    \mu_I^{>k} \;=\; \sum_{p > k} \frac{\mu_p}{d^p}\,
    \mathbb{E}_{\vY}\left[\sum_{|\boldsymbol{j}| = p} \frac{\psi_p^{\boldsymbol{j}}(\vY)^2}{\boldsymbol{j}!}\right] \;-\; \mu_B^{>k},
\end{gather}
with $\vY, \vY'$ two conditionally independent same-cluster noisy copies (i.e., sharing the
underlying $\vx^\nu$ but with independent noise $\vxi$).
\end{proposition}
\begin{proof}
    We split the kernel by polynomial degree at the cutoff $k$:
\begin{equation}
\label{eq:poly:split}
    K \;=\; K_{\le k} \;+\; K_{>k}, \qquad
    K_{\le k}(\vx, \vy) \;:=\; \sum_{p=0}^{k} \sum_{|\boldsymbol{j}| = p}\frac{\mu_p}{d^p\,\boldsymbol{j}!}\,
    \psi_p^{\boldsymbol{j}}(\vx)\,\psi_p^{\boldsymbol{j}}(\vy),
\end{equation}
which induces $\vG = \vG_{\le k} + \vG_{>k}$. The low-degree part is, by
construction, exactly the explicit double sum in the statement. We
must therefore show that, in operator norm,
\begin{equation}
\label{eq:poly:goal}
    \vG_{>k} \rightarrow \mu_I^{>k}\,\vI_N \;+\; \mu_B^{>k}\,\vB_m
\end{equation}
We treat the three index regimes (diagonal, same cluster, distinct clusters)
separately.

Recall $\vB_m = \vI_n \otimes \bm{1}_m\bm{1}_m^\top$ has entries
$(\vB_m)_{(\nu\alpha),(\mu\beta)} = \mathbf{1}\{\nu = \mu\}$. The matrix
$\mu_I^{>k}\,\vI_N + \mu_B^{>k}\,\vB_m$ takes the value
$\mu_I^{>k} + \mu_B^{>k}$ on the diagonal, $\mu_B^{>k}$ on the
same-cluster off-diagonal, and $0$ on cross-cluster pairs. We verify each.

\paragraph{Diagonal entries.} For $(\nu, \alpha) = (\mu, \beta)$:
\[
    \big(\vG_{>k}\big)_{(\nu\alpha),(\nu\alpha)} \;=\; K_{>k}(\vY^{\nu\alpha}, \vY^{\nu\alpha})
    \;=\; \sum_{p > k} \frac{\mu_p}{d^p}\sum_{|\boldsymbol{j}| = p} \frac{\psi_p^{\boldsymbol{j}}(\vY^{\nu\alpha})^2}{\boldsymbol{j}!}.
\]
Each summand is a function of the single sample $\vY^{\nu\alpha}$ alone.
By the law of large numbers in $L^2(P_t)$, the inner sum
$\sum_{|\boldsymbol{j}| = p}\psi_p^{\boldsymbol{j}}(\vY^{\nu\alpha})^2/\boldsymbol{j}!$
concentrates around its expectation under $\vY^{\nu\alpha} \sim P_t$, with
fluctuations a factor $\sqrt{1/d^p}$ smaller than the mean (since the rank
of the Mercer block at degree $p$ is $\sim d^p/p!$). Summing over $p > k$
and dividing by $d^p$ leaves the deterministic value
\[
    \big(\vG_{>k}\big)_{(\nu\alpha),(\nu\alpha)} \;=\; \mu_I^{>k} + \mu_B^{>k} + o_P(1)
\]
uniformly in $(\nu, \alpha)$.

\paragraph{Same-cluster off-diagonal entries.} For
$\nu = \mu, \alpha \ne \beta$, $\vY^{\nu\alpha}$ and $\vY^{\nu\beta}$ share
the underlying $\vx^\nu$ but are conditionally independent given $\vx^\nu$
(through the noise $\vxi$):
\[
    \big(\vG_{>k}\big)_{(\nu\alpha),(\nu\beta)} \;=\; K_{>k}(\vY^{\nu\alpha}, \vY^{\nu\beta})
    \;=\; \sum_{p > k} \frac{\mu_p}{d^p}\sum_{|\boldsymbol{j}| = p} \frac{\psi_p^{\boldsymbol{j}}(\vY^{\nu\alpha})\psi_p^{\boldsymbol{j}}(\vY^{\nu\beta})}{\boldsymbol{j}!}.
\]
Conditional on $\vx^\nu$, each summand has expectation
$\mathbb{E}[\psi_p^{\boldsymbol{j}}(\vY^{\nu\alpha})\,|\,\vx^\nu]\cdot
 \mathbb{E}[\psi_p^{\boldsymbol{j}}(\vY^{\nu\beta})\,|\,\vx^\nu]$. Marginalizing
over $\vx^\nu$, this is precisely the joint expectation appearing in the
definition of $\mu_B^{>k}$. The same concentration argument as for the diagonal entries
gives, uniformly in $\nu, \alpha \ne \beta$,
\[
    \big(\vG_{>k}\big)_{(\nu\alpha),(\nu\beta)} \;=\; \mu_B^{>k} + o_P(1).
\]

\paragraph{Cross-cluster entries.}
For $\nu \ne \mu$, $\vY^{\nu\alpha}$ and $\vY^{\mu\beta}$ are fully independent.
By orthogonality of $\{\psi_p^{\boldsymbol{j}}\}_{p \ge 1}$ to constants
in $L^2(P_t)$, $\mathbb{E}[\psi_p^{\boldsymbol{j}}(\vY)] = 0$ for all $p \ge 1$,
and therefore for $p > k \ge 1$,
\[
    \mathbb{E}\left[\big(\vG_{>k}\big)_{(\nu\alpha),(\mu\beta)}\right]
    \;=\; \sum_{p > k} \frac{\mu_p}{d^p}\sum_{|\boldsymbol{j}| = p} \frac{\mathbb{E}[\psi_p^{\boldsymbol{j}}(\vY)]^2}{\boldsymbol{j}!} \;=\; 0.
\]
Hence the cross-cluster entries are mean-zero. Since $\vY^{\nu\alpha}$ and $\vY^{\mu\beta}$
are independent and the $\psi_p^{\boldsymbol{j}}$ are orthogonal in $L^2(P_t)$, all the
cross terms $(p,\boldsymbol{j})\neq(q,\boldsymbol{l})$ vanish in the second moment,
which is therefore a sum of squares:
\[
    \mathrm{Var}\big[(\vG_{>k})_{(\nu\alpha),(\mu\beta)}\big]
    \;=\; \sum_{p > k} \frac{\mu_p^2}{d^{2p}}
    \sum_{|\boldsymbol{j}| = p}
    \frac{\big(c_p^{\boldsymbol{j}}\big)^2}{(\boldsymbol{j}!)^2}
    \;\sim\; \sum_{p > k} \frac{C_p}{d^p},
\]
where $c_p^{\boldsymbol{j}}=\mathbb{E}_{\vy\sim P_t}[\psi_p^{\boldsymbol{j}}(\vy)^2]$
and the last estimate uses that there are $\Theta(d^p/p!)$ multi-indices of degree $p$.
Let $\vG_{>k}^{\mathrm{cross}} \in \mathbb{R}^{N \times N}$ denote the matrix
obtained from $\vG_{>k}$ by zeroing out the diagonal and same-cluster
blocks; its entries are mean-zero and weakly correlated. By the standard
operator-norm bound for random matrices with i.i.d.-like entries of
variance $\sigma^2$, we have $\lVert \vG_{>k}^{\mathrm{cross}}\rVert_{\mathrm{op}}
\le C\sqrt{N\sigma^2}$. Substituting $N = nm$ and $\sigma^2 \sim 1/d^p$
for the dominant degree $p = k + 1$:
\[
    \lVert \vG_{>k}^{\mathrm{cross}}\rVert_{\mathrm{op}} \;\lesssim\;
    \sqrt{\frac{nm}{d^{k+1}}} \;=\; \sqrt{m \cdot \frac{n}{d^{k+1}}}.
\]
The hypothesis $n \ll d^{k+1}$ then gives $\lVert \vG_{>k}^{\mathrm{cross}}\rVert_{\mathrm{op}} = o_d(1)$.

The entries of $\vG_{>k}$ on the
diagonal, same-cluster off-diagonal and cross-cluster blocks match those of
$\mu_I^{>k}\vI_N + \mu_B^{>k}\vB_m$ up to fluctuations of order $o_P(1)$
entrywise on the structured part, and up to a residual matrix with
operator norm $o_d(1)$ on the cross-cluster part. The structured residual
fluctuations are themselves of operator norm $o_d(1)$ (by the same
$\sqrt{N\sigma^2}$ bound applied to each cluster block). Therefore
\[
    \big\lVert\vG_{>k} - \mu_I^{>k}\vI_N - \mu_B^{>k}\vB_m\big\rVert_{\mathrm{op}} \;=\; o_d(1),
\]
which combined with $\vG = \vG_{\le k} + \vG_{>k}$ proves the claimed
operator-norm equivalence.
\end{proof}

\begin{thm}[Spectrum of the Gram matrix, polynomial regime]
\label{app:thm:poly}
   Let $c_p^{\boldsymbol{j}}=\mathbb{E}_{\vy\sim P_t}[\psi_p^{\boldsymbol{j}}(\vy)^2]$ and $c_{X,p}^{\boldsymbol{j}}=\mathbb{E}_{\vx\sim P_0;\  \vxi,\ \vxi'\sim\mathcal{N}(0,\vI_d)}[\psi_p^{\boldsymbol{j}}(e^{-t}\vx+\sqrt{\Delta_t}\vxi)\psi_p^{\boldsymbol{j}}(e^{-t}\vx+\sqrt{\Delta_t}\vxi')]$. In the scaling $d^{k+1}\gg n\gg d^k$ with $m=O(1)$, the spectrum of the Gram matrix has the following structure:
    \begin{itemize}
        \item for every $p\in\{1,\dots, k\}$, a group of $\Theta(d^p/p!)$ eigenvalues of order $nm/d^p$, with eigenvectors asymptotically $\psi_p^{\boldsymbol{j}}(\vY)$;
        \item a bulk of $\Theta(d^k/k!)$ eigenvalues
        $\mu_I^{>k}+\mu_B^{>k}(m-1)\big(1-c_{X,k}^{\boldsymbol{j}}/c_k^{\boldsymbol{j}}\big)$, with eigenvectors $\Lambda_\perp^{\boldsymbol{j}} \frac{1}{m}\vB_m \psi_k^{\boldsymbol{j}}(\vY)-\Lambda_\parallel^{\boldsymbol{j}} (\vI_N-\frac{1}{m}\vB_m) \psi_k^{\boldsymbol{j}}(\vY)$, where $\Lambda_\parallel^{\boldsymbol{j}}=\frac{\mu_k}{\boldsymbol{j}!}\left(\frac{c_k^{\boldsymbol{j}}}{m}+\frac{m-1}{m}c_{X,k}^{\boldsymbol{j}}\right)$ and $\Lambda_{\perp}^{\boldsymbol{j}}=\frac{\mu_k}{\boldsymbol{j}!}(1-\frac{1}{m})\left(c_k^{\boldsymbol{j}}-c_{X,k}^{\boldsymbol{j}}\right)$ (at $k=1$ and for Gaussian data, $c_{X,1}/c_1=e^{-2t}$ and this is the memorization bulk of Theorem~\ref{app:thm:Spectrum_gram_linear})
        \item a spike at $m\mu_B^{>k}$ with, at leading order, $n-d^k/k!$ eigenvalues;
        \item an atom at $\lambda=\mu_I^{>k}$ with $nm-n-d^k/k!$ eigenvalues.
    \end{itemize}
\end{thm}
\begin{proof}
\textbf{Case $p<k$.} The rank of the term $\sum_{p=0}^{k-1}\sum_{\boldsymbol{j}}\frac{\mu_p}{d^p\boldsymbol{j}!}\psi_p^{\boldsymbol{j}}(\vY)\psi_p^{\boldsymbol{j}}(\vY)^\top$ is not extensive in the dimension of $\vG$, so this part can only generate spikes. Moreover the vectors $\psi_p^{\boldsymbol{j}}(\vY), \psi_q^{\boldsymbol{l}}(\vY)$ for $(p,\boldsymbol{j})\neq(q,\boldsymbol{l})$ are asymptotically orthogonal in the large-$n$ limit $\psi_p^{\boldsymbol{j}}(\vY)^\top\psi_q^{\boldsymbol{l}}(\vY)\sim nm \mathbb{E}_{\vy\sim P_t}[\psi_p^{\boldsymbol{j}}(\vy)\psi_q^{\boldsymbol{l}}(\vy)]=nm c_p^{\boldsymbol{j}}\delta_{p,q}\delta^{\boldsymbol{j},\boldsymbol{l}}$. Hence, this term will generate spikes in the spectrum with eigenvectors $\psi_p^{\boldsymbol{j}}(\vY)$ and eigenvalues $\mu_I^{>k}+\frac{c_p^{\boldsymbol{j}}\mu_pnm}{d^p\boldsymbol{j}!}+O(1)$. For each $p\in\{1,\dots, k-1\}$, there are $\Theta(d^p/p!)$ such eigenvectors. Because of the term $\mu_B^{>k}\vB_m$, the $2\times2$ reduction below applies to every degree $p\le k$: at each $p<k$ it also produces $\Theta(d^p/p!)$ eigenvalues $\mu_I^{>k}+\mu_B^{>k}(m-1)(1-c_{X,p}^{\boldsymbol{j}}/c_p^{\boldsymbol{j}})$, which are negligible in number compared with the $\Theta(d^k/k!)$ ones of degree $k$ and are removed from the atom counts at subleading order. The case $p=k$ is treated below.

\textbf{Case $p=k$.}
We now focus on
\begin{align}
    \vG=\mu^{>k}_I\vI_N+\mu_B^{>k}\vB_m+\sum_{\boldsymbol{j}}\frac{\mu_k}{d^k\boldsymbol{j}!}\psi_k^{\boldsymbol{j}}(\vY)\psi_k^{\boldsymbol{j}}(\vY)^\top
\end{align}
As in the linear case, we observe that in the asymptotic regime
\begin{align}
&\left[\sum_{\boldsymbol{j}}\frac{\mu_k}{d^k\boldsymbol{j}!}\psi_k^{\boldsymbol{j}}(\vY^{\nu\alpha})\psi_k^{\boldsymbol{j}}(\vY^{\nu'\alpha'})\right]\psi_k^{\boldsymbol{l}}(\vY^{\nu'\alpha'})\sim \frac{nm\mu_k}{d^k\boldsymbol{l}!}c_k^{\boldsymbol{l}}\psi_k^{\boldsymbol{l}}(\vY^{\nu\alpha})\\
&\left[\sum_{\boldsymbol{j}}\frac{\mu_k}{d^k\boldsymbol{j}!}\psi_k^{\boldsymbol{j}}(\vY^{\nu\alpha})\psi_k^{\boldsymbol{j}}(\vY^{\nu'\alpha'})\right]\frac{1}{m}\vB_m\psi_k^{\boldsymbol{l}}(\vY^{\nu'\alpha'})\sim \frac{nm\mu_k}{d^k\boldsymbol{l}!}\left(\frac{c_k^{\boldsymbol{l}}}{m}+\frac{m-1}{m}c_{X,k}^{\boldsymbol{l}}\right)\psi_k^{\boldsymbol{l}}(\vY^{\nu\alpha})
\end{align}
 where the sums over $(\nu',\alpha')$ are implicit, and we used that $c_{X,k}^{\boldsymbol{l}}$ is by definition the covariance of two noisy copies of the same clean sample. Up to the shift $\mu_I^{>k}$, the Gram matrix on the space spanned by $\{\frac{1}{m}\vB_m\psi_k^{\boldsymbol{j}}(\vY), (\vI_N-\frac{1}{m}\vB_m)\psi_k^{\boldsymbol{j}}(\vY)\}$ reads
\begin{align}
    &\begin{pmatrix}
        m\mu_B^{>k}+\frac{nm\mu_k}{d^k\boldsymbol{j}!}\left(\frac{c_k^{\boldsymbol{j}}}{m}+\frac{m-1}{m}c_{X,k}^{\boldsymbol{j}}\right)&\frac{nm\mu_k}{d^k\boldsymbol{j}!}\frac{m-1}{m}(c_k^{\boldsymbol{j}}-c_{X,k}^{\boldsymbol{j}})\\\frac{nm\mu_k}{d^k\boldsymbol{j}!}\left(\frac{c_k^{\boldsymbol{j}}}{m}+\frac{m-1}{m}c_{X,k}^{\boldsymbol{j}}\right)& \frac{nm\mu_k}{d^k\boldsymbol{j}!}\frac{m-1}{m}(c_k^{\boldsymbol{j}}-c_{X,k}^{\boldsymbol{j}})
    \end{pmatrix}\\
    &= \begin{pmatrix}
        m\mu_B^{>k}+\frac{nm}{d^k}\Lambda_\parallel^{\boldsymbol{j}}&\frac{nm}{d^k}\Lambda_\perp^{\boldsymbol{j}}\\\frac{nm}{d^k}\Lambda_\parallel^{\boldsymbol{j}}& \frac{nm}{d^k}\Lambda_\perp^{\boldsymbol{j}}
    \end{pmatrix}
\end{align}
Its trace is $m\mu_B^{>k}+\frac{nm}{d^k}(\Lambda_\parallel^{\boldsymbol{j}}+\Lambda_\perp^{\boldsymbol{j}})$ and its determinant $m\mu_B^{>k}\frac{nm}{d^k}\Lambda_\perp^{\boldsymbol{j}}$, with $\Lambda_\parallel^{\boldsymbol{j}}+\Lambda_\perp^{\boldsymbol{j}}=\mu_kc_k^{\boldsymbol{j}}/\boldsymbol{j}!$. For $nm/d^k\gg1$, diagonalizing it yields, at leading order, the eigenvector $\psi_k^{\boldsymbol{j}}(\vY)/\sqrt{nmc_k^{\boldsymbol{j}}}$ with eigenvalue
\begin{align}
  \mu_I^{>k}+\frac{nm\,\mu_kc_k^{\boldsymbol{j}}}{d^k\boldsymbol{j}!}+m\mu_B^{>k}\frac{\Lambda_\parallel^{\boldsymbol{j}}}{\Lambda_\parallel^{\boldsymbol{j}}+\Lambda_\perp^{\boldsymbol{j}}},
\end{align}
and the eigenvector $\Lambda_\perp^{\boldsymbol{j}} \frac{1}{m}\vB_m\psi_k^{\boldsymbol{j}} - \Lambda_\parallel^{\boldsymbol{j}} (\vI_N-\frac{1}{m}\vB_m) \psi_k^{\boldsymbol{j}}$ with eigenvalue
\begin{align}
  \mu_I^{>k}+m\mu_B^{>k} \frac{\Lambda_\perp^{\boldsymbol{j}}}{\Lambda_\perp^{\boldsymbol{j}} + \Lambda_\parallel^{\boldsymbol{j}}} = \mu_I^{>k}+\mu_B^{>k} (m-1) \left( 1 - \frac{c_{X,k}^{\boldsymbol{j}}}{c_k^{\boldsymbol{j}}} \right).
\end{align}
Each of these two bulks has $\Theta(d^k/k!)$ eigenvalues. As a check, at $k=1$ and for isotropic Gaussian data with $\sigma^2=1$, $\psi_1^{\boldsymbol{e}_i}(\vy)=y_i$, $c_1=\Gamma_t=1$ and $c_{X,1}=e^{-2t}$, so that the second eigenvalue is $\mu_I+\mu_B(m-1)\Delta_t$, the memorization bulk of Theorem~\ref{app:thm:Spectrum_gram_linear}.

\end{proof}
\subsection{Bias--variance decomposition}
\label{app:bv_decomposition}
We record here the decomposition used in Sect.~\ref{sect:biasvariance} of the main text and
in the two subsections below.

Write $\vy=e^{-t}\vx+\sqrt{\Delta_t}\,\vxi$ for a fresh test point, whose law is $P_t$, and
let $\vs_{\mathrm{exact}}(\vy)=\nabla_{\vy}\log P_t(\vy)$. Differentiating
$P_t(\vy)=\int\dd P_0(\vx)\,\mathcal{N}(\vy;e^{-t}\vx,\Delta_t\vI_d)$ under the integral sign
gives Tweedie's identity
\begin{align}
\label{eq:app_tweedie}
  \mathbb{E}\big[\vxi\,\big|\,\vy\big]=-\sqrt{\Delta_t}\,\vs_{\mathrm{exact}}(\vy).
\end{align}

Let $\vs_{\mathcal{D},\vtheta_0}$ be the learned score and $\langle\cdot\rangle$ the average
over the training set and the initialization, both independent of the test pair
$(\vx,\vxi)$. Adding and subtracting $\vs_{\mathrm{exact}}(\vy)$ in the test loss \eqref{eq:test_loss},
\begin{align}
  2\,\Ltest=\frac1d\,\mathbb{E}_{\vx,\vxi}\Big\langle\Big\lVert
  \big[\vs_{\mathcal{D},\vtheta_0}(\vy)-\vs_{\mathrm{exact}}(\vy)\big]
  +\Big[\vs_{\mathrm{exact}}(\vy)+\tfrac{\vxi}{\sqrt{\Delta_t}}\Big]
  \Big\rVert^2\Big\rangle .
\end{align}
The cross term vanishes: conditionally on $\vy$ the first bracket is a function of $\vy$ and
of $(\mathcal{D},\vtheta_0)$ alone, while the second has conditional mean zero by
\eqref{eq:app_tweedie}. Splitting what remains of the first bracket into its mean and its
fluctuation produces the two terms \eqref{eq:bias_variance} of the main text, and
\begin{align}
\label{eq:app_bv_decomposition}
  \Ltest=\frac12\big(C_t+\mathcal B^2+\mathcal V\big),
  \qquad
  C_t=\frac1d\,\mathbb{E}_{\vx,\vxi}\Big\lVert \vs_{\mathrm{exact}}(\vy)+\frac{\vxi}{\sqrt{\Delta_t}}\Big\rVert^2 .
\end{align}
The constant $C_t$ is the Bayes error of the denoising problem at noise level $t$ and it is the
value taken by $\Ltest$ at the exact score. It does not depend on the estimator. Using \eqref{eq:app_tweedie} once more,
$C_t=\frac{1}{d\Delta_t}\mathbb{E}\,\Tr\mathrm{Cov}(\vxi\mid\vy)$, that is
\begin{align}
\label{eq:app_Ct_fisher}
  C_t=\frac{1}{\Delta_t}-\frac1d\,\mathbb{E}_{\vy\sim P_t}\big\lVert\vs_{\mathrm{exact}}(\vy)\big\rVert^2.
\end{align}

Let us specialize to the case $P_0=\mathcal{N}(0,\vSigma)$. One has that
$P_t=\mathcal{N}(0,\vSigma_t)$ with $\vSigma_t=e^{-2t}\vSigma+\Delta_t\vI_d$, hence
$\vs_{\mathrm{exact}}(\vy)=-\vSigma_t^{-1}\vy$ and
$\mathbb{E}_{\vy\sim P_t}\lVert\vs_{\mathrm{exact}}(\vy)\rVert^2=\Tr(\vSigma_t^{-1})$.
Equation \eqref{eq:app_Ct_fisher} becomes
\begin{align}
\label{eq:app_Ct_gaussian}
  C_t=\frac1{\Delta_t}-\frac{\Tr(\vSigma_t^{-1})}{d}
     =\frac{e^{-2t}}{d}\sum_{\lambda\in\mathrm{Spec}(\vSigma)}\frac{\lambda}{\Delta_t\,\lambda_t},
  \qquad \lambda_t=e^{-2t}\lambda+\Delta_t .
\end{align}
For isotropic data, $\vSigma=\sigma^2\vI_d$, every $\lambda_t$ equals $\Gamma_t$ and
\begin{equation}
\label{eq:app_Ct_isotropic}
  C_t=\frac{\sigma^2e^{-2t}}{\Delta_t\,\Gamma_t}.
\end{equation}

\subsection{Derivation of Theorem~\ref{thm:bias_variance}: fixed-point equations in the kernel limit}
\label{app:bv_fixed_point}

We derive here the closed-form equations of Theorem~\ref{thm:bias_variance}, used in
Sect.~\ref{sect:biasvariance}. We set $\psi_n=n/d$, and the data are isotropic,
$\vSigma=\sigma^2\vI_d$. We first state the equations at $\sigma^2=1$, where
$\Gamma_t=e^{-2t}+\Delta_t=1$; the general case reduces exactly to this one (see the remark
after Theorem~\ref{app:thm:bias_variance}). We write $\gamma$ for the ridge and
$\mu_1=f'(0)$, $\mu_I=f(1)-f(e^{-2t})-\Delta_tf'(0)$,
$\mu_B=f(e^{-2t})-f(0)-e^{-2t}f'(0)$, $\bar\mu=\mu_I+m\mu_B$ for the kernel constants of
Theorem~\ref{thm:Spectrum_gram_linear}.

\begin{thm}[Bias and variance of the kernel ridge predictor]
\label{app:thm:bias_variance}
Under the assumptions above, let $w_1=me^{-2t}+\Delta_t$, $c_1=\bar\mu$, $w_2=\Delta_t$ and
$c_2=\mu_I$. The bias and the variance \eqref{eq:bias_variance} are
\begin{align}\label{eq:app_observables}
  \mathcal B^2=(b-1)^2,
  \qquad
  \mathcal V=\frac{\mu_1T_4}{\Delta_t}-b^2 ,
\end{align}
where $b$ and $T_4$ are given as follows.
\begin{enumerate}[leftmargin=1.4em,itemsep=2pt,topsep=2pt]
\item \emph{Ridge $\tilde\gamma>0$.} Let $\bar q=1/\tilde\gamma$ and
\begin{align}\label{eq:app_fp_ridge}
  L_a(\bar r)=1+\frac{\mu_1w_a\bar r+c_a\bar q}{\psi_n m},
  \quad
  g(\bar r)=\frac1{L_1}+\frac{m-1}{L_2},
  \quad
  R(\bar r)=\frac{\mu_1}{m}\left(\frac{w_1}{L_1}+\frac{(m-1)w_2}{L_2}\right).
\end{align}
Then
\begin{align}\label{eq:app_fp_obs}
  b=\frac{\mu_1\,\bar r\,g(\bar r)}{m},
  \qquad
  T_4=\frac{-\Phi'(\bar r)}{\bar r^{-2}+R'(\bar r)},
  \qquad
  \Phi(\bar r)=\frac{g}{m}-\frac{\mu_1\Delta_t}{m^2}\,\bar r\,g^2 ,
\end{align}
where $\bar r$ solves $1/\bar r=\tilde\gamma+R(\bar r)$.
\item \emph{Ridgeless limit $\tilde\gamma\to0^+$.} For any $\varrho>0$, let $q=1/\varrho$ and
\begin{align}\label{eq:app_fp_ridgeless}
  e_a(r)=\mu_1 r\,w_a+q\,c_a,
  \quad
  g(r)=\frac1{e_1}+\frac{m-1}{e_2},
  \quad
  R(r)=\mu_1\psi_n\left(\frac{w_1}{e_1}+\frac{(m-1)w_2}{e_2}\right).
\end{align}
Then \eqref{eq:app_fp_obs} holds with $\bar r$ replaced by the solution $r$ of
$1/r=\varrho+R(r)$, $b=\mu_1\psi_n\,r\,g(r)$ and $\Phi(r)=\psi_n g-\mu_1\Delta_t\psi_n^2\,r\,g^2$.
The result does not depend on $\varrho$.
\end{enumerate}
\end{thm}
Case 1 is Theorem~\ref{thm:bias_variance} of the main text. The index $a=1$ refers to the
average of the $m$ noisy copies of a clean sample and $a=2$ to the $m-1$ directions
orthogonal to it; $c_1$ and $c_2$ are the two atoms of
Theorem~\ref{app:thm:Spectrum_gram_linear}. For $m=1$ only the terms with $a=1$ remain.

\paragraph*{Reduction of $\vSigma=\sigma^2\vI_d$ to $\sigma^2=1$.} Let
$\Gamma_t=\sigma^2e^{-2t}+\Delta_t$ and define $t'$ by $e^{-2t'}=\sigma^2e^{-2t}/\Gamma_t$, so
that $\Delta_{t'}=1-e^{-2t'}=\Delta_t/\Gamma_t$. With $\vx'=\vx/\sigma$, which has covariance
$\vI_d$,
\begin{align}
  \vY^{\nu\alpha}=e^{-t}\vx^\nu+\sqrt{\Delta_t}\,\vxi^{\nu\alpha}
  =\sqrt{\Gamma_t}\,\big(e^{-t'}\vx'^\nu+\sqrt{\Delta_{t'}}\,\vxi^{\nu\alpha}\big)
  =\sqrt{\Gamma_t}\,\vY'^{\nu\alpha},
\end{align}
where $\vY'$ is the noised data of the $\sigma^2=1$ problem at time $t'$, with the
\emph{same} noises $\vxi$. Hence:
\begin{itemize}[leftmargin=1.2em]
  \item the Gram matrix is unchanged, $f(\vY^\top\vY/d)=\tilde f(\vY'^\top\vY'/d)$ with
  $\tilde f(u)=f(\Gamma_tu)$, and so are the ridge $\gamma$ and the training time $\tau$;
  \item the target $-\vxi/\sqrt{\Delta_t}=\Gamma_t^{-1/2}\big(-\vxi/\sqrt{\Delta_{t'}}\big)$
  and the exact score $\vs_{\mathrm{exact}}(\vy)=-\vy/\Gamma_t=\Gamma_t^{-1/2}\,\vs'_{\mathrm{exact}}(\vy')$
  are both rescaled by $\Gamma_t^{-1/2}$, so that the predictor, its bias and its variance
  are those of the $\sigma^2=1$ problem, with $\mathcal B^2$ and $\mathcal V$ multiplied by
  $1/\Gamma_t$.
\end{itemize}
Therefore, independently of the Gaussian-equivalence argument (which also extends to
general $\sigma^2$ with modified coefficients, see the end of the proof of
Proposition~\ref{app:prop:gep}), Theorem~\ref{app:thm:bias_variance} holds for
$\vSigma=\sigma^2\vI_d$ with $t$ replaced by $t'$, $f$ by $\tilde f$, and $\mathcal B^2,\mathcal V$ divided by $\Gamma_t$. In the
original variables the kernel constants of $\tilde f$ are exactly those of
Theorem~\ref{app:thm:lin_gram_empirical} at general $\sigma^2$,
\begin{gather}
  \tilde\mu_I=f(\Gamma_t)-f(\sigma^2e^{-2t})-\Delta_tf'(0)=\mu_I,\nonumber\\
  \tilde\mu_B=f(\sigma^2e^{-2t})-f(0)-\sigma^2e^{-2t}f'(0)=\mu_B,\qquad
  \tilde\mu_1=\Gamma_tf'(0),
\end{gather}
and the channel weights are $\tilde w_1=(m\sigma^2e^{-2t}+\Delta_t)/\Gamma_t$ and
$\tilde w_2=\Delta_t/\Gamma_t$. As a check, the same map sends the irreducible term
$C'_{t'}=e^{-2t'}/\Delta_{t'}$ to $C'_{t'}/\Gamma_t=\sigma^2e^{-2t}/(\Delta_t\Gamma_t)$, which is
\eqref{eq:app_Ct_isotropic}.

\begin{proof}
We take the data isotropic and Gaussian, $\vx^\nu\sim\mathcal N(0,\vI_d)$, and treat the
case $\sigma^2=1$; general $\sigma^2$ follows from the reduction below. The proof has four
steps: we reduce the predictor to a linear map, use isotropy to fix its mean, and compute
its first two moments with deterministic equivalents for the resolvent.

\paragraph*{Step 1: linear structure.} By Theorem~\ref{app:thm:lin_gram_empirical} (or,
equivalently, the Gaussian equivalence of Proposition~\ref{app:prop:gep}), the Gram
matrix can be replaced by its linear equivalent, and the test kernel by
$K(\vy,\vY)\simeq\mu_1\vY^\top\vy/d$ for a test point $\vy\sim P_t$ independent of the
training set. The rank-one terms $\mu_0\bm1_N\bm1_N^\top$ and $\mu_0\bm1_N$ change
the normalized traces below by $O(1/d)$ and are dropped. With
$\vE=(\mu_I+\gamma)\vI_N+\mu_B\vB_m$ and $\vR=(\vG+\gamma\vI_N)^{-1}$,
\begin{gather}
  \vG+\gamma\vI_N=\vE+\frac{\mu_1}{d}\vY^\top\vY,\nonumber\\
  \vs(\vy)=K(\vy,\vY)^\top\vR\Big(-\frac{\vxi}{\sqrt{\Delta_t}}\Big)^{\!\top}\!=-\vA\vy,
  \qquad
  \vA=\frac{\mu_1}{d\sqrt{\Delta_t}}\,\vxi\,\vR\,\vY^\top\in\mathbb R^{d\times d},
\end{gather}
where $\vY,\vxi\in\mathbb R^{d\times N}$ stack the noised training points and their noises
as columns. The predictor is thus linear in the test point. The $d$ rows $\vy_i\in\mathbb R^N$ of $\vY$ are
i.i.d.\ $\mathcal N(0,\vC)$ with
\begin{align}
  \vC=e^{-2t}\vB_m+\Delta_t\vI_N ,
\end{align}
since $\mathbb E[\vY_i^{\nu\alpha}\vY_i^{\mu\beta}]=\delta_{\nu\mu}(e^{-2t}+\Delta_t\delta_{\alpha\beta})$.
The matrices $\vC$ and $\vE$ are diagonal in the same orthonormal basis: on the $n$
block-average directions $\bm1_m/\sqrt m$ they take the values $w_1=me^{-2t}+\Delta_t$ and
$c_1+\gamma=\bar\mu+\gamma$, and on the $n(m-1)$ directions transverse to them the values
$w_2=\Delta_t$ and $c_2+\gamma=\mu_I+\gamma$. This is the decoupling into the two channels
of the statement.

\paragraph*{Step 2: isotropy.} The joint law of $(\vY,\vxi)$ is invariant under
$(\vY,\vxi)\mapsto(\vO\vY,\vO\vxi)$ for any orthogonal $\vO\in O(d)$, under which
$\vR$ is unchanged and $\vA\mapsto\vO\vA\vO^\top$. Hence $\langle\vA\rangle$ commutes
with every rotation, and
\begin{align}
  \langle\vA\rangle=b\,\vI_d,\qquad b=\frac1d\big\langle\Tr\vA\big\rangle .
\end{align}
This is the isotropic form of the mean predictor, obtained here without an ansatz. Since
$\vs_{\mathrm{exact}}(\vy)=-\vy$ and $\mathbb E[\vy\vy^\top]=\vI_d$, the definitions
\eqref{eq:bias_variance} give
\begin{align}
  \mathcal B^2=(b-1)^2,\qquad
  \mathcal V=\frac1d\big\langle\Tr\vA\vA^\top\big\rangle-b^2
  =\frac{\mu_1T_4}{\Delta_t}-b^2,\qquad
  T_4=\frac{\mu_1}{d^3}\big\langle\Tr\big(\vxi\vR\vY^\top\vY\vR\vxi^\top\big)\big\rangle,
\end{align}
which is \eqref{eq:app_observables}, with $b$ and $T_4$ the first two moments of $\vA$. It
remains to compute them.

\paragraph*{Step 3: the resolvent.} Since $\vG+\gamma\vI_N$ is a deterministic matrix plus a
sample covariance of $d$ i.i.d.\ Gaussian vectors in $\mathbb R^N$, with $N/d=m\psi_n$ fixed,
its resolvent has the deterministic equivalent
\citep{SILVERSTEIN1995,louart2018}
\begin{align}
\label{eq:app_det_equiv}
  \vR\simeq\bar\vR=\Big(\vE+\frac{\mu_1\vC}{1+\delta}\Big)^{-1},
  \qquad
  \delta=\frac{\mu_1}{d}\Tr\big(\vC\bar\vR\big),
\end{align}
in the sense that $\frac1d\Tr(\vM\vR)-\frac1d\Tr(\vM\bar\vR)\to0$ for deterministic $\vM$ of
bounded norm; we also use the leave-one-out identity
$\vR\vy_i=\vR_{-i}\vy_i/(1+\frac{\mu_1}{d}\vy_i^\top\vR_{-i}\vy_i)$, with $\vR_{-i}$ the
resolvent without row $i$ and $\frac{\mu_1}{d}\vy_i^\top\vR_{-i}\vy_i\to\delta$. In the joint
eigenbasis of $\vC$ and $\vE$, $\bar\vR$ is diagonal with entries
$(c_a+\gamma+\mu_1w_a/(1+\delta))^{-1}$. Writing $\gamma=m\psi_n\tilde\gamma=m\psi_n/\bar q$ and
setting
\begin{align}
  \bar r=\frac{\bar q}{1+\delta},
\end{align}
these entries are $\bar q/(m\psi_nL_a(\bar r))$ with $L_a$ as in \eqref{eq:app_fp_ridge}.
Hence $\frac1d\Tr\bar\vR=\bar qg(\bar r)/m$, and the equation for $\delta$ becomes
$\delta=\mu_1\psi_n\big(w_1\bar R_1+(m-1)w_2\bar R_2\big)=\bar q\,R(\bar r)$, i.e.\
$1/\bar r=(1+\delta)/\bar q=\tilde\gamma+R(\bar r)$: this is the fixed point of the statement.

\paragraph*{Step 4: the two moments.} Row by row, $(\vy_i,\vxi_i)$ is a centered Gaussian
pair with $\mathbb E[\vxi_i\vy_i^\top]=\sqrt{\Delta_t}\,\vI_N$, so that
\begin{align}
  \vxi=\sqrt{\Delta_t}\,\vY\vC^{-1}+\vH ,
\end{align}
with the rows of $\vH$ i.i.d.\ $\mathcal N(0,\vI_N-\Delta_t\vC^{-1})$ and independent of
$\vY$, hence of $\vR$.
\begin{itemize}[leftmargin=1.2em]
  \item \emph{Bias.} The $\vH$ part of $\Tr(\vxi\vR\vY^\top)=\sum_i\vxi_i^\top\vR\vy_i$ has
  zero mean, and by the leave-one-out identity
  $\vy_i^\top\vC^{-1}\vR\vy_i\simeq\Tr(\vC^{-1}\vC\bar\vR)/(1+\delta)=\Tr\bar\vR/(1+\delta)$. Thus
  \begin{align}
    b=\frac{\mu_1}{d^2\sqrt{\Delta_t}}\,d\sqrt{\Delta_t}\,\frac{\Tr\bar\vR}{1+\delta}
    =\frac{\mu_1\bar q\,g}{m(1+\delta)}=\frac{\mu_1\bar r\,g(\bar r)}{m}.
  \end{align}
  \item \emph{Second moment.} Let $\vR_\varepsilon=\big(\vE+(1-\varepsilon)\frac{\mu_1}{d}\vY^\top\vY\big)^{-1}$.
  Since $\partial_\varepsilon\vR_\varepsilon|_0=\vR\,\frac{\mu_1}{d}\vY^\top\vY\,\vR$,
  \begin{align}
    T_4=\partial_\varepsilon S(\varepsilon)\big|_{\varepsilon=0},\qquad
    S(\varepsilon)=\frac1{d^2}\big\langle\Tr\big(\vxi\vR_\varepsilon\vxi^\top\big)\big\rangle ,
  \end{align}
  which only involves first-order traces. Using the decomposition of $\vxi$ and, for the
  $\vY$ part, Sherman--Morrison,
  $\vy_i^\top\vC^{-1}\vR_\varepsilon\vC^{-1}\vy_i\simeq\Tr(\vC^{-1}\bar\vR_\varepsilon)
  -\frac{\mu_\varepsilon}{d}(\Tr\bar\vR_\varepsilon)^2/(1+\delta_\varepsilon)$ with
  $\mu_\varepsilon=(1-\varepsilon)\mu_1$, the terms in $\Tr(\vC^{-1}\bar\vR_\varepsilon)$ cancel
  between the two parts and
  \begin{align}
    S(\varepsilon)=\frac1d\Tr\bar\vR_\varepsilon-\frac{\Delta_t\mu_\varepsilon}{d^2}\,
    \frac{(\Tr\bar\vR_\varepsilon)^2}{1+\delta_\varepsilon}
    =\bar q\left(\frac{g}{m}-\frac{\mu_\varepsilon\Delta_t}{m^2}\,\bar r_\varepsilon\,g^2\right),
  \end{align}
  where $\bar r_\varepsilon$ and $g$ are those of Step 3 with $\mu_1$ replaced by
  $\mu_\varepsilon$. All these quantities depend on $\mu_\varepsilon$ and $\bar r_\varepsilon$ only through
  $\rho=\mu_\varepsilon\bar r_\varepsilon$, except for the fixed point, which reads
  $1/\rho=1/(\mu_\varepsilon\bar q)+h(\rho)/m$ with
  $h=w_1/L_1+(m-1)w_2/L_2$. Differentiating it at $\varepsilon=0$ gives
  $\partial_\varepsilon\rho=-\big[\mu_1\bar q\,(\rho^{-2}+h'(\rho)/m)\big]^{-1}$, and
  $S=\bar q\,\tilde\Phi(\rho)$ with $\tilde\Phi(\rho)=g/m-\Delta_t\rho g^2/m^2$, so that
  \begin{align}
    T_4=\bar q\,\tilde\Phi'(\rho)\,\partial_\varepsilon\rho
    =-\frac{\tilde\Phi'(\rho)}{\mu_1\big(\rho^{-2}+h'(\rho)/m\big)}
    =-\frac{\Phi'(\bar r)}{\bar r^{-2}+R'(\bar r)},
  \end{align}
  the last form following from $\Phi(\bar r)=\tilde\Phi(\mu_1\bar r)$ and
  $R(\bar r)=\frac{\mu_1}{m}h(\mu_1\bar r)$. This is \eqref{eq:app_fp_obs}.
\end{itemize}

\paragraph*{Test loss.} Combining with \eqref{eq:app_bv_decomposition} and
$C_t=e^{-2t}/\Delta_t$,
\begin{align}\label{eq:app_Ltest_middle}
  \Ltest=\frac12\big(C_t+\mathcal B^2+\mathcal V\big)
  =\frac{1-2\Delta_t\,b+\mu_1T_4}{2\Delta_t}.
\end{align}

\paragraph*{Ridgeless limit.} As $\tilde\gamma\to0^+$, $\bar q\to\infty$ and
$\bar r=\bar q\,r$ with $r$ finite. Then $L_a\simeq\bar q\,e_a/(\psi_nm)$ with
$e_a=\mu_1rw_a+c_a$, so that $g=(\psi_nm/\bar q)\,g_0$, $R=R_0/\bar q$ and
$\Phi=\Phi_0/\bar q$, where $g_0,R_0,\Phi_0$ are the functions of
\eqref{eq:app_fp_ridgeless} at $q=1$. The fixed point becomes $1/r=1+R_0(r)$, and
$b=\mu_1\psi_nrg_0$ and $T_4=-\Phi_0'/(r^{-2}+R_0')$, which is item 2 of the statement with
$\varrho=q=1$. The equations are invariant under $r\mapsto r/\varrho$, $q=1/\varrho$, which
is why the result does not depend on the gauge $\varrho$.

\paragraph*{Derivatives.}
The derivatives entering $T_4$ follow from the chain rule. In the
ridge parametrization \eqref{eq:app_fp_ridge}, with $\partial L_a=\mu_1w_a/(\psi_nm)$,
\begin{align}
  g'&=-\frac{\partial L_1}{L_1^2}-\frac{(m-1)\partial L_2}{L_2^2},
  \quad
  R'=-\frac{\mu_1}{m}\left[\frac{w_1\partial L_1}{L_1^2}+\frac{(m-1)w_2\partial L_2}{L_2^2}\right],\\
  &\Phi'=\frac{g'}{m}-\frac{\mu_1\Delta_t}{m^2}\big(g^2+2\bar rgg'\big),
\end{align}
and in the ridgeless parametrization \eqref{eq:app_fp_ridgeless},
\begin{align}
  g'&=-\frac{\mu_1w_1}{e_1^2}-\frac{(m-1)\mu_1w_2}{e_2^2},
  \quad
  R'=-\mu_1^2\psi_n\left[\frac{w_1^2}{e_1^2}+\frac{(m-1)w_2^2}{e_2^2}\right],\\
  &\Phi'=\psi_ng'-\mu_1\Delta_t\psi_n^2\big(g^2+2rgg'\big).
\end{align}

\end{proof}

\subsection{Numerical check of Theorem~\ref{thm:bias_variance}}
\label{app:subsec:krr_bv}

This subsection documents the finite-dimensional experiments performed for kernel ridge regression.

\begin{figure}[t]
    \centering
    \includegraphics[width=0.49\linewidth]{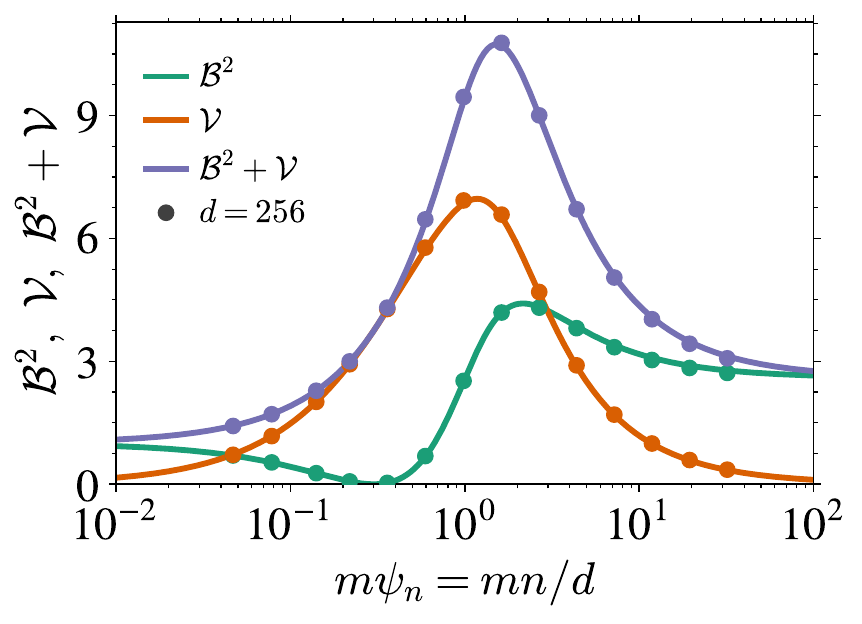}\hfill
    \includegraphics[width=0.49\linewidth]{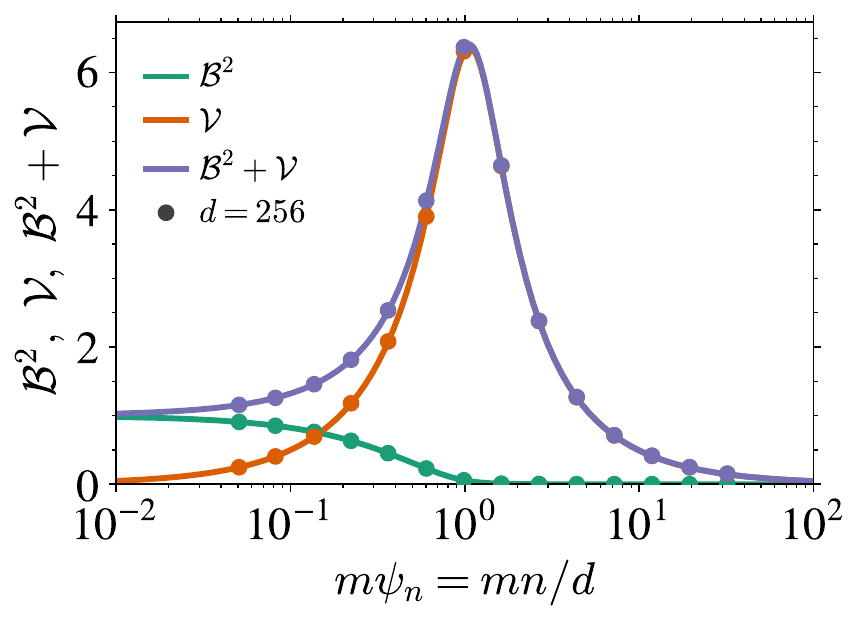}
    \caption{\textbf{Bias and variance against the sample complexity.} $\mathcal B^2$,
    $\mathcal V$ and $\mathcal B^2+\mathcal V$ against $m\psi_n=mn/d$ for the ridgeless
    estimator $\tilde\gamma\to0^+$, at $t=0.1$, $\sigma^2=1$, with the kernel of
    Figure~\ref{fig:bias_variance_m_1_m_4}, for $m=4$ (\textit{left}) and $m=1$
    (\textit{right}). Solid lines are the equations of Theorem~\ref{thm:bias_variance};
    markers are experiments at $d=256$.}
    \label{fig:bv_budget}
\end{figure}

\paragraph*{Setup.} Data are isotropic, $\vx^\nu\sim\mathcal N(0,\vI_d)$ and
$\vY^{\nu\alpha}=e^{-t}\vx^\nu+\sqrt{\Delta_t}\vxi^{\nu\alpha}$ with $m$ noise realizations per
data point, so that $\vSigma_t=\vI_d$ and the exact score is $\vs_{\mathrm{exact}}(\vy)=-\vy$. The
kernel is the dual activation $f(u)=\mathbb E[\sigma(z_1)\sigma(z_2)]$ of
$\sigma=\tanh$ at correlation $u$. Since $f$ is a kernel only for $|u|\le1$, and at finite $d$ the norms
$\|\vY^{\nu\alpha}\|^2/d$ fluctuate around their limit $\Gamma_t=1$, the rows are
rescaled to the sphere of radius $\sqrt d$ in the kernel argument only.

\paragraph*{Bias and variance.} With $n_D$ independent training sets and $n_z$ fresh test
points, writing $\bar{\vs}(\vy)$ for the average of the $n_D$ predictors,
\begin{align}\label{eq:app_bv_estimators}
  \widehat{\mathcal V}(\vy)&=\frac{1}{d\,(n_D-1)}\sum_{i=1}^{n_D}
    \big\|\vs_i(\vy)-\bar{\vs}(\vy)\big\|^2,
  &
  \widehat{\mathcal B^2}(\vy)&=\frac{1}{d}\big\|\vs_{\mathrm{exact}}(\vy)-\bar{\vs}(\vy)\big\|^2
    -\frac{\widehat{\mathcal V}(\vy)}{n_D},
\end{align}
both averaged over the $n_z$ test points. The $1/(n_D-1)$ and the $-\widehat{\mathcal
V}/n_D$ subtraction make the two estimators unbiased at any $n_D$, so that neither the
variance nor the bias is inflated by the finite number of realizations.

\paragraph*{Parameters.} The sweeps at fixed $\psi_n$ of
Figure~\ref{fig:bias_variance_m_1_m_4} (left panel and inset) use
$\psi_n=8$, $n_z=512$ test points and $d\in\{64,128,256\}$, with $n_D=16$ training sets for
$m=1$ and $n_D=16,12,8$ for $m=4$ at $d=64,128,256$ respectively, while the sweep over the budget $m\psi_n$ of
Figure~\ref{fig:bv_budget} uses $n_z=1024$ and $n_D$ between $32$
at the smallest $N$ and $8$ at the largest.

\paragraph*{Error bars.} The error bars drawn on every finite-$d$ point are a delete-one
jackknife over the $n_D$ realizations of the dataset.

\paragraph*{Early stopping versus ridging.} Fixing a training time $\tau$ in
\eqref{eq:estimation_score_DSM} imposes a spectral cutoff $\lambda_c\sim nm/(d\tau)$,
whereas a ridge cuts at $\lambda\sim\gamma$. Equating the two cutoffs gives
$\gamma_{\mathrm{KRR}}=m\psi_n/\tau$, that is $\tilde\gamma=1/\tau$ on the scale of
Theorem~\ref{thm:bias_variance} --- the same calibration under which
\citet{ali2019_earlystopping} bound the risk of gradient flow by that of
ridge regression along the whole path, for least squares with no assumption on the
design. Figure~\ref{fig:tau_gamma} tests that identification by
drawing the analytical $\mathcal L_{\mathrm{test}}$, $\mathcal B^2$ and $\mathcal V$ at
$\tilde\gamma=1/\tau$ against the measured gradient flow at $\tau$. The ridge looks therefore like a faithful proxy for early stopping for the values of the
observables and for the scaling of the timescales.

\begin{figure}[t]
    \centering
    \includegraphics[width=\linewidth]{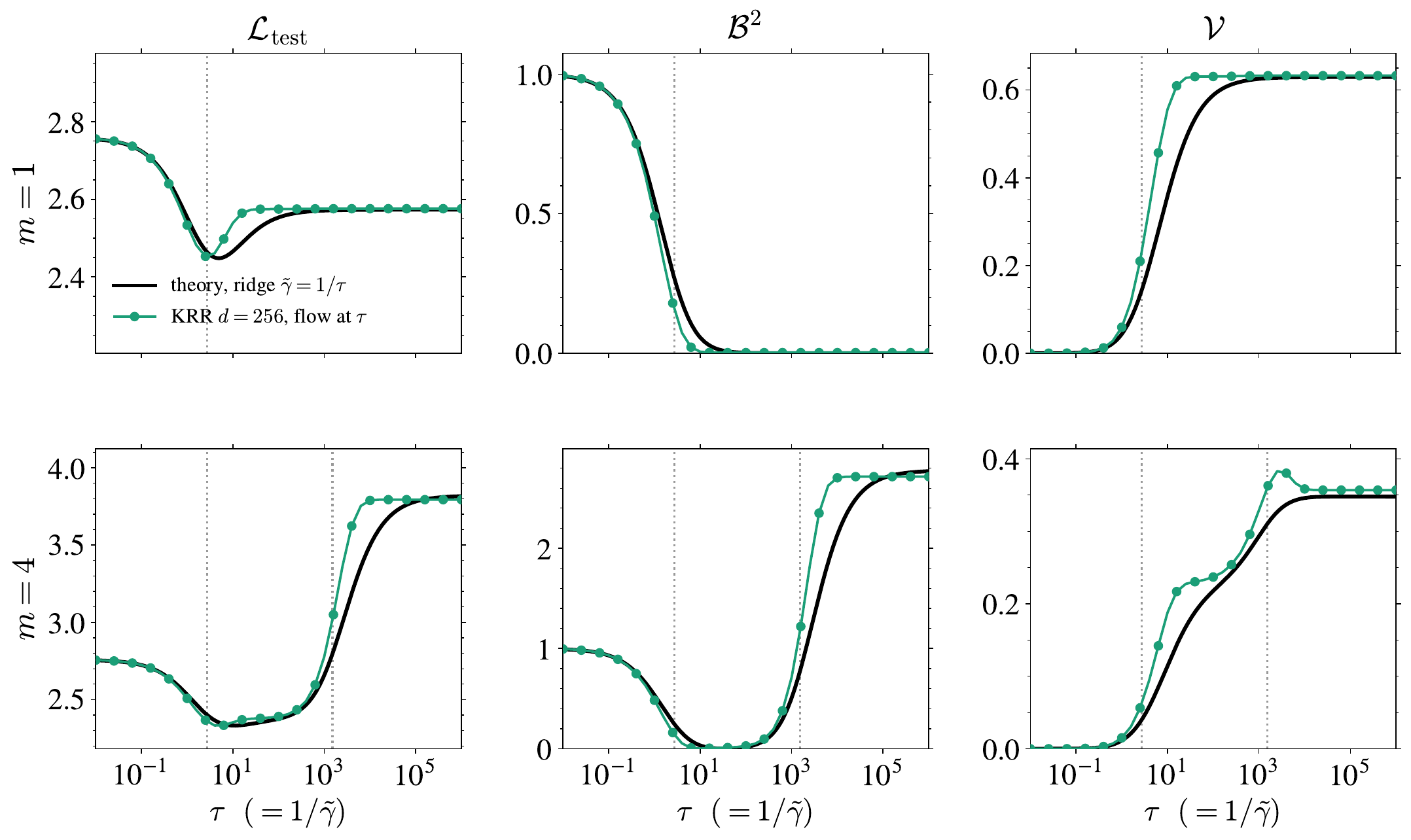}
    \caption{\textbf{The ridge as a proxy for early stopping.}
    $\mathcal L_{\mathrm{test}}$, $\mathcal B^2$ and $\mathcal V$ against the training
    time $\tau$, at $m=1$ (\textit{top}) and $m=4$ (\textit{bottom}). Black: the
    theory of Theorem~\ref{thm:bias_variance} evaluated at
    $\tilde\gamma=1/\tau$. Green: kernel gradient flow of \eqref{eq:estimation_score_DSM}
    measured at finite $d=256$, $\psi_n=8$, averaged over $16$ ($m=1$) and $8$ ($m=4$)
    realizations of the dataset and $512$ test points. Dotted verticals are
    $\tau_{\mathrm{gen}}$ and $\tau_{\mathrm{mem}}$. Panels are scaled independently: at $m=1$ the bias vanishes whereas at
    $m=4$ it saturates at $\Theta(1)$.}
    \label{fig:tau_gamma}
\end{figure}

\subsection{Small-$t$ behavior of the bias and the variance}
\label{app:bv_smallt}
Theorem~\ref{thm:bias_variance} holds at a fixed noise level, and most figures are drawn
at $t=0.1$. Since memorization is a small-$t$ phenomenon, we record here how
$\mathcal B^2$ and $\mathcal V$ behave as $t\to0$ (Fig.~\ref{fig:bv_tscaling}).

\begin{figure}[t]
  \centering
  \includegraphics[width=0.5\linewidth]{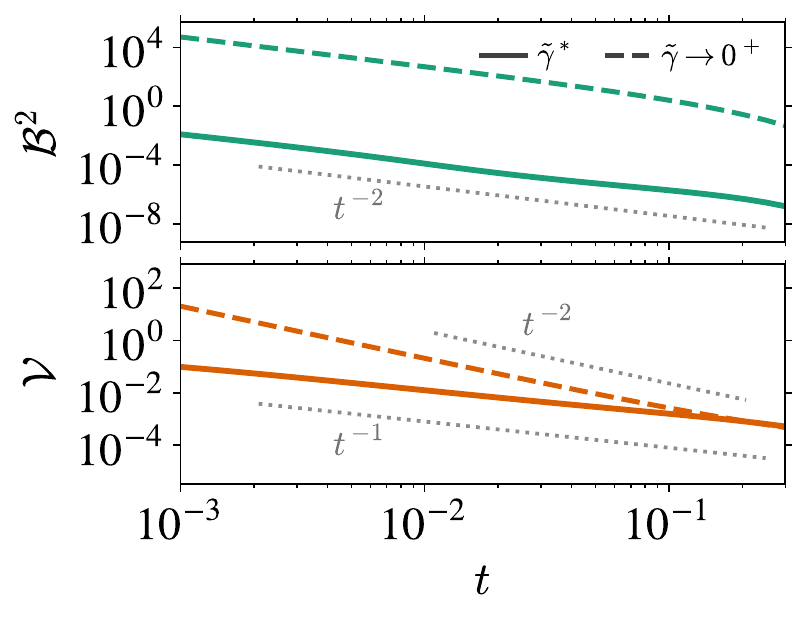}
  \caption{\textbf{Scalings with $t$.} $\mathcal B^2$ (\emph{top}) and $\mathcal V$
  (\emph{bottom}) against $t$ at $\psi_n=10^3$ and $m=4$, where $1/(m\psi_n)\ll t\ll1$, for
  the optimally ridged estimator $\tilde\gamma^*$ (solid) and the ridgeless one
  $\tilde\gamma\to0^+$ (dashed), with the kernel of
  Fig.~\ref{fig:bias_variance_m_1_m_4}. Lines are the equations of
  Theorem~\ref{thm:bias_variance}.}
  \label{fig:bv_tscaling}
\end{figure}

\paragraph*{What collapses at small $t$.} Expanding the constants of
Theorem~\ref{app:thm:lin_gram_empirical} as $\Delta_t\to0$,
\begin{align}
\label{eq:app_smallt_constants}
  \mu_I=\Delta_t\,\iota+O(\Delta_t^2),\quad \iota=f'(\sigma^2)-f'(0),
  \qquad
  \mu_B\to\mu_B^0=f(\sigma^2)-f(0)-\sigma^2f'(0),
\end{align}
so that, at $\sigma^2=1$ where $\lambda_t=1$, the memorization eigenvalue of
Theorem~\ref{app:thm:Spectrum_gram_linear} becomes
\begin{align}
\label{eq:app_smallt_lambda2}
  \lambda_{\mathrm{mem}}=\mu_I+\frac{\mu_B(m-1)\Delta_t}{\lambda_t}
  \;\simeq\;\Delta_t\big(\iota+(m-1)\mu_B^0\big) .
\end{align}
It vanishes linearly in $\Delta_t$, while the generalization eigenvalue
$\lambda_{\mathrm{gen}}=\mu_1\psi_nm\lambda_t$ stays $\Theta(\psi_n m)$: the two bulks separate further
and further as the noise level drops. The memorization modes become both arbitrarily slow to
learn and arbitrarily ill-conditioned for the ridgeless predictor, and the scalings below are
the quantitative form of that second statement.

\begin{result}[Small-$t$ scalings of the bias and the variance]
\label{res:bv_smallt}
Let $\sigma^2=1$ and $\Delta_t=1-e^{-2t}\simeq2t$, and write $s=m\psi_n\Delta_t$.
\begin{enumerate}[leftmargin=1.4em,itemsep=2pt,topsep=2pt]
  \item \emph{Ridgeless, $\Delta_t\to0$ at fixed $\psi_n$ and $m$.}
  \begin{itemize}[leftmargin=1.2em]
    \item $m=1$: $\mathcal B^2$ is independent of $t$, and
          $\psi_n\Delta_t\mathcal V\to1$ as $\psi_n\to\infty$.
    \item $m>1$: $\mathcal B^2\simeq R_{\psi_n}^2/\Delta_t^2$ and
          $\mathcal V\simeq c_{m,\psi_n}/(\psi_n\Delta_t^2)$, where the prefactors depend on
          $\psi_n$ only through $O(1/\psi_n)$ corrections, with limits
          \begin{align}
          \label{eq:app_smallt_R}
            \begin{aligned}
            R&=\lim_{\psi_n\to\infty}R_{\psi_n}=\frac{(m-1)\mu_B^0}{\iota+(m-1)\mu_B^0}\in(0,1),\\
            c_m&=\lim_{\psi_n\to\infty}c_{m,\psi_n}=\frac{m(m-1)\,\iota^2(\mu_B^0)^2}{\big(\iota+(m-1)\mu_B^0\big)^4}.
            \end{aligned}
          \end{align}
  \end{itemize}
  \item \emph{Optimal ridge, $\Delta_t\to0$ and $\psi_n\to\infty$ at fixed $s$, for every $m$:}
        $\tilde\gamma^*=\mu_1/s$ and
        \begin{align}
        \label{eq:app_smallt_collapse}
          \mathcal B^2=\frac1{(1+s)^2},\qquad
          \mathcal V=\frac{s}{(1+s)^2},\qquad
          \mathcal B^2+\mathcal V=\frac1{1+s} .
        \end{align}
  \item \emph{The window of Result~\ref{res:bv_scalings}.} For $m>1$ and
        $1/(m\psi_n)\ll t\ll1$, i.e.\ $s\gg1$ with $\Delta_t\ll1$, items 1 and 2 give
        $\mathcal B^2\simeq R^2/\Delta_t^2=\Theta(t^{-2})$ and
        $\mathcal V\simeq c_m/(\psi_n\Delta_t^2)=\Theta(\psi_n^{-1}t^{-2})$ without ridge, and
        $\mathcal B^2\simeq s^{-2}=\Theta(\psi_n^{-2}t^{-2})$ and
        $\mathcal V\simeq s^{-1}=\Theta(\psi_n^{-1}t^{-1})$ at the optimal ridge. In the
        opposite regime $s\ll1$ the optimally ridged estimator returns $\mathcal B^2\to1$ and
        $\mathcal V\simeq s$: it predicts almost nothing.
\end{enumerate}
\end{result}

\paragraph*{The ridgeless prefactors.} Take the ridgeless parametrization
\eqref{eq:app_fp_ridgeless} with $\varrho=1$ (the result does not depend on this choice) and
let $\Delta_t\to0$ with \eqref{eq:app_smallt_constants}. Then $w_1\to m$, $c_1\to m\mu_B^0$,
$w_2=\Delta_t$ and $c_2=\Delta_t\iota$, so that $e_1\to m(\mu_1r+\mu_B^0)$ and
$e_2=\Delta_t(\mu_1r+\iota)$. The order parameter equation becomes independent of $t$,
\begin{align}
\label{eq:app_smallt_r}
  \frac1r=1+\mu_1\psi_n\left(\frac1{\mu_1r+\mu_B^0}+\frac{m-1}{\mu_1r+\iota}\right),
\end{align}
whereas $g=G/\Delta_t+O(1)$ with $G=(m-1)/(\mu_1r+\iota)$ is dominated by the transverse
directions. Hence $b=\mu_1\psi_nrg$ diverges as $R_{\psi_n}/\Delta_t$ with
$R_{\psi_n}=\mu_1\psi_n rG$. As $\psi_n\to\infty$ the solution of \eqref{eq:app_smallt_r} is
$r\simeq\iota\mu_B^0/[\mu_1\psi_n(\iota+(m-1)\mu_B^0)]$, which gives the value of $R$ in
\eqref{eq:app_smallt_R}; $\mathcal B^2=(b-1)^2\simeq R_{\psi_n}^2/\Delta_t^2$. For the variance,
$\Phi=\psi_n g-\mu_1\Delta_t\psi_n^2rg^2$ is also $\Theta(1/\Delta_t)$, so that both terms of
$\mathcal V=\mu_1T_4/\Delta_t-b^2$ are $\Theta(\Delta_t^{-2})$. Their leading parts in
$1/\psi_n$ are both equal to $G^2/(\mu_B^{0\,-1}+(m-1)\iota^{-1})^2$ and cancel, and the next order,
obtained by expanding \eqref{eq:app_smallt_r} to $O(\psi_n^{-2})$, gives
$\psi_n\Delta_t^2\mathcal V\to c_m$. At $m=1$ there are no transverse directions:
$w_1=1$ and $c_1=\mu_I+\mu_B=f(1)-f(0)-f'(0)$ are exactly independent of $t$, hence so is
$\mathcal B^2$, and only the $1/\Delta_t$ in front of $T_4$ survives, which gives
$\psi_n\Delta_t\mathcal V\to1$.

\paragraph*{Why the ridged branch is a function of $s$ alone.} The last item of
Result~\ref{res:bv_smallt} is not only read off the numerics: it follows from
Theorem~\ref{app:thm:bias_variance} in four steps, in the limit $\Delta_t\to0$ at fixed
$s=m\psi_n\Delta_t$, hence $\psi_n m=s/\Delta_t\to\infty$.

First, the two channel denominators saturate. With $w_1=me^{-2t}+\Delta_t\to m$,
$c_1=\bar\mu\to m\mu_*^2$, $w_2=\Delta_t$ and $c_2=\mu_I=\Delta_t\iota+O(\Delta_t^2)$,
equation \eqref{eq:app_fp_ridge} gives
\begin{align}
\label{eq:app_smallt_L}
  L_1-1=\frac{m\Delta_t}{s}\Big(\mu_1\bar r+\frac{\mu_*^2}{\tilde\gamma}\Big),
  \qquad
  L_2-1=\frac{\Delta_t^2}{s}\Big(\mu_1\bar r+\frac{\iota}{\tilde\gamma}\Big),
\end{align}
so $L_1,L_2\to1$ as soon as $\tilde\gamma\gg\mu_*^2/\psi_n$. Second, the fixed point
becomes explicit: $g\to m$ and $R\to\mu_1$, so $1/\bar r=\tilde\gamma+R(\bar r)$ collapses
to $\bar r=(\tilde\gamma+\mu_1)^{-1}$ and
\begin{align}
\label{eq:app_smallt_b}
  b=\frac{\mu_1\bar rg}{m}=\frac{\mu_1}{\tilde\gamma+\mu_1},
  \qquad
  \mathcal B^2=(b-1)^2=\frac{\tilde\gamma^2}{(\tilde\gamma+\mu_1)^2} .
\end{align}
Third, the trace. From $L_1'=\mu_1m\Delta_t/s$ and $L_2'=\mu_1\Delta_t^2/s$ one gets
$g'\to-\mu_1m\Delta_t/s$, hence $R'\to0$, $\bar r^{-2}+R'\to(\tilde\gamma+\mu_1)^2$ and
$\Phi'\to-\mu_1\Delta_t(1+1/s)$, so that $\mu_1T_4/\Delta_t=\mu_1^2(1+1/s)/(\tilde\gamma+\mu_1)^2$
and
\begin{align}
\label{eq:app_smallt_V}
  \mathcal V=\frac{\mu_1T_4}{\Delta_t}-b^2=\frac{\mu_1^2}{s\,(\tilde\gamma+\mu_1)^2} .
\end{align}
Equations \eqref{eq:app_smallt_b} and \eqref{eq:app_smallt_V} are stronger than
Result~\ref{res:bv_smallt}: the whole ridge curve, not only its minimum, depends on
$\psi_n$, $m$ and $t$ through $s$ alone, and $\mathcal B^2$ does not depend on $s$ at all.
Fourth, the optimum. The excess risk is
\begin{align}
\label{eq:app_smallt_risk}
  \mathcal B^2+\mathcal V=\frac{\tilde\gamma^2+\mu_1^2/s}{(\tilde\gamma+\mu_1)^2},
\end{align}
whose derivative is proportional to $(\tilde\gamma+\mu_1)(\mu_1\tilde\gamma-\mu_1^2/s)$ and
vanishes at $\tilde\gamma^*=\mu_1/s$; substituting gives
\eqref{eq:app_smallt_collapse}.

\subsection{Link between kernels and random features}
\label{app:extra_results}
\label{app:kernel_rf}

Any dot-product kernel $K(\vx,\vy)=f(\frac{\vx^\top\vy}{d})$ can be written as the infinite-width limit of a random-features model \citep{Rahimi_2007} $K(\vx,\vy)=\underset{p\rightarrow\infty}{\operatorname{lim}}\frac{1}{p}\sum_{i=1}^p \sigma(\frac{\vw_i^\top \vx}{\sqrt{d}})\sigma(\frac{\vw_i^\top \vy}{\sqrt{d}})$ with $\vW\sim\mathcal{N}(0,\vI_{p\times d})$ and an activation $\sigma$ that depends on $f$. We show below that the large-$p$ limit of the Gram matrix of the random-features network on our dataset yields the same linear equivalent as Theorem~\ref{app:thm:lin_gram_empirical}. The difference with \citet{bonnaire2025diffusionmodelsdontmemorize} is that they take the number of noises $m\to\infty$ first, whereas we keep $m=O(1)$.
The results of this section are presented for $\vSigma=\vI_d$ but extend to arbitrary covariance. We first recall the choice of activation realizing a given dot-product kernel.
\begin{lemma}
   For any dot-product kernel $K(\vx,\vy)=f(\frac{\vx^\top\vy}{d})$, there exists an activation function $\sigma$ such that
\begin{align}
    K(\vx,\vy)=\mathbb{E}_{\vw\sim\mathcal{N}(0,\vI_{d})}[\sigma(\frac{\vw^\top\vx}{\sqrt{d}})\sigma(\frac{\vw^\top\vy}{\sqrt{d}})]
\end{align}

\end{lemma}
\begin{proof}
    Assume that $f$ can be expanded as a Taylor series $f(u)=\sum_{k\ge0}a_k u^k$ with $a_k\ge0$\footnote{The positivity of the coefficients holds since we assumed that $K(\vx,\vy)$ is a positive-definite kernel.} and that $\sigma$ can be expanded on a basis of Hermite polynomials $\sigma(u)=\sum_{k\ge0}b_k H_k(u)$. By Mehler's formula \citep{kibble1945},
    \begin{align}
        \mathbb{E}_{\vw}[\sigma(\frac{\vw^\top\vx}{\sqrt{d}})\sigma(\frac{\vw^\top\vy}{\sqrt{d}})]=\sum_{k\ge0}b_k^2k! \left(\frac{\vx^\top\vy}{d}\right)^k,
    \end{align}
    hence it suffices to choose the Hermite coefficients of $\sigma$ such that $\forall k\ge 0,\ b_k^2k!=a_k$.
\end{proof}
We then derive a Gaussian Equivalence Principle (GEP) \citep{Gerace_2020, goldt_2021} for the random-features model.
\begin{proposition}[Gaussian equivalence for the random-features model]
\label{app:prop:gep}
    Assume $\mu'_0=\mathbb{E}_{\mathcal{N}(0,1)}[\sigma(u)]=0$. In the proportional regime $p\asymp n\asymp d\gg 1$, the random-features model $\sigma(\frac{\vW\vY}{\sqrt{d}})$ is equivalent to the Gaussian
    \begin{align}
        \mu'_1\frac{\vW\vY}{\sqrt{d}}+\mu'_*\vOmega
    \end{align}
    with $\vOmega\in\mathbb{R}^{p\times mn}$ with Gaussian entries and covariance $\mathbb{E}[\vOmega_{i}^{\nu\alpha}\vOmega_{j}^{\nu'\alpha'}]=\delta_{ij}\delta^{\nu\nu'}\left(\delta^{\alpha\alpha'}+\kappa_t(1-\delta^{\alpha\alpha'}) \right)$ with $\mu'_1=\mathbb{E}[\sigma(z)z]$, ${\mu'}_*^2=\mathbb{E}[\sigma^2(z)]-{\mu'}_1^2$ and $\kappa_t=\frac{1}{{\mu'}_*^2}\mathbb{E}_{u,v,w}[(\sigma(e^{-t}u+\sqrt{\Delta_t}v)-\mu'_1e^{-t}u)(\sigma(e^{-t}u+\sqrt{\Delta_t}w)-\mu'_1e^{-t}u)].$
\end{proposition}
\begin{proof}
   According to the Gaussian Equivalence Principle, $\sigma(\frac{\vW\vY}{\sqrt{d}})$ and  $\frac{\vW\vY}{\sqrt{d}}$ are jointly Gaussian; hence, to characterize their statistics, it suffices to compute their covariance.

For each $i \in [p]$, $\nu \in [n]$, $\alpha \in [m]$, define the
pre-activation
\[
    \vz_i^{\nu\alpha} = \frac{1}{\sqrt{d}}\sum_{k=1}^d \vW_{ik}\,\vY_k^{\nu\alpha}.
\]
Conditioned on $\vY$, the $\vz_i^{\nu\alpha}$ are jointly centered Gaussian
with
\[
    \mathbb{E}\left[\vz_i^{\nu\alpha}\,\vz_j^{\nu'\alpha'} \,\Big|\, \vY\right]
    \;=\; \delta_{ij}\,\frac{\vY^{\nu\alpha}\cdot \vY^{\nu'\alpha'}}{d}.
\]

Define the centered nonlinearity
\[
    \tilde\Omega(z) \;:=\; \frac{\sigma(z) - \mu'_1\,z}{\mu'_*}, \qquad
    \mu'_1 = \mathbb{E}_{z\sim\mathcal{N}(0,1)}[z\,\sigma(z)], \quad
    (\mu'_*)^2 = \mathbb{E}[\sigma(z)^2] - (\mu'_1)^2.
\]
Direct calculation, using the assumption $\mu'_0 = \mathbb{E}[\sigma] = 0$
and Stein's identity $\mathbb{E}[z\sigma(z)] = \mathbb{E}[\sigma'(z)]$ for
$z \sim \mathcal{N}(0,1)$, gives the orthonormality conditions
\[
    \mathbb{E}_{z}[\tilde\Omega(z)] = 0, \qquad
    \mathbb{E}_{z}[z\,\tilde\Omega(z)] = 0, \qquad
    \mathbb{E}_{z}[\tilde\Omega(z)^2] = 1.
\]
Setting
$\vOmega_i^{\nu\alpha} = \tilde\Omega(\vz_i^{\nu\alpha})$ yields
\begin{equation}
    \sigma\left(\tfrac{(\vW\vY)_i^{\nu\alpha}}{\sqrt{d}}\right)
    \;=\; \mu'_1\,\tfrac{(\vW\vY)_i^{\nu\alpha}}{\sqrt{d}}
       \;+\; \mu'_*\,\vOmega_i^{\nu\alpha}.
\end{equation}
This equality holds pointwise, not just asymptotically. We still need to show that $\vOmega$ behaves as a Gaussian matrix with a specific covariance
in the proportional regime; we verify the covariance below and invoke
the GEP for the equivalence as Gaussian objects.
Let us compute $\mathbb{E}[\vOmega_i^{\nu\alpha} \vOmega_j^{\nu'\alpha'}]$
case by case.

\textbf{Different rows of $\vW$ ($i \neq j$).} The rows $\vw_i, \vw_j$ are
independent, so $\vz_i^{\nu\alpha}, \vz_j^{\nu'\alpha'}$ are conditionally
independent and
$\mathbb{E}[\vOmega_i^{\nu\alpha}\vOmega_j^{\nu'\alpha'}]
= \mathbb{E}[\tilde\Omega(\vz_i^{\nu\alpha})]\mathbb{E}[\tilde\Omega(\vz_j^{\nu'\alpha'})] = 0$.
This produces the prefactor $\delta_{ij}$.

\textbf{Different clusters ($i = j$, $\nu \neq \nu'$).} The overlap
$\vY^{\nu\alpha}\cdot \vY^{\nu'\alpha'}/d$ vanishes asymptotically
(by independence of $\vx^\nu, \vx^{\nu'}$ and Gaussian concentration),
so $\vz_i^{\nu\alpha}$ and $\vz_i^{\nu'\alpha'}$ are asymptotically
independent. Then
$\mathbb{E}[\vOmega_i^{\nu\alpha}\vOmega_i^{\nu'\alpha'}]
\to \mathbb{E}[\tilde\Omega(z)]^2 = 0$, giving the prefactor
$\delta^{\nu\nu'}$.

\textbf{Diagonal ($i = j$, $\nu = \nu'$, $\alpha = \alpha'$).} Same
Gaussian, $\mathbb{E}[\tilde\Omega(z)^2] = 1$.

\textbf{Same cluster, different noise copies ($i = j$, $\nu = \nu'$,
$\alpha \neq \alpha'$).} The overlap concentrates around
\[
    \frac{\vY^{\nu\alpha}\cdot \vY^{\nu\alpha'}}{d} \;\xrightarrow{d\to\infty}\; e^{-2t}\sigma^2 \;=\; e^{-2t}
\]
(the last equality assumes $\sigma^2 = 1$; the general isotropic case is discussed at the
end of this proof). Conditional on
$\vY$, $(\vz_i^{\nu\alpha}, \vz_i^{\nu\alpha'})$ is a centered Gaussian
pair with unit marginal variance and correlation $\rho = e^{-2t}$. We
realize this pair as
\[
    \vz_i^{\nu\alpha} = e^{-t}u + \sqrt{\Delta_t}\,v, \qquad
    \vz_i^{\nu\alpha'} = e^{-t}u + \sqrt{\Delta_t}\,w,
\]
with $u, v, w \stackrel{\mathrm{i.i.d.}}{\sim} \mathcal{N}(0, 1)$. Then
\begin{equation}
\label{eq:gep:cov_step1}
    (\mu'_*)^2\,\mathbb{E}\left[\vOmega_i^{\nu\alpha}\vOmega_i^{\nu\alpha'}\right]
    \;=\; \mathbb{E}\left[(\sigma(z) - \mu'_1\,z)(\sigma(z') - \mu'_1\,z')\right].
\end{equation}
Expand the product and use Stein's lemma at correlation $\rho = e^{-2t}$,
which gives $\mathbb{E}[z'\,\sigma(z)] = \rho\,\mathbb{E}[\sigma'(z)] = \rho\,\mu'_1$
and similarly $\mathbb{E}[z\,\sigma(z')] = \rho\,\mu'_1$:
\begin{align*}
    \mathbb{E}[(\sigma(z) - \mu'_1 z)(\sigma(z') - \mu'_1 z')]
    &= \mathbb{E}[\sigma(z)\sigma(z')] \;-\; 2\,\mu'_1\,\rho\,\mu'_1 \;+\; (\mu'_1)^2\,\mathbb{E}[zz']\\
    &= \mathbb{E}[\sigma(z)\sigma(z')] - (\mu'_1)^2\,\rho.
\end{align*}

The proposition defines $\kappa_t$ via the integrand
$(\sigma(e^{-t}u + \sqrt{\Delta_t}v) - \mu'_1 e^{-t}u)(\sigma(e^{-t}u + \sqrt{\Delta_t}w) - \mu'_1 e^{-t}u)$,
which subtracts only $\mu'_1 e^{-t}u$ from each factor instead of the
full $\mu'_1 z$. The two definitions coincide because the missing
$\mu'_1 \sqrt{\Delta_t}v$ and $\mu'_1\sqrt{\Delta_t}w$ pieces cancel:
they are mean-zero, mutually independent, and each
of them is independent of the remaining terms in the cross expectation.
Concretely, expanding the product
\[
    (\sigma(z) - \mu'_1 e^{-t}u - \mu'_1\sqrt{\Delta_t}v)(\sigma(z') - \mu'_1 e^{-t}u - \mu'_1\sqrt{\Delta_t}w)
\]
and taking expectation, the cross terms involving $v$ or $w$ alone vanish
(because $v$ is independent of $u$ and $w$, and $\mathbb{E}[v] = 0$;
similarly for $w$), and the term $(\mu'_1)^2\Delta_t\,vw$ vanishes by
$\mathbb{E}[vw] = 0$. Hence \eqref{eq:gep:cov_step1} equals
\[
    \mathbb{E}\left[(\sigma(e^{-t}u + \sqrt{\Delta_t}v) - \mu'_1 e^{-t}u)
    (\sigma(e^{-t}u + \sqrt{\Delta_t}w) - \mu'_1 e^{-t}u)\right]
    \;=\; (\mu'_*)^2\,\kappa_t,
\]
which is the definition of $\kappa_t$ in the proposition. So
$\mathbb{E}[\vOmega_i^{\nu\alpha}\vOmega_i^{\nu\alpha'}] = \kappa_t$ for
$\alpha \ne \alpha'$.

Combining the four cases:
\[
    \mathbb{E}[\vOmega_i^{\nu\alpha}\vOmega_j^{\nu'\alpha'}]
    \;=\; \delta_{ij}\,\delta^{\nu\nu'}\,\Big(\delta^{\alpha\alpha'} + \kappa_t\,(1 - \delta^{\alpha\alpha'})\Big).
\]
This is the claimed covariance.

\emph{General isotropic data $\vSigma=\sigma^2\vI_d$.} The same computation goes through
with two changes: the pre-activations have marginal variance
$\lVert\vY^{\nu\alpha}\rVert^2/d\to\Gamma_t=\sigma^2e^{-2t}+\Delta_t$ instead of $1$, and two
copies of the same clean sample have correlation $\rho=\sigma^2e^{-2t}/\Gamma_t$ instead of
$e^{-2t}$. The Gaussian equivalence then holds with $\mu'_1$, $\mu'_*$ and $\kappa_t$ computed
for $z\sim\mathcal N(0,\Gamma_t)$ and pairs of correlation $\rho$, and the resulting
constants of the Gram matrix,
$\mu_I=\mu_*'^2(1-\kappa_t)$, $\mu_B=\mu_*'^2\kappa_t$ and $\mu_1={\mu'}_1^2$, are exactly
those of Theorem~\ref{app:thm:lin_gram_empirical} at general $\sigma^2$. The equations of
Theorem~\ref{app:thm:bias_variance} therefore hold with these constants; equivalently, the
exact reduction to $\sigma^2=1$ stated after Theorem~\ref{app:thm:bias_variance} applies.

The Gaussian Equivalence Principle of
\citet{Gerace_2020, goldt_2021, hu2023, mei2020}
asserts that, in the proportional regime $p \asymp n \asymp d$, all
asymptotic statistics of $\sigma(\vW\vY/\sqrt d)$ coincide with those obtained by replacing
$\vOmega$ by a centered Gaussian matrix with the same covariance.
Combining with the exact decomposition above gives the claimed
distributional equivalence
$\sigma(\vW\vY/\sqrt d) \stackrel{\mathrm{d}}{\sim} \mu'_1 \vW\vY/\sqrt d + \mu'_* \vOmega^{\mathrm{Gauss}}$,
where $\vOmega^{\mathrm{Gauss}}$ has the prescribed covariance and is
otherwise jointly Gaussian.

\end{proof}
\begin{proposition}
    In the overparametrized regime $p\gg n\asymp d\gg 1$, the Gram matrix of the random-features model $\vG_{\mathrm{RF}}=\frac{1}{p}\sum_{i=1}^p\sigma(\frac{\vW_i\vY}{\sqrt{d}})^\top\sigma(\frac{\vW_i\vY}{\sqrt{d}})$ converges to
    \begin{align}
        \mu_I\vI_N+\mu_B\vB_m+\mu_1\frac{\vY^\top\vY}{d}
    \end{align}
    with $\mu_I={\mu'}_*^2(1-\kappa_t),\  \mu_B={\mu'}_*^2\kappa_t,\ \mu_1={\mu'}_1^2$.
\end{proposition}
\begin{proof}
    Replace the random features by their Gaussian equivalent and the sum over hidden units by an expectation over $\vW$ and $\vOmega$.
    \begin{align}
        \vG_{\mathrm{RF}}^{\nu\alpha,\nu'\alpha'}&=\mathbb{E}_{\vW,\vOmega}\left[\left(\mu'_1\frac{\vW^\top \vY^{\nu\alpha}}{\sqrt{d}}+\mu'_*\vOmega^{\nu\alpha}\right)\left(\mu'_1\frac{\vW^\top \vY^{\nu'\alpha'}}{\sqrt{d}}+\mu'_*\vOmega^{\nu'\alpha'}\right)\right]\\
&={\mu'}_1^2\frac{\vY^{\nu\alpha} \cdot\vY^{\nu'\alpha'}}{d}+{\mu'}_*^2\mathbb{E}[\vOmega^{\nu\alpha} \vOmega^{\nu'\alpha'}]\\
&={\mu'}_1^2\frac{\vY^{\nu\alpha} \cdot\vY^{\nu'\alpha'}}{d}+{\mu'}_*^2\left(\delta^{\nu\nu'}\delta^{\alpha \alpha'}+\kappa_t \delta^{\nu\nu'}(1-\delta^{\alpha \alpha'})\right).
    \end{align}
    Hence
    \begin{align}
        \vG_{\mathrm{RF}}={\mu'}_1^2 \frac{\vY^\top\vY}{d}+{\mu'}_*^2(1-\kappa_t)\vI_N+{\mu'}_*^2\kappa_t\vB_m.
    \end{align}
\end{proof}
The infinite-width limit of the Gram matrix of a random-features network is thus consistent with the linear equivalent of the Gram matrix of the associated kernel.

The same Gaussian Equivalence Principle also accounts for the finite-width, infinite-$m$
limit studied by \citet{bonnaire2025diffusionmodelsdontmemorize}. These authors work with
the Hessian of the loss at initialization,
\begin{align}
    \vU^{ij}=\frac{1}{n}\sum_{\nu=1}^n\mathbb{E}_{\vxi}\left[\sigma\left(\frac{(\vW(e^{-t}\vx^\nu+\sqrt{\Delta_t}\vxi))_i}{\sqrt{d}}\right)\sigma\left(\frac{(\vW(e^{-t}\vx^\nu+\sqrt{\Delta_t}\vxi))_j}{\sqrt{d}} \right)\right],
\end{align}
a $p\times p$ matrix in feature space whose counterpart at finite $m$ is
\begin{align}
    \vU^{ij}(m)=\frac{1}{nm}\sum_{\nu=1}^n \sum_{\alpha=1}^m\sigma\left(\frac{(\vW\vY^{\nu\alpha})_i}{\sqrt{d}}\right)\sigma\left(\frac{(\vW\vY^{\nu\alpha})_j}{\sqrt{d}} \right),
\end{align}
so that $\vU=\lim_{m\to\infty}\vU(m)$. Replacing the features by their Gaussian equivalent,
\begin{align}
     \vU^{ij}(m)=\frac{1}{nm}\sum_{\nu=1}^n \sum_{\alpha=1}^m\left(\mu'_1\tfrac{(\vW\vY)_i^{\nu\alpha}}{\sqrt{d}}
       + \mu'_*\,\vOmega_i^{\nu\alpha}\right)\left(\mu'_1\,\tfrac{(\vW\vY)_j^{\nu\alpha}}{\sqrt{d}}
       + \mu'_*\,\vOmega_j^{\nu\alpha}\right),
\end{align}
and the three groups of terms average separately. In the Gaussian equivalent $\vOmega$ is
centered and uncorrelated with $\vW\vY/\sqrt d$, so the two cross terms vanish,
\begin{align}
    \frac{\mu'_1\mu'_*}{nm}\sum_{\nu,\alpha}\tfrac{(\vW\vY)_i^{\nu\alpha}}{\sqrt{d}}\,\vOmega_j^{\nu\alpha}\;\longrightarrow\;0 .
\end{align}
The $\vOmega$--$\vOmega$ term pairs every copy with itself and never with another, so the
cross-copy covariance $\kappa_t$ is never probed and
\begin{align}
    \frac{{\mu'}_*^2}{nm}\sum_{\nu,\alpha}\vOmega_i^{\nu\alpha}\vOmega_j^{\nu\alpha}\;\longrightarrow\;{\mu'}_*^2\,\delta_{ij}.
\end{align}
The average of the linear part over the copies becomes an expectation over
$\vxi\sim\mathcal{N}(0,\vI_d)$ at fixed $\vW$ and fixed data, which gives
\begin{align}
    \frac{{\mu'}_1^2}{nmd}\sum_{\nu,\alpha}(\vW\vY)_i^{\nu\alpha}(\vW\vY)_j^{\nu\alpha}
    \;\longrightarrow\;{\mu'}_1^2\left[\Delta_t\frac{(\vW\vW^\top)_{ij}}{d}+e^{-2t}\frac{(\vW\hat{\vSigma}\vW^\top)_{ij}}{d}\right],
\end{align}
with $\hat{\vSigma}=\frac1n\sum_\nu\vx^\nu(\vx^\nu)^\top$ the empirical covariance of the
clean samples. Collecting the three pieces, and using $\mu_1={\mu'}_1^2$ together with
${\mu'}_*^2=\mu_I+\mu_B$ from the proposition above,
\begin{align}
\label{eq:app:hessian_minf}
    \vU \;=\; \mu_1\,\frac{\vW\hat{\vSigma}_t\vW^\top}{d}\;+\;(\mu_I+\mu_B)\,\vI_p,
    \qquad \hat{\vSigma}_t=e^{-2t}\hat{\vSigma}+\Delta_t\vI_d .
\end{align}
Equation \eqref{eq:app:hessian_minf} is the feature-space mirror of the linear equivalent
of Theorem~\ref{app:thm:lin_gram_empirical}, and it reproduces the two-group spectrum from which \citet{bonnaire2025diffusionmodelsdontmemorize} read off the two timescales:
$\vW\hat{\vSigma}_t\vW^\top/d$ has rank $d$, so $\vU$ carries $d$ large eigenvalues, the
fast directions, above a plateau of $p-d$ eigenvalues pinned exactly at $\mu_I+\mu_B$, the
slow ones.

The comparison with the finite-$m$ Gram matrix is then transparent. Since $\kappa_t$ drops
out, $\mu_I$ and $\mu_B$ enter \eqref{eq:app:hessian_minf} only through their sum: the
$m=\infty$ Hessian has exactly the structure of our own Gram matrix at $m=1$, where
$\vB_m=\vI_N$ and $\vG_{\mathrm{lin}}=(\mu_I+\mu_B)\vI_N+\mu_1\vY^\top\vY/d$. What $m>1$ does is
resolve that single degenerate plateau into the three objects of
Theorem~\ref{app:thm:Spectrum_gram_linear}: the atom at $\mu_I$, the atom at
$\bar\mu=\mu_I+m\mu_B$ and the memorization bulk at $\mu_I+\mu_B(m-1)\Delta_t/\lambda_t$,
each of which depends on $\mu_I$ and $\mu_B$ separately and therefore on the cross-copy
correlation $\kappa_t$. Averaging over the copies before forming the matrix, as the
$m=\infty$ limit does, removes precisely the structure that the memorization bulk lives on.

\section{Experimental details}
\label{app:sect:numerical}

\subsection{General comments} \label{app:subsec:general}

\paragraph*{Code.} A reference implementation of all the numerical experiments will be made available upon publication of the paper.

\paragraph*{Generation.} Once the velocity field $\hat{\vv}(\vx_s, s)$ is estimated (either by a U-Net or by the spectrally truncated estimator $\hat{\vv}_r$) on a uniform grid of $S$ time slices $\{s_k\}_{k=0}^{S-1}\subset[0,1]$, we generate samples by integrating the rectified-flow ODE
\begin{align}\label{eq:rf_ode}
    \frac{\dd \tilde{\vx}_s}{\dd s} = \hat{\vv}(\tilde{\vx}_s, s),
    \qquad \tilde{\vx}_1 \sim \mathcal{N}(0,\vI_d),
\end{align}
backward with an Euler scheme on the same $S$ slices used at training. For the truncation experiment, at each timestep, the cross-kernel $K(\tilde{\vx}_{s_k},\vY_{s_k})$ between the current iterate and the noised training anchors is evaluated through the same closed-form NTK recursion and the velocity is computed with a single matrix-vector product against the precomputed coefficients $\valpha_r$ of the truncated estimator $\hat{\vv}_r$ of \eqref{eq:app_truncated_estimator}, i.e.\ $K(\tilde{\vx}_{s_k},\vY_{s_k})\,\valpha_r(s_k)$.

\paragraph*{Details on $f_\mathrm{mem}$.} For each generated sample $\tilde{\vx}_0$ with the backward dynamics, let $d_{\mathrm{NN}_1}$ and $d_{\mathrm{NN}_2}$ denote the cosine distances to its first and second nearest neighbors in the training set $ \{\vx^{\nu}\}_{\nu=1}^{n}$. We declare
$\tilde{\vx}_0$ memorized whenever $\frac{d_{\mathrm{NN}_1}(\tilde{\vx}_0)}{d_{\mathrm{NN}_2}(\tilde{\vx}_0)} < \tfrac{1}{3}$, and we define $f_{\mathrm{mem}}\in[0,1]$ as the fraction of generated samples that meet this condition among $2048$ generated samples. We emphasize that $f_\mathrm{mem}$ is evaluated on backward trajectories starting from training noises $\tilde{\vx}_1 \in \{\vxi^{\mu\alpha}\}$ and can therefore be viewed as a long-time version of the U-turn protocol from \citet{Behjoo_2025, sclocchi_2024}. A model that has learned the correct population velocity field should produce an image that resembles those of the data manifold but is distinct from its associated clean image, therefore producing $f_\mathrm{mem} = 0$, while a model in the memorization regime should fall back onto the clean training image and lead to $f_\mathrm{mem}=1$.

\subsection{Closed-form CNTK computation and top-$r$ eigendecomposition} \label{app:subsec:cntk}

\paragraph*{Kernel recursion.}
We follow the depth-$D$ recursion of \citet{Arora2019}, described in Sect.~\ref{app:ntk:cnn} and implemented in the authors' \href{https://github.com/ruosongwang/CNTK/tree/master}{\texttt{CNTK} repository}, except we use a $1\times 1$ readout which is equivalent to the flatten+linear output up to a $1/P$ normalization with $P=L^2$ which is a global multiplicative constant that is absorbed in the kernel-regression predictors.

\paragraph*{Top-$r$ eigenvalues.} For $N=nm < 10^4$, we materialize the empirical Gram matrix $\vG \in \mathbb{R}^{N\times N}$ and call a standard Lanczos solver (\texttt{ARPACK} package) to find its eigenvalues. For $N > 10^4$, we use the randomized SVD method from \citet{Halko2011} with a matrix-free linear operator that evaluates matrix-vector products in chunks without ever forming $\vG$ explicitly. We use an oversampling of $p=10$ columns and a second-order ($q=2$) power iteration, i.e.\ $(q+2)(r+p) \approx 4r$ matrix-vector products in total to compute the top-$r$ eigenvalues.

\paragraph*{Spectrally truncated estimator.} Writing the resulting decomposition as
$\vG\approx\vU_r\vLambda_r\vU_r\tran$, with $\vU_r$ collecting in column the top-$r$
eigenvectors of $\vG$ and $\vLambda_r$ the corresponding descending-ordered eigenvalues,
the spectrally truncated kernel regression estimator of the velocity field used in
Sect.~\ref{sect:Numerical} is
\begin{align}\label{eq:app_truncated_estimator}
    \hat{\vv}_r(\vx, s) = K(\vx,\vY)\,\valpha_r
      = K(\vx,\vY)\,\vU_r\vLambda_r^{-1}\vU_r\tran\vV ,
\end{align}
where $\vV\in\mathbb{R}^{N\times d}$ stacks the training target velocities.

\paragraph*{Truncation rank and training time.} Truncating at rank $r$ is equivalent to
fixing a training time. The gradient-flow filter of \eqref{eq:estimation_score_DSM}
activates the modes above $\lambda_c\sim nm/(d\tau)$ and leaves the others untouched, so
discarding every eigenvalue below $\lambda_r$ is the hard-threshold counterpart of stopping
at $\tau\sim nm/(d\lambda_r)$. This is what makes the truncation rank of
Sect.~\ref{sect:Numerical} a proxy for the training time, and the generalization to
memorization transition it produces the same one that the theory of
Sect.~\ref{sect:Analytical} describes as a function of $\tau$. The three rules (truncation, gradient flow, ridge) differ in the shape
of the filter, not in the order in which they activate the modes; the paragraph
\emph{Early stopping versus ridging} of Sect.~\ref{app:subsec:krr_bv} quantifies how far
the ridge and the flow depart from one another.

\begin{figure}
    \centering
    \includegraphics[width=0.75\linewidth]{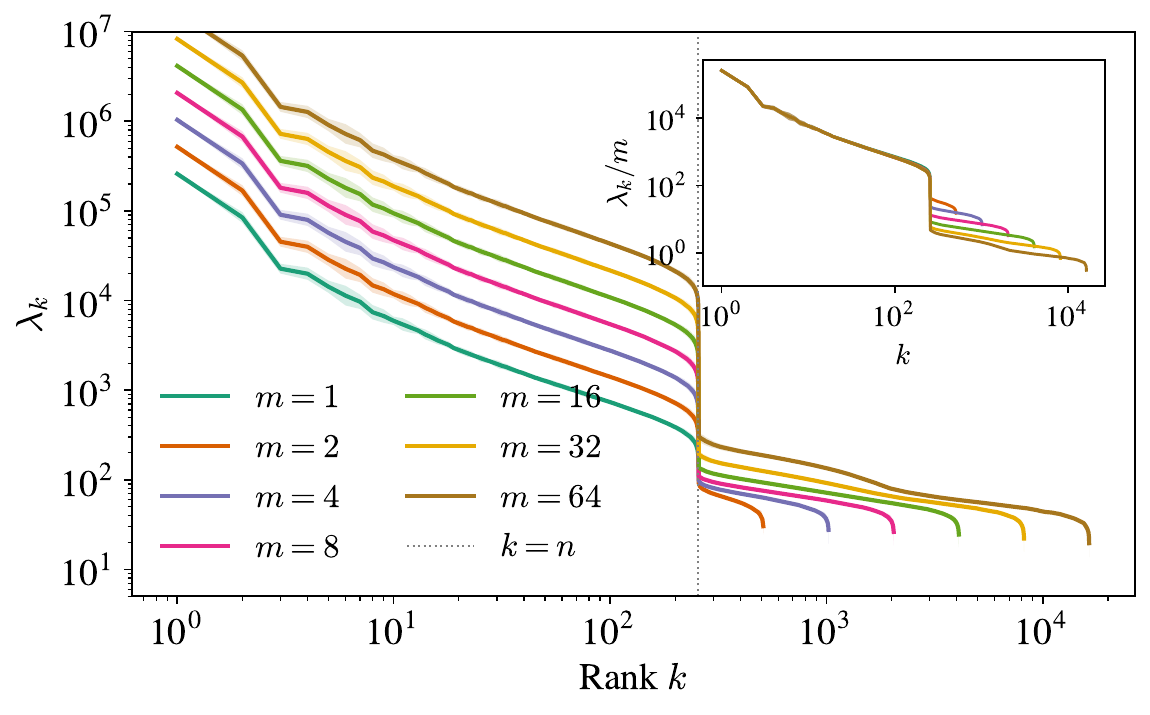}
    \caption{Ordered eigenvalues for several $m$ at fixed $n=256$ and their rescaling by $m$ in the inset. Spectra obtained at $s=0.05$ and averaged over 10 realizations of the training set.}
    \label{fig:cntk_m_scaling}
\end{figure}

\paragraph*{Scaling with $m$ of the CNTK.} Figure~\ref{fig:cntk_m_scaling} displays the evolution of the CNTK spectrum for several values of $m\in\{1,2,4,8,16,32,64\}$ at fixed $n=256$. The entire spectrum shifts with $m$: the first bulk linearly in $m$, in agreement with the theoretical prediction, and the second bulk sub-linearly, as emphasized by the inset. The figure also shows the absence of the memorization bulk at $m=1$, where the standard single-bulk spectrum is recovered.

\subsection{U-Net training and empirical NTK on a finite-width network} \label{app:subsec:unet}

At the end of Sect.~\ref{sect:Numerical}, we compute the spectrum of the time-dependent empirical NTK (eNTK) $K(\vtheta; \vx,\vx')=\sum_{i,k}\partial_{\theta_i}\vv^{(k)}_{\vtheta}(\vx)\,
\partial_{\theta_i}\vv^{(k)}_{\vtheta}(\vx')$ of a U-Net architecture.

\paragraph*{U-Net architecture and training.} The model follows the implementation of \citet{bonnaire2025diffusionmodelsdontmemorize} with three residual blocks, base width $W=32$ and channel multipliers $\{1, 2, 4\}$, trained on CelebA with $n=256$ images and $m=8$ noise realizations per data point. The model is trained on the rectified-flow objective with Adam at a learning rate of $10^{-4}$ for a total of $\tau=60\,000$ steps.

\paragraph*{Empirical NTK estimator.} Materializing the full Jacobian $\partial_{\vtheta}\vv_{\vtheta}(\vx)\in\mathbb{R}^{d\times|\vtheta|}$ is intractable, so we use a Hutchinson-style projection estimator of the empirical NTK matrix
\begin{equation}
    K_{\mathrm{eNTK}} \approx \frac{1}{n_{\mathrm{proj}}}\sum_{r=1}^{n_{\mathrm{proj}}}
    \langle\nabla_{\vtheta}(\vw_r\tran \vv_{\vtheta}(\vx)),\nabla_{\vtheta}(\vw_r\tran \vv_{\vtheta}(\vx'))\rangle,
\end{equation}
with $\vw_r\sim\mathcal{N}(0,\vI_{d})$ and $n_{\mathrm{proj}}=16$. Each scalar $\vw_{r}^{\top}\vv_{\vtheta}(\vx)$ is differentiated with a single reverse-mode pass so that the full Jacobian is never formed. The resulting gradients are batched over data points and contracted into the kernel in chunks.

\paragraph{Regularized rectified-flow training.}
We train the velocity network $\vv_{\vtheta}$ on the training set
$\{(\vx^\nu,\vxi^{\nu\alpha})\}$ with a spectrum-aware Tikhonov
penalty
\begin{equation}
  \mathcal{L}(\vtheta) = \frac{1}{nmd}\sum_{\nu=1}^n \sum_{\alpha=1}^m\mathbb{E}_{s}\,\big\|
    \vv_{\vtheta}(\vY_{s}^{\nu\alpha},s)-(\vxi^{\nu\alpha}-\vx^{\nu})\big\|_{2}^{2} + \frac{\gamma_{s, \tau}}{nmd^{2}}\|\vtheta\|_{2}^{2},
  \label{eq:reg-loss}
\end{equation}
where $\vY_{s}^{\nu\alpha}=(1-s)\vx^{\nu}+s\vxi^{\nu\alpha}$, and $\gamma_{s, \tau}$ is the regularization parameter, chosen as the $n$-th eigenvalue of the empirical NTK Gram
matrix $\vG_\tau(s) \in \mathbb{R}^{nm\times nm}$ evaluated on the full training set at flow-matching time $s$.
The additional $1/d$ factor in the prefactor of the regularization converts the traced-eNTK eigenvalues we measure into those of the full NTK under the output-isotropy assumption \eqref{eq:ntk_decomposable}. In the linearization around $\vtheta_0$,
the stationarity of Eq.~\eqref{eq:reg-loss} is equivalent to a kernel-ridge problem $(\vG(s)+\gamma\vI)\valpha=\vV$. A regularization of amplitude $\gamma_{s,\tau}=\lambda_n$ then damps the contributions of the eigenmodes of $\vG$ below the $n$-th. As the eNTK drifts with training time in the feature-learning regime, the parameter $\gamma_{s,\tau}$ has to be tracked. At step $0$ we eigendecompose $\vG(s)$ on the current parameters $\vtheta_0$ for each $s$ in the discrete grid of size $S$, and refresh every $\Delta\tau=1000$ training steps by repeating the same procedure on $\vtheta(\tau)$.
In Fig.~\ref{fig:regularization}, we display the evolution of the train and test errors corresponding to the training runs of Fig.~\ref{fig:time_evolution_NTK}. Without regularization (left panel), the test loss increases sharply at large training time, together with a steep rise in memorization, whereas with the regularization (right panel) the network stays close to the configuration of optimal test loss.

\begin{figure}
    \centering
    \includegraphics[width=0.49\linewidth]{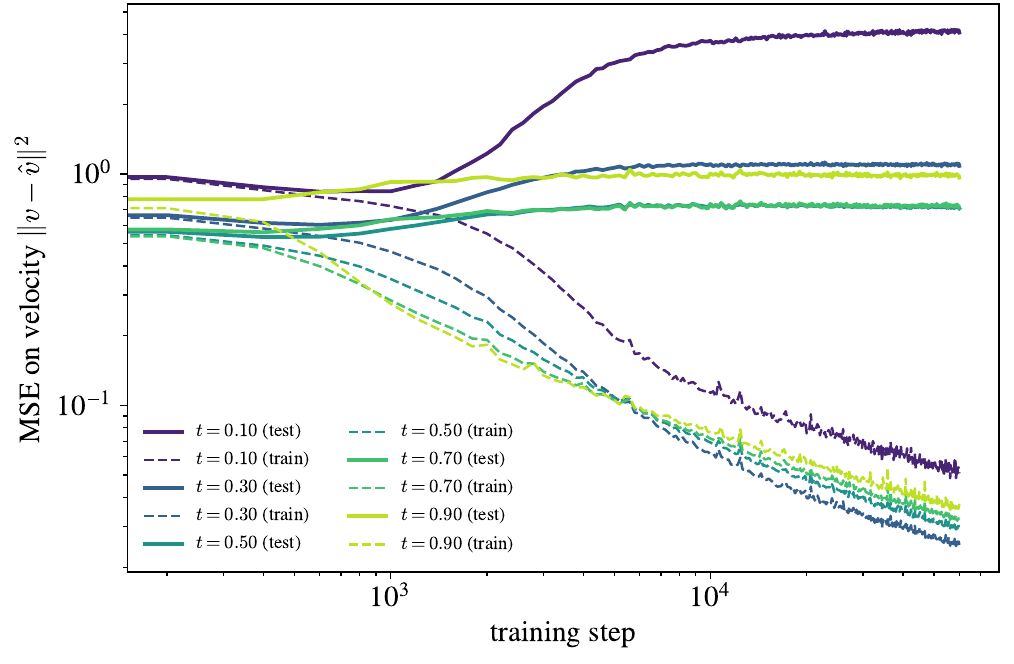}
    \includegraphics[width=0.49\linewidth]{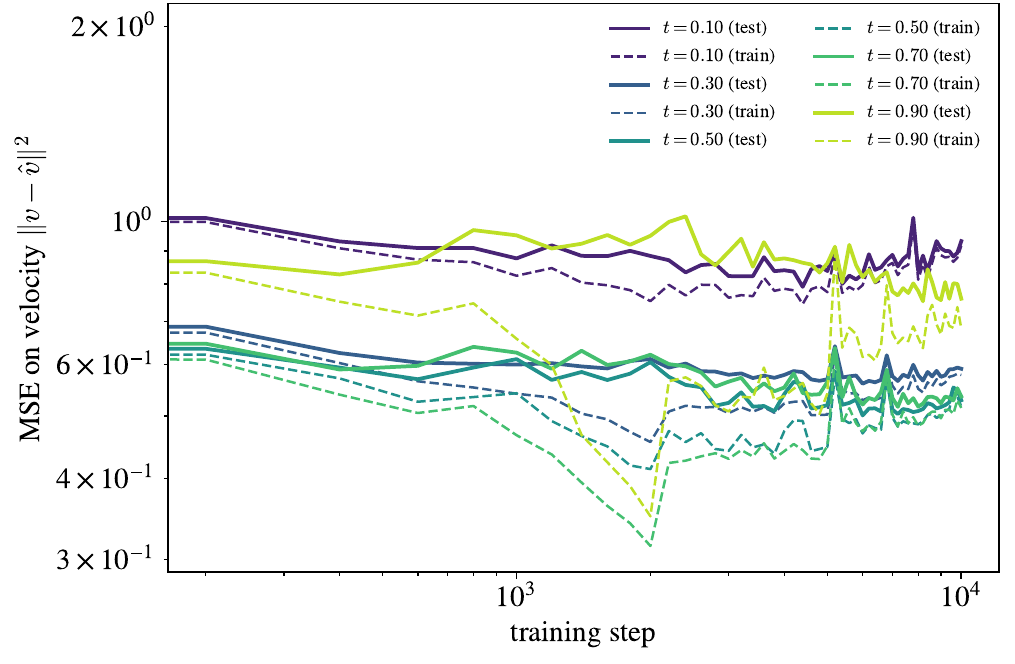}
    \caption{Training and test losses along training for the \emph{(Left)} unregularized and \emph{(Right)} regularized dynamics of Fig.~\ref{fig:time_evolution_NTK}.}
    \label{fig:regularization}
\end{figure}

\subsection{Computational resources}
All experiments were run on NVIDIA H100 GPUs: a single GPU for $n<2048$ and four GPUs in parallel for $n\geq 2048$. As an example, the truncation experiment (kernel regression and generation) for $n=2048$ takes about 4 hours on four GPUs.

\section{LLM usage}

We describe here the precise role played by large
language models (LLMs) in the preparation of this work.
\paragraph*{Code.} LLMs were used to write and debug parts of the code, both for the
numerical experiments and for the scripts producing the figures of this paper.

\paragraph*{Writing and presentation.} LLMs were used for copy-editing throughout:
correcting grammar, finding typographical errors, and harmonizing notation, style and
cross-referencing across sections.

\paragraph*{Proof assistant.} LLMs were used as a proof assistant: to check derivations step by step, to look for errors in the proofs, and to write the code performing the numerical verifications of the analytical results.

\paragraph*{Responsibility.} Every LLM-assisted derivation was checked independently and thoroughly by
the authors, and every suggested edit was reviewed before inclusion. The authors take
full responsibility for the content, the originality and the scientific integrity of this
work, including any remaining errors.

\end{document}